%% file: main.tex
\documentclass[12pt]{article}
\RequirePackage[l2tabu, orthodox]{nag}
\usepackage{titlesec}
\titleformat*{\section}{\Large\bfseries}

\usepackage[english]{babel}
\usepackage[parfill]{parskip}
\usepackage{afterpage}
\usepackage{framed}
\usepackage{nicefrac}
\usepackage[utf8]{inputenc} 
\usepackage[T1]{fontenc}    
\usepackage{amsthm}
\usepackage{subfig}
\usepackage{thmtools}
\usepackage{algorithm, algpseudocode}%
\floatname{algorithm}{Algorithm}

\usepackage{amsthm}

\algtext*{EndFunction}
\usepackage{thm-restate}
\usepackage{multicol}
\usepackage{caption}
\usepackage{minitoc}
\usepackage{subcaption}
\usepackage[authoryear, round]{natbib}
\usepackage{graphicx}
\usepackage{hyperref}       
\usepackage{url}            
\usepackage{booktabs}       
\usepackage{amsfonts}       
\usepackage{bm}
\usepackage[title]{appendix}
\usepackage{thm-restate}
\newtheorem{defn}{Definition}
\newtheorem{assumption}{Assumption}
\newtheorem{assumptiontrick}{Assumption}
\newtheorem{proposition}[defn]{Proposition}
\usepackage[margin=1in]{geometry}
\usepackage{tcolorbox}
\usepackage{adjustbox}

\usepackage[defaultlines=3,all]{nowidow}
\newtheorem{theorem}[defn]{Theorem}
\usepackage{nicefrac}       
\usepackage{microtype}      
\usepackage{xcolor}         

\usepackage{tabularx}
\usepackage[labelfont=bf,format=plain,justification=raggedright,singlelinecheck=false]{caption}
\usepackage{enumitem}
\usepackage{subcaption}
\usepackage{stackrel}
\usepackage{mathtools}
\usepackage{authblk}

\usepackage{cleveref}

\allowdisplaybreaks

\input{utils/math}
\usepackage{amsmath}

\title{Credal Two-Sample Tests of Epistemic Uncertainty}

\author[1]{\textbf{Siu Lun Chau}}
\author[2,3]{\textbf{Antonin Schrab}}
\author[3]{\textbf{Arthur Gretton}}
\author[4]{\textbf{Dino Sejdinovic}}
\author[1]{\textbf{Krikamol Muandet}}

\affil[1]{\small{Rational Intelligence Lab, CISPA Helmholtz Center for Information Security, Germany}}
\affil[2]{\small{Centre for Artificial Intelligence, University College London \& Inria London, United Kingdom}}
\affil[3]{\small{Gatsby Computational Neuroscience Unit, University College London, United Kingdom}}
\affil[4]{\small{School of Computer and Mathematical Sciences \& AIML, University of Adelaide, Australia}}

\begin{document}

\maketitle

\begin{abstract}

We introduce credal two-sample testing, a new hypothesis testing framework for comparing credal sets---convex sets of probability measures where each element captures aleatoric uncertainty and the set itself represents epistemic uncertainty that arises from the modeller's partial ignorance. 
Compared to classical two-sample tests, which focus on comparing precise distributions, the proposed framework provides a broader and more versatile set of hypotheses. This approach enables the direct integration of epistemic uncertainty, effectively addressing the challenges arising from partial ignorance in hypothesis testing. By generalising two-sample test to compare credal sets, our framework enables reasoning for equality, inclusion, intersection, and mutual exclusivity, each offering unique insights into the modeller's epistemic beliefs. As the first work on nonparametric hypothesis testing for comparing credal sets, we focus on finitely generated credal sets derived from i.i.d. samples from multiple distributions---referred to as \emph{credal samples}.
We formalise these tests as two-sample tests with nuisance parameters and introduce the first permutation-based solution for this class of problems, significantly improving existing methods. Our approach properly incorporates the modeller's epistemic uncertainty into hypothesis testing, leading to more robust and credible conclusions, with kernel-based implementations for real-world applications.
\end{abstract}


\input{sections/001_introduction}

\input{sections/02_preliminary}
\input{sections/03_main}
\input{sections/04_related_work}

\input{sections/05_experiments}
\input{sections/06_discussion}

\paragraph{Acknowledgements.} The authors would like to thank Yusuf Sale and Eyke Hüllermeier for their insightful discussions during the early stages of project development. Special thanks to Michele Caprio for his invaluable feedback, which significantly improved the draft. Additionally, discussions with Jean-Francois Ton and Anurag Singh contributed to enhancing the clarity and quality of the writing.
Antonin Schrab acknowledges support from the U.K. Research and Innovation under
grant number EP/S021566/1 and from the Gatsby Charitable Foundation.
Arthur Gretton acknowledges
support from the Gatsby Charitable Foundation.



\begin{appendices}
\input{sections/appendix/00_header}
\input{sections/appendix/A_algorithm}

\input{sections/appendix/B_proofs}

\input{sections/appendix/C_experiments}

\input{sections/appendix/E_background_on_kernel_testing}
\end{appendices}

\end{document}

%% file: utils/math.tex
\def\bflambda{\bm \lambda}
\def\bfeta{\bm \eta}

\def\bfP{\bm P}
\def\bfS{\bm S}

\def\bfZ{{\bf Z}}

\def\NN{{\mathbb N}}    
\def\RR{{\mathbb R}}    
\def\PP{{\mathbb P}}     
\def\EE{{\mathbb E}}    

\def\11{{\mathbf 1}}    

  \def\cG{{\mathcal G}}     \def\cH{{\mathcal H}}   \def\cC{{\mathcal C}}        \def\cP{{\mathcal P}}       \def\cF{{\mathcal F}}  \def\cL{{\mathcal L}}  \def\cX{{\mathcal X}}   \def\cZ{{\mathcal Z}}

  \def\bfz{{\bf z}} 
  
 \def\bfQ{{\bf Q}}

\def\bfeta{\boldsymbol{\eta}}

\def\bfX{{\bf X}} \def\bfY{{\bf Y}} 

\def\bfM{{\bf M}}

\newcommand{\nocontentsline}[3]{}
\newcommand{\tocless}[2]{\bgroup\let\addcontentsline=\nocontentsline#1{#2}\egroup}

%% file: sections/001_introduction.tex
\section{Introduction}
\label{sec: intro}

Science is inherently inductive and thus involves uncertainties. They are commonly categorised as \emph{aleatoric uncertainty} (AU), which refers to inherent variability, and \emph{epistemic uncertainty} (EU), arising from limited information such as finite data or model assumptions \citep{hora1996aleatory}. These uncertainties often overlap and intertwine, as scientists may be epistemically uncertain about the aleatoric variation in their inquiry. Distinguishing and acknowledging them is crucial for the safe and trustworthy deployment of intelligent systems~\citep{kendall2017uncertainties,hullermeier_aleatoric_2021}, as they lead to different down-stream decisions. For example, experimental design aims to reduce EU \citep{nguyen2019epistemic, chau2021bayesimp}, while risk management uses hedging strategy to address AU \citep{mashrur2020machine}



\begin{figure}
    \centering
    \includegraphics[width=0.8\linewidth]{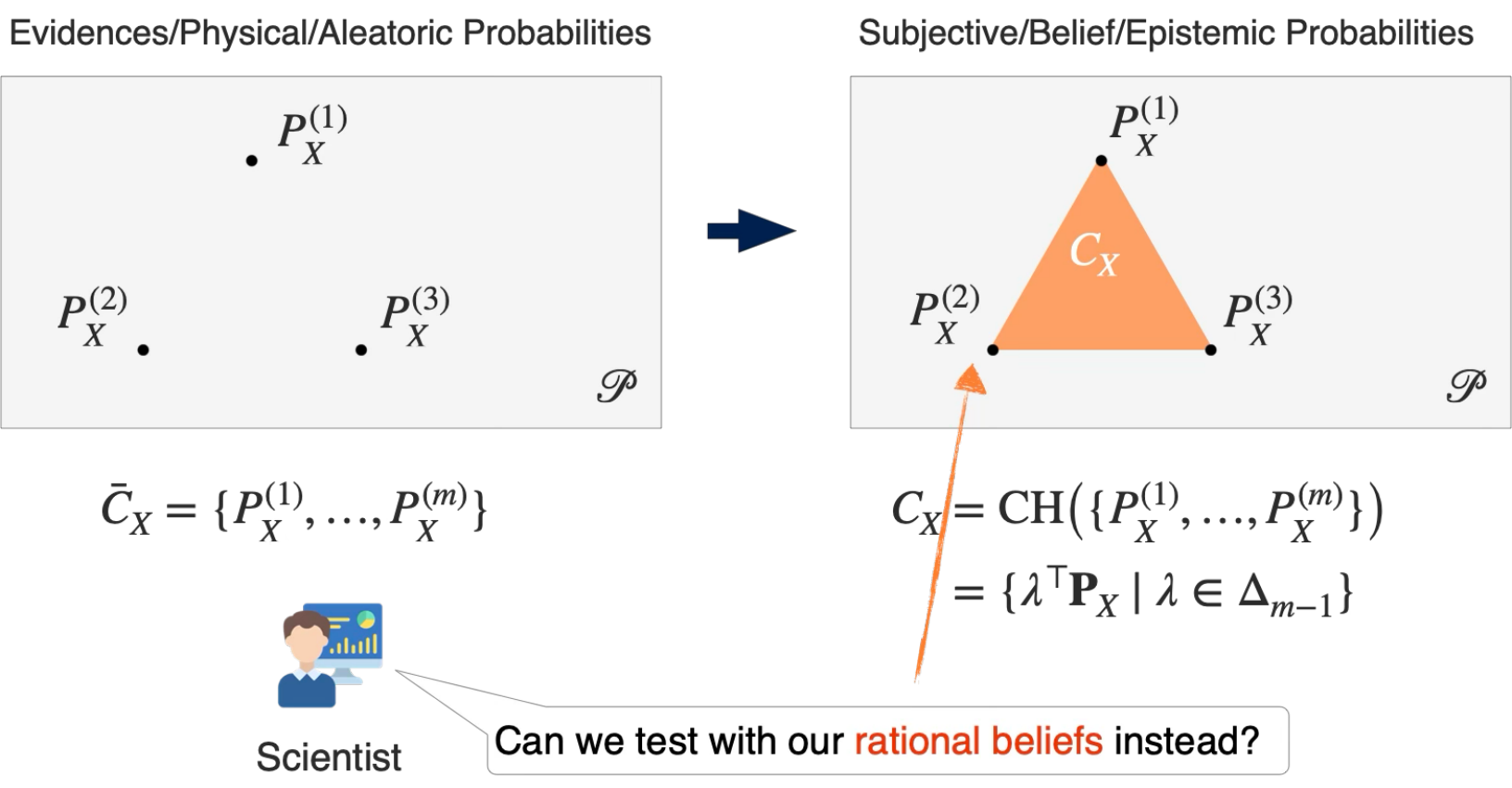}
    \caption{Motivation: From comparing precise distributions to comparing rational epistemic beliefs.}
    \label{fig:enter-label}
\end{figure}

While AU is often modelled using probability distributions, modelling EU---particularly in states of epistemic ignorance, also known as partial ignorance or incomplete knowledge~\citep{dubois1996representing}---poses greater challenges. For instance, a scientist analysing insulin levels in Germany may have data from multiple hospitals, each representing aleatoric variation as a probability distribution. However, these distributions are merely proxies for the population-level insulin distribution, which is difficult to infer due to data collection limitations. At this point, the scientist is facing what is known as a \emph{dataset-level uncertainty}. A Bayesian approach could aggregate the data based on a prior if the representativeness of each source is known, but in many cases, scientists operate under partial ignorance, lacking such prior information~\citep{bromberger1971science}. Assigning a uniform prior by following the \emph{principle of indifference}~\citep{bernoulli1713,laplace1812,keynes1921treatise} 
only reflects indifference, not epistemic ignorance. Epistemologists~\citep{elkin2017imprecise} term this challenge of calibrating belief objectively under multiple evidence as \emph{Chance Calibration}~\citep{williamson_defence_2010}. They argue rational agents ought to represent ambiguity through the convex hull of the available distributions, capturing all plausible ways to aggregate the evidence. This convex set, called a \emph{credal set}, has a robust Bayesian sensitivity analysis interpretation~\citep{berger1994overview}, as it incorporates all possible priors to represent partial ignorance.


But how can we conduct a statistical hypothesis test under such epistemic uncertainty? Suppose now that insulin data are collected from hospitals in China and Germany, serving as proxies for their populations. The World Health Organization might use a two-sample test~\citep{student1908probable} to determine if there's a significant difference between the countries. However, standard tests require comparing precise distributions, forcing the analyst to overlook EU arising from partial ignorance and relying on subjective judgements for evidence aggregation. The test’s outcome then heavily depends on their subjective choices. Alternatively, using credal sets to represent partial ignorance and comparing them would directly incorporate EU into the analysis, leading to more objective and credible conclusions. 
However, there has been no valid method for comparing sample-based credal sets under a hypothesis-testing framework.


\textbf{Our contributions.} To address this gap, we propose \emph{credal two-sample testing}, a new testing framework that introduces four null hypotheses for comparing epistemic ignorance represented as credal sets. Our null hypotheses generalise the standard two-sample null hypothesis since comparing two precise distributions is equivalent to comparing singleton credal sets. Our credal tests, however, allow for reasoning not only about equality but also inclusion, intersection, and mutual exclusivity of credal sets, offering deeper insights into imprecise beliefs~(see Figure~\ref{fig:hypotheses}). For example, the \emph{credal specification test} checks if a distribution belongs to a credal set, assessing the representativeness of EU or whether the distribution fits the evidence. The \emph{credal equality test} evaluates the consistency of belief states across evidence, while the \emph{credal inclusion test} compares ambiguity between nested credal sets, indicating which set has less uncertainty. This offers an alternative approach for uncertainty comparison given the lack of consensus on how to quantify EU for credal sets~\citep{sale2023volume}. Lastly, the \emph{credal plausibility test} checks whether two credal sets overlap, the null hypothesis which indicates some agreement, prompting further investigation to resolve ambiguity. Rejection of plausibility, on the other hand, implies that aleatoric variations exhibit a statistically significant irreconcilable difference even having taken all EU into account.

We provide valid testing procedures for each null hypothesis with minimal distributional assumptions. First, we show that all four credal tests can be formalised as precise two-sample tests involving nuisance parameters, in line with recent advances in two-sample testing~\citep{bruck2023distribution}. Next, we develop kernel-based non-parametric tests that asymptotically control Type I error (false positives) under the null hypotheses and are consistent, achieving zero Type II error (false negatives) under any fixed alternative hypothesis. Our approach extends beyond credal testing, offering a versatile framework applicable to a wider range of emerging testing problems involving nuisance parameters. Our permutation-based method empirically outperforms existing methods that rely on asymptotic normality of the studentised statistic~\citep{bruck2023distribution}.

The paper is organized as follows: Section~\ref{sec: preliminary} reviews credal sets and kernel two-sample tests. Section~\ref{sec: main_method} introduces credal two-sample tests and proves their validity. Section~\ref{sec: related_work_new} discusses related work, followed by experimental results in Section~\ref{sec: experiments}. Finally, Section~\ref{sec: discussion} explores potential applications and future directions. We also included further discussions on the philosophy and interpretation of our testing procedure in Section~\ref{appendix_Sec: further_remarks}.

Our JAX-based~\citep{jax2018github} implementation of credal tests, along with the code to reproduce the experiments, is also available\footnote{\url{https://github.com/muandet-lab/Credal2STests}}.

\begin{figure}
    \centering
    \includegraphics[width=0.95\linewidth]{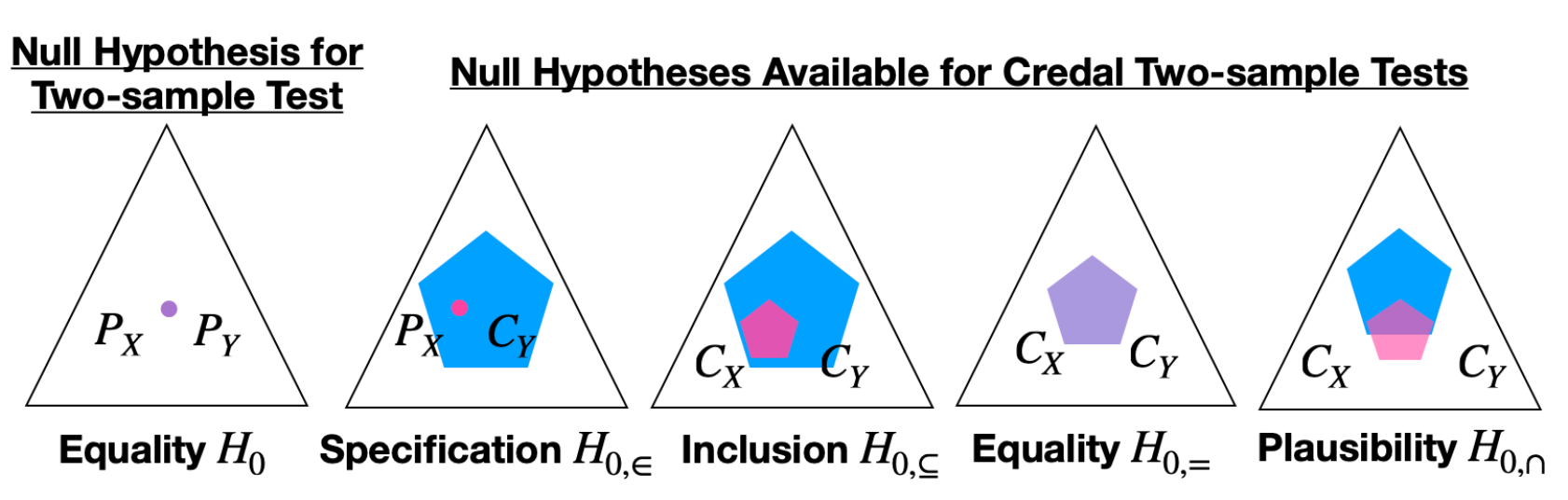}
    \caption{\small{Different comparisons between credal sets within a probability simplex with $2$ degrees of freedom.}}
    \label{fig:hypotheses}
\end{figure}



%% file: sections/02_preliminary.tex
\section{Preliminaries}
\label{sec: preliminary}

Let $X$ and $Y$ be random variables defined on a topological space $\cX$, with respective probability measures $P_X, P_Y \in \cP(\cX)$, where $\cP(\cX)$ denotes the set of all probability measures on $\cX$. We denote $S_X = \{x_i\}_{i=1}^n$ and $S_Y=\{y_i\}_{i=1}^m$, each as independent and identically distributed (i.i.d.) samples from $P_X$ and $P_Y$, respectively. For multiple observations, superscripts like $X^{(j)}$ denote the random variable is from the $j^{th}$ dataset, with corresponding samples $S^{(j)}_X$ and distribution $P^{(j)}_{X}$, for $j=1,\dots,\ell$. Boldface notation represents the concatenation, e.g., $\bfX = \{X^{(j)}\}_{j=1}^\ell$, $\bfS_X = \{S^{(j)}_X\}_{j=1}^\ell$, and $\bfP_X = \{P_X^{(j)}\}_{j=1}^\ell$. The same notation applies to $Y^{(j)}$, $S_Y^{(j)}$, $P^{(j)}_{Y}$, for $j=1\dots, r$, with corresponding boldface notations $\bfY$, $\bfS_Y$, and $\bfP_Y$. We also refer to $\bfS_X, \bfS_Y$ as \textbf{credal samples}, as they are the samples we later use to construct credal sets. The probability simplex with $\ell - 1$ degree of freedom is denoted as $\Delta_\ell:=\{\bflambda \in \RR_{\geq 0}^{\ell}\mid {\bf1}^\top \bflambda = 1 \}$, $\Delta_r$ is defined analogously.


\subsection{Epistemic Uncertainty and Credal Sets}
\label{subsec: credal set}


Epistemic uncertainty (EU) is typically modelled in two ways. The first involves defining a second-order distribution in $\cP(\cP(\cX))$ to capture uncertainty about the primary distribution $P_X$. This approach is common in supervised learning, where query-label data inform the second-order distribution, reflecting model uncertainty \citep{gelman1995bayesian,kendall2017uncertainties,ulmer2021survey}. However, second-order distributions have limitations in representing partial ignorance. For example, a ``uniform'' second-order distribution fails to distinguish true ignorance from certainty with uniformly distributed beliefs \citep[Sec 5.10]{walley1991statistical}. 

In contrast, credal sets $\mathcal{C}\subseteq \cP(\cX)$, rooted in \emph{imprecise probability}~\citep{walley1991statistical}, have gained popularity for modelling EU. Credal sets can be constructed in various ways, such as through probability bounds or contamination sets~\citep{huber2011robust}, but we focus on those formed as the convex hull of a discrete set of probability distributions representing different information sources, also known as the finitely generated credal sets~\citep{augustin_introduction_2014}. Given $\bfP_X$, a credal set is modeled as $\cC_X = \{\bflambda^\top \bfP_X \mid \bflambda \in \Delta_\ell \}$, with an analogous construction for $\cC_Y$. Here, $\bfP_X$ and $\bfP_Y$ are extreme points, fully describing closed and convex credal sets via the Krein-Milman theorem \citep[Theorem 3.23]{Rudin91:FA}. Credal sets have been applied across learning algorithms to model EU, including classification \citep{zaffalon2002naive}, Bayesian networks \citep{cozman2000credal}, decision trees \citep{abellan2010ensemble,abellan2017random}, and deep learning \citep{caprio2023credal}. While less explicitly stated, credal sets are also used in domain generalisation and distributionally robust optimisation to represent epistemic ignorance about the deployment distributions \citep{mansour_multiple_2012, sagawa_distributionally_2020, foll_gated_2023}. See, also, \citet{singh2024domain} and \citet{caprio_credal_2024}.

\begin{table}[t]
    \centering
    \caption{\small{Different hypotheses to compare credal sets.}}
    \resizebox{\columnwidth}{!}{
    \begin{tabular}{cccc}
    \toprule
    \textbf{Specification} & \textbf{Inclusion} & \textbf{Equality} & \textbf{Plausibility} \\
    \midrule
     $H_{0,\in}: P_X \in \cC_Y $& $H_{0, \subseteq}: \cC_X \subseteq \cC_Y$ & $H_{0, =}: \cC_X = \cC_Y$  &${ H_{0, \cap}: \cC_X \cap \cC_Y \neq \emptyset }$\\
     $H_{A,\in}: P_X\not \in \cC_Y$ & $H_{A, \subseteq}: \cC_X \not \subseteq \cC_Y$ & $H_{A, =}: \cC_X \neq \cC_Y$ &${\small H_{A, \cap}: \cC_X \cap \cC_Y = \emptyset }$ \\
    \bottomrule
\end{tabular}}
    \label{table: hypotheses}
\end{table}

\textbf{Credal discrepancy.} Comparison of credal sets has been explored in \citet{abellan_measures_2006}, \citet{destercke_handling_2012}, and \citet{bronevich_characteristics_2017}, where they proposed various discrepancy measures between credal sets, however not under a hypothesis testing setting.
These approaches can be unified as follows: Given a statistical divergence \( d:\cP(\cX) \times \cP(\cX) \to \RR_{\geq 0} \), the \textit{degree of inclusion} of \(\cC_X \subseteq \cC_Y\) is measured as $$\operatorname{Inc}(\cC_X, \cC_Y) = \sup_{P_X\in\cC_X}\inf_{P_Y\in\cC_Y}d(P_X, P_Y).$$ The \textit{degree of equality} is the Hausdorff distance $$\operatorname{Eq}(\cC_X,\cC_Y) = \max(\operatorname{Inc}(\cC_X, \cC_Y), \operatorname{Inc}(\cC_Y, \cC_X)),$$ and the \textit{degree of intersection} is $$\operatorname{Int}(\cC_X, \cC_Y) = \inf_{P_X\in\cC_X}\inf_{P_Y\in\cC_Y}d(P_X, P_Y).$$ These measures are valid as shown by Proposition~\ref{prop: credal discrepancy}.
\begin{restatable}[]{proposition}{CredalDiscrepancy}
\label{prop: credal discrepancy}    
$\operatorname{Inc}(\cC_X,\cC_Y) = 0$ if and only if $\cC_X\subseteq \cC_Y$, $\operatorname{Eq}(\cC_X, \cC_Y) = 0$ if and only if $\cC_X = \cC_Y$, and $\operatorname{Int}(\cC_X, \cC_Y) = 0 $ if and only if $\cC_X \cap \cC_Y \neq \emptyset$.
\end{restatable}
All proofs in this paper are provided in Appendix~\ref{appendix_Sec: proofs}. These discrepancies inspired our testing procedures. Unlike previous work, which focused on discrete distributions or cases where the parametric form of distributions are known explicitly, we leverage the \emph{Maximum Mean Discrepancy} (cf.  Section~\ref{subsec: two-sample testing}) as the divergence $d$ to derive a kernel credal discrepancy (KCD)~(cf. Proposition~\ref{prop: credal_mmd}). KCD is nonparametric, sample-based, and applicable to a broad range of data types, including continuous data, graphs, sets, and images, making it more versatile~\citep{gartner2008kernels}. 

\subsection{Classical Kernel Two-sample Testing }
\label{subsec: two-sample testing}

A two-sample test~\citep{student1908probable} determines whether two distributions, $P_X$ and $P_Y$, differ statistically based on their respective i.i.d. samples, $S_X$ and $S_Y$. Specifically, we test for the null hypothesis $H_0: P_X = P_Y$ against the alternative $H_A: P_X \neq P_Y$. Modern approaches require minimal distributional assumptions, with notable examples such as energy-distance~\citep{SzeRiz05,baringhaus2004new,sejdinovic2013equivalence}, and kernel-based tests~\citep{gretton2006kernel,gretton2012kernel}, which form the foundation of our methods due to their simplicity and flexibility to handle various data types, including both structured and unstructured data such as graphs, strings, and images.

\textbf{Maximum mean discrepancy (MMD) \citep{gretton2006kernel,gretton2012kernel}.} 
Let $k:\cX\times\cX\to\RR$ be a real-valued positive definite kernel on $\cX$, with $\cH_k$ the corresponding reproducing kernel Hilbert space~(RKHS)~\citep{aronszajn1950theory}. Denote $\cF_k=\{f\in\cH_k\mid \|f\|_{\cH_k} \leq 1\}$ as the unit ball of $\cH_k$. The MMD between $P_X$ and $P_Y$, which serves as the test statistic for a kernel two-sample test, is defined as the integral probability metric~\citep{muller1997integral} over $\cF_k$:
\begin{align}\label{eq:MMD-org}
\MoveEqLeft\operatorname{MMD}(P_X, P_Y) \nonumber \\ 
&:= \sup_{f\in\cF_k}\left|\EE_{X
\sim P_X}[f(X)] - \EE_{Y\sim P_Y}[f(Y)]\right| \nonumber \\
&=\sup_{f\in\cF_k}\left|\langle f,\EE_{X
\sim P_X}[k(X,\cdot)] - \EE_{Y\sim P_Y}[k(Y,\cdot)]\rangle_{\cH_k}\right| \nonumber \\
&= \|\mu_{P_X} - \mu_{P_Y}\|_{\cH_k},
\end{align}
where the second equality follows from the reproducing property of $f$, i.e., $f(x) = \langle f, k(x,\cdot)\rangle_{\cH_k}$. The function $k(x,\cdot)$ can be thought of as a canonical feature map of $x$ in $\cH_k$. By taking expectation over this canonical feature map, the function $\mu_P := \EE_{X
\sim P}[k(X,\cdot)] \in \cH_k$ used in \eqref{eq:MMD-org} is known as the kernel mean embedding (KME)~\citep{smola2007hilbert,muandet2017kernel} of $P$. 
In other words, MMD measures the discrepancy between $P_X$ and $P_Y$ as the RKHS norm of the difference between their corresponding KMEs $\mu_{P_X}$ and $\mu_{P_Y}$.

For a certain class of kernel functions, known as \emph{characteristic} kernels, $\cF_k$ becomes rich enough to differentiate any two probability distributions. In this case, the MMD becomes a statistical divergence such that  $P_X = P_Y$  if and only if \( \operatorname{MMD}(P_X, P_Y) = 0 \).
\begin{defn}[\citealt{SriGreFukLanetal10,sriperumbudur2011universality}]
    A kernel $k$ is characteristic iff $P\mapsto \mu_{P}$ is injective. 
\end{defn}
The Gaussian kernel  $k(x, x^{\prime}) = \exp\left(-\|x - x^{\prime}\|^2/2\sigma^2\right)$ with  bandwidth $\sigma$ is characteristic, for example. The KME provides a flexible yet powerful representation of distributions since such representation can be estimated purely based on samples $S_X$, i.e., $$\hat{\mu}_{P_X} = \frac{1}{n}\sum_{i=1}^n k(x_i,\cdot),$$ without needing any distributional assumptions on $P_X$. The estimation is also quite statistically efficient, with $\|\frac{1}{n}\sum_{i=1}^n k(\cdot, x_i) - \mu_{P_X}\|_{\cH_k}$ converges to 0 at rate $\frac{1}{\sqrt{n}}$ under some regularity conditions on $k$, see \citet[Proposition A.1]{tolstikhin2017minimax}.

It follows from \eqref{eq:MMD-org} that the squared MMD can be expressed solely in terms of kernel evaluations, i.e., $$\operatorname{MMD}^2(P_X, P_Y) = \EE_{X,X'}[k(X,X')] - 2\EE_{X,Y}[k(X,Y)] + \EE_{Y,Y'}[k(Y,Y')],$$ with $X,X'$ distributed as $P_X$ and $Y,Y'$ as $P_Y$. 
This expression leads to an unbiased estimator, $\operatorname{MMD}^2(S_X, S_Y)$, now expressed as a function of samples instead of distributions, when $n=m$, can be compactly expressed as $$\frac{1}{n(n-1)}\sum_{i\neq j} h(x_i,y_i,x_j,y_j)$$ where $h$ is the core of the U-statistic, given by $h(x_i, y_i, x_j, y_j) = k(x_i, x_j) + k(y_i,y_j) - k(x_i, y_j) - k(x_j, y_i)$.

In practice, the test rejects the null hypothesis when the test statistic deviates significantly from zero. This is determined by comparing the test statistic to a critical value. In order to control the Type I error by $\alpha$ as desired, the critical value should be set to the $(1-\alpha)$ quantile of the distribution of the MMD statistic under the null. This quantile can be estimated using permutation to simulate this distribution under the null due to sample exchangeability, resulting in a permutation test of exact level $\alpha$~\citep[Chapter 10]{lehmann1986testing}. See \Cref{appendix subsubsec: permutation} for further discussion on permutation test.

\Cref{appendix: kernel stuff} provides additional materials for readers who are interested in kernel methods~(\ref{appendix subsec: kernel methods}), kernel mean embedding~(\ref{appendix subsec: kme}), and kernel-based testing~(\ref{appendix subsec: kernel two-sample test}). 


%% file: sections/03_main.tex
\section{Credal Two-sample Tests}
\label{sec: main_method}



The goal of credal two-sample tests is to compare the population-level credal sets \(\cC_X\) and \(\cC_Y\) which are the convex hulls formed by the population-level extreme points $\bfP_X$ and $\bfP_Y$, which one has access to only via the credal samples $\bfS_X$ and $\bfS_Y$. For simplicity, we assume all datasets in $\bfS_X$ and $\bfS_Y$ have the same sample size $n$, but our theory extends to the general case of different sample sizes. We begin by introducing several key foundational concepts used in the framework.

\subsection{Fundamental Concepts in Credal Tests}



\paragraph{Credal hypotheses.} The credal two-sample tests enable the comparison of credal sets under different null hypotheses, aligning with specific scientific objectives as overviewed in Section~\ref{sec: intro}. Table~\ref{table: hypotheses} outlines the null hypotheses, as visualised in Figure~\ref{fig:hypotheses}. Although these hypotheses follow a natural hierarchy (i.e., $\mathcal{C}_X = \mathcal{C}_Y \Rightarrow \mathcal{C}_X \subseteq \mathcal{C}_Y \Rightarrow \mathcal{C}_X \cap \mathcal{C}_Y \neq \emptyset$), we focus on tackling each credal hypothesis separately and on proving the validity of each individual credal test. While not the primary focus of this work, our proposed tests can be combined via multiple testing to tackle the nested hypotheses problem~\citep{bauer1987multiple}.



\paragraph{Precise tests with nuisance parameters.} Although credal discrepancies (Section~\ref{subsec: credal set}) may seem appropriate as test statistics, determining their limiting distributions, and subsequently their critical values, is challenging due to the loss of sample exchangeability under the null credal hypotheses. To address this, we formalise credal testing as a precise two-sample test involving nuisance parameters. For instance, under the specification hypothesis $H_{0,\in}$, $P_X\in\cC_Y$ holds if, and only if, there exists a plausible epistemic belief $\bfeta_0\in\Delta_r$ such that the aggregated evidence $\bfeta_0^\top\bfP_Y$ aligns with $P_X$, i.e., $\bfeta_0^\top \bfP_Y = P_X$. Here $\bfeta_0$ is the nuisance parameter, which is unknown \emph{a priori} under partial ignorance. However, credal discrepancies enable the estimation of these plausible beliefs from the available samples, leading to the following two-stage approach:

\begin{enumerate}[leftmargin=*]
    \item \textbf{Epistemic alignment (EA):} Observations from each $S_X^{(j)}, S_Y^{(j)}$ in $\bfS_X$ and $\bfS_Y$ are divided into $n_e$ samples for estimation and $n_t$ samples for testing. An optimisation process uses the $(\ell + r)n_e$ samples to identify convex weights $\bfeta^e$ and/or $\bflambda^e$, which represent plausible epistemic attitudes that align the aggregated distributions in each credal set.
    \item \textbf{Hypothesis testing (HT):} After alignment, $n_t$ samples $\tilde{S}_{Y,\bfeta^e}$ and/or $\tilde{S}_{X,\bflambda^e}$ are simulated from the aggregated distributions ${\bfeta^e}^\top\bfP_Y$ and/or ${\bflambda^e}^\top\bfP_X$, through resampling the unused samples in $\bfS_X,\bfS_Y$ from the previous step. A precise two-sample test is then performed based on these samples.
\end{enumerate}


Algorithm~\ref{algo: redraw_samples} details our resampling approach. A similar resampling-based test was studied in \citet{thams2023statistical} but they assume the weights are known while ours require estimation. \citet{key_composite_2024} uses a similar two-stage approach for composite goodness-of-fit tests, but they allow unlimited redrawing from actual distributions, whereas we are restricted to resampling from observations. These differences lead to distinct theoretical analyses and contributions. Although some frameworks~\citep{davies1987hypothesis,chen2024biased} address how estimation impacts test validity, it remains underexplored in nonparametric two-sample testing. \citet{bruck_distribution_2023} first demonstrated as long as estimation error converges at order \(\nicefrac{1}{\sqrt{n_e}}\), asymptotic normality of their proposed statistic is maintained. In contrast, our tests, using the standard kernel two-sample statistic, achieve the same asymptotic Type I error control through a permutation procedure and demonstrate higher power. Crucially, we show that adaptively splitting samples to manage the estimation error’s decay rate relative to the increase of test power is necessary to preserve Type I control, as fixed splits (e.g., 50:50) can lead to inflated Type I errors (see Figure~\ref{fig: simulation}).

\paragraph{Sample splitting.} To prepare the datasets for the EA and HT steps, we apply a standard sample-splitting procedure with a split ratio $$\rho = \frac{n_e}{n},$$ setting $n_t = n - n_e$. In Appendix~\ref{appendix_subsubsec: comparing_with_double_dipping}, we also explore an alternative approach of sample splitting considered in \citet{key_composite_2024}, referred to as ``double-dipping'', where samples used for estimation are reused for testing, and examine its effect on test validity. Other alternative approaches to sample-splitting have been studied in \citet{Kubler20:LKT-WDS,Kubler22:WTS,Kubler22:ATS}. 





\paragraph{Kernel credal discrepancy.} Our optimisation objectives are based on the MMD between credal elements, referred to as the kernel credal discrepancy (KCD).


\begin{restatable}[]{proposition}{KernelCredalDiscrepancy}
\label{prop: credal_mmd}    
    Let $k$ be a bounded kernel. For any $P_X\in\cC_X$, $P_Y\in\cC_Y$, there exists $\bflambda\in\Delta_\ell$, $\bfeta \in \Delta_r$ such that $\operatorname{MMD}^2(P_X, P_Y) = L(\bflambda, \bfeta)$ where
    $$L(\bflambda, \bfeta) = \bflambda^\top \bfM_{XX}\bflambda -2 \bflambda^\top \bfM_{XY}\bfeta + \bfeta^\top\bfM_{YY}\bfeta,$$ $\bfM_{XY} \in \RR^{\ell \times r}$ with $[\bfM_{XY}]_{ij} = \EE_i\EE_j[k(X^{(i)},Y^{(j)})]$, and $\bfM_{XX}, \bfM_{YY}$ defined analogously.    
    \vspace{-0.5em}
\end{restatable}

We denote $L$ as the population KCD. The matrices $\bfM_{XY}, \bfM_{XX}, \bfM_{XY}$ are the Gram matrices between kernel mean embeddings (KMEs) of the extreme points. Substituting them with their empirical counterparts $\widehat{\bfM}_{XY}, \widehat{\bfM}_{XX}$, and $\widehat{\bfM}_{YY}$, using the empirical KMEs $\hat{\mu}_{P_X^{(i)}}$ and $\hat{\mu}_{P_Y^{(j)}}$ constructed based on estimation samples, gives us the empirical KCD objective $$L_{n_e}(\bflambda,\bfeta) = \bflambda^\top\widehat{\bfM}_{XX}\bflambda - 2\bflambda^\top \widehat{\bfM}_{XY}\bfeta + \bfeta^\top \widehat{\bfM}_{YY}\bfeta,$$ where $[\widehat{\bfM}_{XY}]_{ij} = n_e^{-2}\sum_{k=1}^{n_e}\sum_{l=1}^{n_e} k(x^{(i)}_k,y^{(j)}_l)$ and $\widehat{\bfM}_{XX}, \widehat{\bfM}_{YY}$ defined analogously. \citet{briol_statistical_2019} and \citet{cherief2022finite} also utilise the MMD as a minimum distance estimator. 

Our analyses rely on the following assumptions:
\begin{assumption}
    The extreme points of the credal set are linearly independent.
\end{assumption}
\begin{assumption}
    The kernel $k$ is continuous, bounded, positive definite, and characteristic.
\end{assumption}
\begin{assumption}
    There exists some $n_0\in\NN$, such that for $n_t > n_0$, the function $\cL_{n_t}: \bflambda,\bfeta \mapsto \operatorname{MMD}^2(\tilde{S}_{X,\bflambda}, \tilde{S}_{Y,\bfeta})$ is continuous over $\Delta_\ell \times \Delta_r$ and differentiable over its interior. Furthermore, the gradient $\nabla \cL_{n_t}$ satisfies Lipschitz continuity and a technical condition $\|\nabla \cL_{n_t} - \nabla L\|_{\infty} \leq C' \|\cL_{n_t} - L\|_{\infty}$ for some constant $C'$.
\end{assumption}
\begin{assumption}
    The Schur complement $\bfM_{XX} - \bfM_{XY}\bfM_{YY}^{-1}\bfM_{YX}$ is positive definite.
\end{assumption}

Both the test statistic $\cL_{n_t}(\bflambda, \bfeta)$ and the empirical KCD $L_{n_e}(\bflambda, \bfeta)$, are estimators of $L(\bflambda, \bfeta)$. They differ as follows: $\cL_{n_t}(\bflambda, \bfeta)$ estimates KMEs of the mixture distributions directly from samples of the mixtures, while $L_{n_e}(\bflambda, \bfeta)$ estimates KMEs of the mixture distributions as mixtures of KMEs of each extreme point distribution. Assumption~\ref{assumption 0} facilitates our theoretical analysis, but even if this assumption is violated, our credal tests remain valid (c.f. Appendix~\ref{appendix subsub linearly dependent}). Assumption~\ref{assumption: 1} imposes regularity conditions on the RKHS and ensures the MMD serves as a valid divergence measure. These conditions are satisfied by commonly used kernels, such as the Gaussian kernel. Assumption~\ref{assumption: 2}, the smoothness condition, allows us to explicitly analyse the relationship between the estimation error and the test statistic. The smoothness assumption may be less reliable for small sample sizes, but it holds as sample size increases since $\cL_{n_t}$ (and $L_{n_e}$) converge uniformly to the population KCD, which is itself continuous over $\Delta_\ell \times \Delta_r$ and differentiable in its interior. Formally,


\begin{restatable}[]{proposition}{KCDUniformConvergence}
\label{prop: uniform_convergence}
Under Assumption~\ref{assumption: 1}, $L_{n_e}$ and $\cL_{n_t}$ converges uniformly to $L$ at $O(\nicefrac{1}{\sqrt{n_e}})$ and $O (\nicefrac{1}{\sqrt{n_t}}).$
\end{restatable}

The gradient conditions on $\nabla \cL_{n_t}$ are challenging to verify for the sample-based estimator $\cL_{n_t}$. However, these conditions can be confirmed for $L_{n_e}$ (see Proposition~\ref{prop: gradient_uniform_convergence} in Appendix~\ref{appendix_Sec: proofs}), as we have the analytical form of $L_{n_e}$. Given that both $L_{n_e}$ and $\cL_{n_t}$ are uniformly consistent estimators of $L$, it is reasonable to extend this assumption to $\cL_{n_t}$ as well. Assumption~\ref{assumption: 3} is a technical assumption to analyse the convergence rate of KCD optimisers in the plausibility test. 

\subsection{Credal Specification Hypothesis}

To lay the groundwork for testing other null hypotheses, we first introduce the specification test, which checks whether a precise distribution $P_X$ belongs to a credal set $\cC_Y$. Specifically, we test if there exists a convex weight $\bfeta_0\in\Delta_r$ such that $\bfeta_0^\top \bfP_Y = P_X$. We estimate $\bfeta_0$ by minimising the empirical KCD: $$L_{n_e}(1, \bfeta) = \bfeta^\top \widehat{\bfM}_{YY}\bfeta - 2\widehat{\bfM}_{XY}\bfeta + c, $$ where $c$ is a constant, using samples from $S_X$ and $\bfS_Y$. The argument for $\bflambda$ is set to $1$ because a single distribution corresponds to a singleton credal set. The minimiser, $\bfeta^e$, is found via quadratic cone programming~\citep{andersen2013cvxopt} given that $L_{n_e}(1,\bfeta)$ is convex in $\bfeta$ (since kernel is positive definite). We then simulate two sets of i.i.d. samples $\tilde{S}_X$ and $\tilde{S}_{Y, \bfeta^e}$, each of size $n_t$, from $P_X$ and ${\bfeta^e}^\top\bfP_Y$, and conduct the standard kernel two-sample test. Algorithm~\ref{algo: p_inclusion_test} outlines the full procedure, with supplementary algorithms provided in Appendix~\ref{appendix_sec: credal_algorithms} due to space constraints. Since our test uses an estimated parameter $\bfeta^e$ instead of $\bfeta_0$, this raises validity concerns. Nonetheless, Theorem~\ref{thm: main_theorem_h0} addresses this by showing that a carefully selected adaptive sample splitting scheme ensures the null distribution of the statistic based on $\bfeta^e$ converges to the same null distribution as the statistic based on the oracle parameter \( \boldsymbol{\eta}_0 \).

\begin{restatable}[]{theorem}{MainTheoremOne}
\label{thm: main_theorem_h0}
    Under $H_{0,\in}$ and Assumptions \ref{assumption 0}, \ref{assumption: 1} and \ref{assumption: 2}, there exists some $n_0$, such that for $n_t>n_0$, $$|n_t\cL_{n_t}(1,\bfeta^e) - n_t\cL_{n_t}(1,\bfeta_0)| = O( \sqrt{\nicefrac{n_t}{n_e}}).$$ Furthermore, if splitting ratio $\rho$ is chosen adaptively such that $n_t/n_e \to 0$ as $n\to \infty$, then  $$n_t\cL_{n_t}(1, \bfeta^e) \xrightarrow{D} \sum_{i=1}^\infty \zeta_i Z_i^2,$$
    where $\{Z_i\}_{i\geq 1} \overset{i.i.d.}{\sim} N(0,1)$ and $\{\zeta_i\}_{i\geq 1}$ are certain eigenvalues depending on the choice of kernel and $P_X$, with $\sum_{i=1}^\infty \zeta_i < \infty$. Furthermore, under $H_{A,\in}$, $$n_t\cL_{n_t}(1,\bfeta^e) \to \infty$$ as $n\to\infty$.
\end{restatable}



The key intuition behind the theorem is that, under the null, although the estimation error $\|\bfeta^e - \bfeta_0\|$ decreases as the sample size increases, the test statistic also converges to $0$ as the sample size increases. The effect on the difference between the test statistics based on the estimated parameter and the oracle parameter, i.e. $|n_t\cL_{n_t}(1,\bfeta^e) - n_t\cL_{n_t}(1,\bfeta_0)|$, does not converge for any fixed splitting ratio $\rho$ under the null hypothesis.
To mitigate this, we split our sample adaptively such that $\nicefrac{n_t}{n_e} \to 0$ as $n\to\infty$, making the estimation error converge faster than the test statistic's decay (see also \citep{pogodin24conditional} for case of conditional independence testing, where regression is required in computing the test statistic). Consequently, using Slutsky's theorem, we can show the (scaled) test statistic converges in distribution to the same limit~\citep[Theorem 12]{gretton2012kernel} as if the true parameter $\bfeta_0$ were known (see Appendix~\ref{appendix_subsubsec: null_dist_specification} for the effect of estimation demonstrated using empirical null statistic distributions). Consequently, we can still deploy the permutation procedure to determine the critical values, maintaining correct Type I error control, asymptotically. In short, the choice of adaptive sample splitting balances the trade-off between minimising estimation error and preserving test power, ensuring the test is not overly sensitive to minor estimation inaccuracies while still capable of detecting true differences. Furthermore, the consistency of the test under the alternative $H_{A,\in}$ shows that our test can always reject the null hypothesis which does not hold given enough data.

\subsection{Credal Inclusion and Equality Hypotheses }

To verify the inclusion hypothesis $H_{0,\subseteq}: \cC_X \subseteq \cC_Y$, it suffices to check whether all extreme points $P_X$ of $\cC_X$ lie within $\cC_Y$. 
This insight forms the basis for the testing procedure outlined in Algorithm~\ref{algo: s_inclusion_test}, which relies on conducting multiple specification tests. For simplicity, Algorithm~\ref{algo: s_inclusion_test} employs Bonferroni correction~\citep{weisstein2004bonferroni}, but more advanced multiple testing correction techniques~\citep{schrab2023mmd} can be directly employed. Similarly, for testing equality between credal sets $\cC_X = \cC_Y$, we only need to verify whether both $\cC_X \subseteq \cC_X$ and $\cC_Y\subseteq \cC_X$ hold, leading to Algorithm~\ref{algo: equality_test}. As each specification test is proven to be asymptotically valid, Bonferroni correction yields a valid asymptotic Type I control for both inclusion and equality tests.

\begin{algorithm}[t]
\caption{\texttt{specification\_test} for $H_{0,\in}$}
\label{algo: p_inclusion_test}
\begin{algorithmic}[1]
    \Require Sample $S_X$, credal sample $\bfS_Y$, kernel $k$, split ratio $\rho$, level $\alpha$, number of simulated statistics $B$.
    \State $S_X^e, \mathbf{S}^e_Y, S^t_X, \mathbf{S}^t_Y \leftarrow \texttt{split\_data}(\{S_X\}, \bfS_Y, \rho)$.
    \State $\bfeta^e = \arg\min_{\bfeta\in\Delta_{r}} L_{n_e}(1, \bfeta)$, where the objective is constructed using the estimation samples $S_X^e, \mathbf{S}^e_Y$.
    \State $\tilde{S}_X, \tilde{S}_{Y, \eta^e} \leftarrow \texttt{redraw\_samples}(\{S^t_X\}, \bfS^t_Y, 1, \bfeta^e)$.
    \State \Return $\texttt{kernel\_2S\_test}(\tilde{S}_X, \tilde{S}_{Y, \eta^e}, k, B, \alpha).$
\end{algorithmic}
\end{algorithm}

\begin{algorithm}[t]
\caption{\texttt{inclusion\_test} for $H_{0,\subseteq}$}
\label{algo: s_inclusion_test}
\begin{algorithmic}[1]
    \Require Credal samples $\bfS_X$,$\bfS_Y$, kernel $k$, split ratio $\rho$, test level $\alpha$, number of simulated statistics $B$.
    \For{$S_X \in \bfS_X$}
    \State r $\leftarrow$ \texttt{specification\_test}$(S_X, \bfS_Y,  k,\rho, \frac{\alpha}{\ell}, B)$
    \State if r = ``reject'', \Return ``reject''
    \EndFor \ and \Return ``fail to reject''
\end{algorithmic}
\end{algorithm}

\begin{algorithm}[t]
\caption{\texttt{equality\_test} for $H_{0,=}$}
\label{algo: equality_test}
\begin{algorithmic}[1]
    \Require Credal samples $\bfS_X$,$\bfS_Y$, kernel $k$, split ratio $\rho$, test level $\alpha$, number of simulated statistics $B$.
    \State result$_1$ $\leftarrow$ \texttt{inclusion\_test}$(\bfS_X,\bfS_Y,  k, \rho,\frac{\alpha}{2}, B)$
    \State result$_2$ $\leftarrow$ \texttt{inclusion\_test}$(\bfS_Y,\bfS_X, k, \rho,\frac{\alpha}{2}, B)$
    \State \Return ``reject'' if either result rejects, else \Return ``fail to reject''.
\end{algorithmic}
\end{algorithm}

\begin{algorithm}[t]
\caption{\texttt{plausibility\_test} for $H_{0,\cap}$}
\label{algo: intersection_test}
\begin{algorithmic}[1]
    \Require Credal samples $\bfS_X$, $\bfS_Y$, kernel $k$, split ratio $\rho$, test level $\alpha$, number of simulated statistics $B$.
    \State $\bfS_X^e, \mathbf{S}^e_Y, \bfS^t_X, \mathbf{S}^t_Y \leftarrow \texttt{split\_data}(\bfS_X, \bfS_Y, \rho)$.
    \State Estimate $\bflambda^e, \bfeta^e = \arg\min_{\bflambda\in\Delta_\ell, \bfeta\in\Delta_{r}} L_{n_e}(\bflambda, \bfeta)$ using the estimation samples $\bfS_X^e, \mathbf{S}^e_Y$.
    \State $\tilde{S}_{X, \bflambda^e}, \tilde{S}_{Y, \bfeta^e} \leftarrow \texttt{redraw\_samples}o(\bfS^t_X, \bfS^t_Y, \bflambda^e, \bfeta^e)$.
    \State $\text{r} \leftarrow \texttt{kernel\_2S\_test}(\tilde{S}_{X,\bflambda^e}, \tilde{S}_{Y, \bfeta^e}, k, B, \alpha)$
    \State \Return $\text{r}$
\end{algorithmic}
\end{algorithm}

\subsection{Credal Plausibility Hypothesis}

The plausibility hypothesis \(H_{0,\cap}: \cC_X\cap\cC_Y \neq \emptyset\) holds if there exist convex weights \(\bflambda_0\in\Delta_r\) and \(\bfeta_0\in\Delta_\ell\) such that the aggregated distributions epistemically align, i.e., \(\bflambda_0^\top\bfP_X = \bfeta_0^\top\bfP_Y\). To estimate \(\bflambda\) and \(\bfeta\), we minimise the empirical KCD: $$\bflambda^e, \bfeta^e = \arg\min_{\bflambda\in\Delta_\ell,\bfeta\in\Delta_r} L_{n_e}(\bflambda,\bfeta).$$ Unlike in the specification test, this optimisation is \emph{biconvex}: convex in \( \boldsymbol{\lambda} \) when \( \boldsymbol{\eta} \) is fixed and vice versa, but not jointly convex. Therefore, we solve it using iterative coordinate descent, alternately minimising with respect to \(\bflambda\) and \(\bfeta\). Since the problem is convex in each variable when the other is fixed, convergence to a local minimum is guaranteed~\citep{boyd2004convex}. Once the parameters are estimated, we simulate samples \(\tilde{S}_{X,\bflambda^e}\) and \(\tilde{S}_{Y,\bfeta^e}\) from \({\bflambda^e}^\top \bfP_X\) and \({\bfeta^e}^\top \bfP_Y\), respectively, and perform a two-sample test. The full procedure is detailed in Algorithm~\ref{algo: intersection_test}.

To prove the validity of our plausibility test, in addition to the estimation error, another challenge is the non-convexity of the objective, which can lead to multiple minimisers and no unique solution. For example, when testing the plausibility between identical credal sets, any point in the joint simplex $\Delta_\ell \times \Delta_r$ minimises the population KCD, meaning the sequence of estimators may not converge as the sample size increases. Moreover, iterative coordinate descent only guarantees local minima, which may prevent identifying the correct weights \(\bflambda_0\) and \(\bfeta_0\) needed for the null hypothesis to hold, even with access to the population KCD. Despite these difficulties, the test achieves asymptotic Type I error control, as proven in Theorem~\ref{thm: main_theorem_2}. This is because the objective $L_{n_e}$ uniformly converges to $L$ over a compact domain, any minimiser of the limiting objective is also a minimiser of the population objective, i.e., $\lim_{n_e \to \infty} \arg\min L_{n_e} \subseteq \arg\min L$~\citep[Theorem 2]{kall1986approximation}. Furthermore, under Assumption~\ref{assumption: 1}, it can be proven that any local minimiser of $L$ is a global minimiser (see Proposition~\ref{prop: local_is_global}), enabling our procedure to identify a pair of convex weights that satisfy the null hypothesis. As a result, as the sample size increases, the uniform convergence of the KCD objective ensures that the sequence of estimators, while alternating, increasingly approach their corresponding accumulation points in the solution set of $L$. As a result, by adjusting the splitting ratio adaptively, as in the other tests, the impact of estimation error on the scaled test statistic diminishes, and again we can use the permutation procedure to estimate a valid critical value as if we had access to a certain pair of oracle parameters. We now state the theorem formally.

\begin{restatable}[]{theorem}{MainTheoremTwo}
    \label{thm: main_theorem_2}
    Let \(\Theta =\arg\min_{\bflambda\in\Delta_\ell, \bfeta\in\Delta_r} L(\bflambda,\bfeta)\) and \(\cZ = \{\sum_{i=1}^\infty \zeta_{i, \bflambda,\bfeta}Z_i^2 \mid (\bflambda,\bfeta) \in \Theta\}\) with $Z_i\overset{i.i.d.}{\sim}N(0,1)$ and constants \(\{\zeta_{i,\bflambda,\bfeta}\}_{i\geq 1}\) depends on the kernel and weights. Under $H_{0,\cap}$ and Assumptions~\ref{assumption 0},~\ref{assumption: 1},~\ref{assumption: 2}, and~\ref{assumption: 3}, there exists some $n_0$, such that for $n_t > n_0$, there exists $\bflambda_0,\bfeta_0 \in \Theta$, such that $$|n_t \cL_{n_t}(\bflambda^e,\bflambda^e) - n_t \cL_{n_t}(\bflambda_0,\bfeta_0)| = O(\sqrt{\nicefrac{n_t}{n_e}}).$$ Furthermore, if the split ratio $\rho$ is chosen adaptively such that $\nicefrac{n_t}{n_e} \to 0$ as $n\to\infty$, then for all $\epsilon > 0$, there exists some $n_1$, such that for all $n_t > n_1$, there exists $Z\in\cZ$ such that $$|F_{n_t \cL_{n_t}(\bflambda^e,\bfeta^e)}(x) - F_{Z}(x)| < \epsilon$$ for all $x\in\RR$ where $F$ is the cumulative distribution function. Furthermore, under $H_{A,\cap}$, $$n_t\cL_{n_t}(\bflambda^e, \bfeta^e) \to \infty$$ as $n\to\infty$.
\end{restatable}

%% file: sections/04_related_work.tex
\section{Related Work}
\label{sec: related_work_new}


Several works have introduced imprecision in hypothesis testing through set-based approaches or second-order distributions. \citet{bellot_kernel_2021} focused on testing the equality of second-order distributions, while we address first-order distributions with second-order uncertainties represented by credal sets. Other studies~\citep{kutterer2004statistical,liu2020novel} used fuzzy theory to handle measurement imprecision. Bayesian two-sample tests~\citep{holmes2015two,zhang2022bayesian} account for EU about the hypothesis with Bayes factors, while we compare EU about the distributions through credal sets. \citet{hibshman2021higher} used a single credal set in parametric likelihood ratio tests, and \citet{mortier2023calibration} developed a calibration test to assess whether the credal set generated from classification models are well-calibrated. However, neither addresses credal set comparison purely based on samples. Our specification test can also be seen as a hypothesis test for finite mixture models~\citep{aitkin_estimation_1985}, determining if a sample belongs to a mixture of distributions. Unlike traditional one-sample goodness-of-fit tests designed for parametric mixtures~\citep{li2007hypothesis,wichitchan_hypothesis_2019}, we infer mixing distributions directly from samples, opening new research avenues.  

To the best of our knowledge, our work provides the first fully nonparametric method for statistically comparing epistemic uncertainties using credal sets.

%% file: sections/05_experiments.tex
\begin{figure*}[t]
    \centering
    \includegraphics[width=0.95\linewidth]{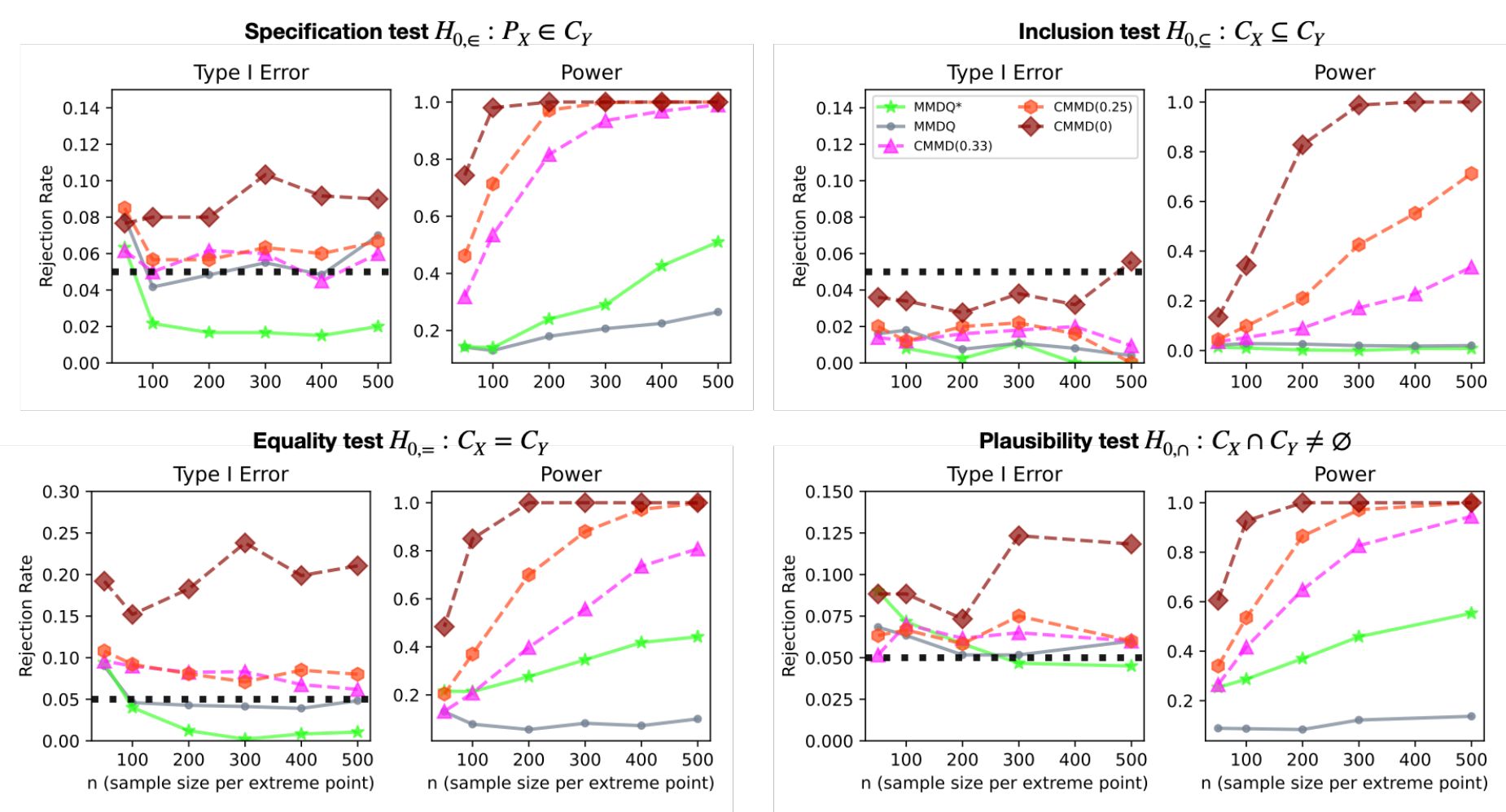}
    \caption{\small{We present the experimental results of our credal tests (labelled as CMMD) on synthetic data at a $0.05$ significance level (black dotted line). CMMD$(0)$ uses fixed sample splitting and fails to control Type I error, rendering it invalid. \textbf{It is included in the power plot for completeness but should not be compared to other valid tests.}}}
    \label{fig: simulation}
\end{figure*}

\section{Experiments}
\label{sec: experiments}

This section shows that our tests are valid and powerful using synthetic data. Semi-synthetic experiments based on MNIST data and detailed ablation studies are provided in Appendix~\ref{appendix_sec: further_experiments}, including larger-scale experiments, the impact of the number of credal samples and sample splitting ratio, comparisons with the double-dipping sample preparation, and sensitivity to violation of Assumption~\ref{assumption 0}.

\textbf{Benchmarking.} As no previous work has compared credal sets in a hypothesis testing framework, we benchmark our method against existing kernel two-sample tests that handle nuisance parameters. To our knowledge, the only relevant work is \citet{bruck2023distribution}, which introduced MMDQ and MMDQ$^\star$. MMDQ uses a distribution-free, studentised MMD test statistic, while MMDQ$^\star$ improves test power by combining the standard and studentised MMD statistics while maintaining the distribution-free property. 
Our methods are referred to as CMMD with variations based on the bias convergence rate $\nicefrac{n_t}{n_e}$ of the test statistic (see Theorems \ref{thm: main_theorem_h0} and \ref{thm: main_theorem_2}). To ensure the ratio decays to $0$ as sample size increases, we choose $\nicefrac{n_t}{n_e}$ in the form of $n_e^{-\beta}$ for $\beta\in[0, 1)$. Specifically, we determine the split ratio $\rho$ for sample size $n$ by solving an optimisation~(Algorithm~\ref{algo: compute_adaptive_split_ratio}) to ensure $\nicefrac{n_t}{n_e} = n_e^{-\beta}$ for $\beta \in \{\nicefrac{1}{3}, \nicefrac{1}{4}, 0\}$, with the constraint $n_t + n_e = n$. We denote our methods as CMMD$(\beta)$ correspondingly. The choice of $\beta = 0$, corresponds to a standard fixed sample-splitting strategy. MMDQ, MMDQ$^\star$, and CMMD share the same epistemic alignment step by minimising the same empirical KCD objective. MMDQ and MMDQ$^\star$ use $n$ samples for testing, rather than $n_t$, following the original ``double-dipping'' approaches in \citet{key_composite_2024,bruck_distribution_2023}.

\textbf{Experimental setup.} In our simulations, let $\bfP_Y$ represent a vector of \( r = 3 \) isotropic Gaussians in 10 dimensions, with means sampled randomly from a 10-dimensional unit sphere. Similarly, let $\bfQ_Y$ be of the same dimension and size but distributed as 10-dimensional Student's t-distributions with 3 degrees of freedom and identical means. For the \emph{specification test}, we simulate the null hypothesis by drawing a total of \( rn \) credal samples $\bfS_Y$ from the distributions in $\bfP_Y$, and generate $n$ samples \( S_X \) from $\bfeta_0^\top \bfP_Y$, where \( \boldsymbol{\eta}_0 \) is uniformly drawn from the simplex \( \Delta_r \). To simulate the alternative hypothesis, we instead generate $S_X$ from $\bfeta_0^\top\bfQ_Y$. For the \emph{inclusion test}, the null hypothesis is simulated by drawing credal samples \( \mathbf{S}_X \) from \( \{ \boldsymbol{\eta}_0^{(i)\top} \mathbf{P}_Y \}_{i=1}^\ell \), with \( \ell = 3 \) and \( \{\boldsymbol{\eta}_0^{(i)}\}_{i=1}^\ell \) uniformly sampled from \( \Delta_r \). The alternative hypothesis is simulated by drawing \( \mathbf{S}_X \) from \( \{ \boldsymbol{\eta}_0^{(i)\top} \mathbf{Q}_Y \}_{i=1}^\ell \). For the \emph{equality test}, under the null hypothesis, credal samples \( \mathbf{S}_X \) and \( \mathbf{S}_Y \) are drawn from \( \mathbf{P}_Y \), while under the alternative hypothesis, samples are drawn from \( \mathbf{Q}_Y \) and \( \mathbf{P}_Y \), respectively. At last for the \emph{plausibility test}, the null hypothesis is simulated with credal samples drawn from $\bfP_Y$ and from $\{P_Y^{(1)}, P_Y^{(2)}, Q_Y^{(3)}\}$, and the alternative is simulated with credal samples drawn from $\bfP_Y$ and $\bfQ_Y$. All algorithms in the experiments use a Gaussian kernel (cf Section~\ref{subsec: two-sample testing}), with bandwidth $\sigma$ determined by the median heuristic~\citep{gretton2012kernel}. All tests are conducted at a significance level of $0.05$. For methods requiring permutation tests, we permute the statistic $500$ times to estimate the critical value. 


\textbf{Analysis.} Figure~\ref{fig: simulation} presents the simulation results for the four tests with rejection rates computed over 500 repetitions and plotted on the y-axis, while the x-axis represents $n$, the sample size \emph{from each extreme point}. Under the null hypothesis, $\text{CMMD}(0)$ shows an inflated Type I error across all tests, except for the inclusion test, which may be attributed to the conservativeness of Bonferroni correction. This Type I inflation persists even as sample size increases, consistent with our theory that a non-decaying bias exists between the statistics based on estimated parameters and the oracle ones. In contrast, $\text{CMMD}(\nicefrac{1}{3})$ and $\text{CMMD}(\nicefrac{1}{4})$ exhibit slight Type I inflation at smaller sample sizes but converge to the correct level as sample size increases, supporting our theory that controlling estimation error is crucial for faster convergence relative to the test statistic's decay. MMDQ follows a similar Type I convergence, while MMDQ$^\star$ exhibits strong conservativeness.

In the experiments where the alternative hypothesis is true (right panels), $ \text{CMMD}(0) $ consistently has the highest rejection rates across all tests, partly because it uses a higher number $n_t$ of samples for testing, but more importantly due to its inflated Type I error. Meanwhile, among the tests which reach the desired Type I control, $ \text{CMMD}(\nicefrac{1}{4}) $ consistently outperforms $ \text{CMMD}(\nicefrac{1}{3}) $, as expected due to having more testing samples (with a caveat that it also converges more slowly to the correct Type I error level). Both MMDQ and MMDQ$^\star$ show lower power despite using a larger number of testing samples ($n$) compared to CMMD approaches ($n_t$), a common drawback of distribution-free studentized statistics compared to permutation-based methods. Overall, we recommend using CMMD$(\nicefrac{1}{4})$, which offers the highest power while maintaining theoretical guarantees for correct Type I error control.

%% file: sections/06_discussion.tex
\section{Discussion}
\label{sec: discussion}
We conclude by highlighting the potential uses of credal tests in machine learning. In domain generalisation~\citep{singh2024domain, caprio_credal_2024}, credal sets often represent unknown deployment distributions. During deployment, when a small amount of data is available, our specification test offers a statistically valid way to verify these assumptions, reducing the risk of harm from incorrect or unverified models. Additionally, as discussed in Section~\ref{sec: related_work_new}, specification tests open new possibilities for nonparametric mixture model testing, enabling scientists to collect samples directly from mixture component distributions without relying on parametric assumptions prone to misspecification. Next, we anticipate the inclusion test will be particularly useful in uncertainty quantification research, where no consensus exists on measuring the uncertainty of credal sets~\citep{sale2023volume, hofman2024quantifying}. Nonetheless, a desirable property of any such method is monotonicity, where if $\cC_X \subseteq \cC_Y$, then $\cC_Y$ is more epistemically uncertain than $\cC_X$. Our inclusion test facilitates this comparison without predefined quantification metrics. The equality test, similar to the standard two-sample test, helps detect treatment effects or significant changes even in the presence of uncertainty. Finally, the plausibility test acts as a distributionally robust two-sample test, determining whether a consensus exists despite ambiguity. Failure to reject the test suggests common ground and encourages further data collection or review of evidence, whereas rejection indicates strong evidence of no consensus. We also envision our credal tests serving as a foundation for independence testing between credal sets, a challenging problem due to the multiple definitions of (conditional) independence for credal sets~\citep{cozman2008sets} and the absence of statistical methods for testing these relationships---similar to how two-sample tests underpin nonparametric independence testing.

\textbf{Future work.} There remain several areas of study which warrant further investigation. One direction is to explore other methods for generating credal sets~\citep{augustin_introduction_2014}. Another is kernel selection, as heterogeneity within credal sets complicates kernel choice so as to maximize test power. Advanced multiple sample techniques~\citep{guo2024rank} may increase the test power for credal tests. Finally, it is of interest to develop a data-driven mechanism for selectively adjusting the split ratio to balance Type I error control with increased test power.




In addition to advancing two-sample testing with nuisance parameters, our credal tests promote the broader integration of a modeller's epistemic uncertainty into scientific practices, leading to conclusions that are more objective, robust, and ultimately more credible.


\input{sections/appendix/D_further_remarks}

%% file: sections/appendix/D_further_remarks.tex
\section{Further Discussion}
\label{appendix_Sec: further_remarks}

In this section, we address additional important concepts about credal hypothesis testing.

\subsection{Limitations of Credal Sets}

Probability theory is the de facto mathematical formulation for modelling uncertainty and randomness in most scientific disciplines. However, researchers have increasingly recognised the limitations of relying on a single probability distribution to capture the diverse forms of uncertainty inherent in complex systems. Generalisations of probability theory, including Dempster-Shafer theory~\citep{shafer1992dempster}, interval-valued probabilities~\citep{kyburg1998interval}, the Choquet integral~\citep{choquet1953}, upper-lower probabilities~\citep{troffaes2014lower}, and comparative probabilities~\citep{walley1979varieties}, represent efforts to address these limitations. Despite their differences, these theories share a common feature: a credal set, which is a closed, convex set of probability distributions that serves as a unifying characterization.\footnote{For an excellent overview of uncertainty models beyond using a single probability distribution, we recommend reading \citet{walley2000towards}, \citet{hullermeier_aleatoric_2021}, and \citet{cuzzolin2024uncertainty}.}

This work focuses entirely on a finitely generated credal set, defined as the convex hull of a discrete set of probabilities, for which we have access to samples, to represent our epistemic ignorance. There are strong theoretical justifications from both fields of formal epistemology and mathematics supporting this representation of epistemic uncertainty \citep{lewis1980subjectivist,walley1991statistical}. For example, while a single probability, with a few additional axioms, represents a complete preference \citep[Chap. 13]{Fishburn70:utility}, the credal set offers a generalized representation of \emph{partial} preference \citep{giron_quasi-bayesian_1980,seidenfeld_shared_1989,walley1991statistical}, reflecting a rational agent’s partial ignorance. Moreover, any two credal sets with the same convex hull represent the same partial preference. Nevertheless, we highlight important limitations that must be considered when adopting this approach in practice. 

One limitation is the paradox of increasing epistemic uncertainty even as more evidence becomes available. Consider a finitely generated credal set $\cC_X = \operatorname{CH}(P_1, \dots, P_{\ell-1})$. Now, suppose we gather an additional source of information, $P_{\ell}$, and wish to incorporate it into the existing set of evidence to represent epistemic ignorance as a credal set. In this case, the credal set will either remain the same if $P_{\ell}$ lies within the set $\cC_X$ or expand if $P_{\ell}$ is linearly independent of $P_1, \dots, P_{\ell-1}$. This is counterintuitive to the usual notion of epistemic uncertainty, which typically decreases as more information is gathered. Furthermore, this property could result in credal sets being disproportionately “stretched” by a single ``outlier'' distribution.

To address this issue, we recommend that practitioners apply subjective judgment to assess the quality of such distributions, using techniques like distributional outlier detection, such as One-Class Support Measure Machines~\citep{muandet2013one}, a generalization of One-Class SVM at the distributional level. Additionally, when the sample sizes from each extreme point differ significantly, it is important to consider whether a distribution with very few samples should be included as part of the objective representation of epistemic ignorance.

\subsection{Interpretation of Probabilities in Credal Tests}

Classical probability theory, grounded in the Kolmogorov axioms~\citep{Kolmogorov60:Foundations,kolmogorov2018foundations}, offers a formal mathematical framework for analysing and representing uncertainty and the likelihood of events. However, the interpretation of probability remains flexible, leading to the emergence of different schools of thought. For an excellent overview of these interpretations, interested readers are encouraged to consult \citet{hajek2002interpretations}. Below, we briefly discuss different interpretations of probability that are relevant to understanding and interpreting results from credal tests.

\paragraph{Different types of probabilities.} There are different categorisations of probabilities. In machine learning, probabilities are often classified as either aleatoric, representing inherent randomness in systems, or epistemic, representing uncertainty due to limited knowledge~\citep{hullermeier_aleatoric_2021}. Another common distinction is between physical probability and belief probability. The physical probability, also referred to as risk or chance, describes objective phenomena and should be invariant to an observer’s perspective (unless we venture into quantum mechanics\footnote{The wavefunction of subatomic particles describes the probabilities of possible outcomes for their position, momentum, and other physical properties. However, when an observation is made, the probability ``collapse'' into one specific state. This is known as the observer effect.}). The frequentist interpretation of probability belongs to this category, where probability is derived from counting event occurrences over infinite repetitions. The classical probability, which calculates the ratio of favorable outcomes to total possible outcomes, is also a form of physical probability.

\paragraph{Belief probability.} In contrast, the belief probability reflects an agent’s degree of confidence in a particular proposition. Since belief is subjective, it can vary between individuals, making purely subjective probabilities challenging to study. However, if we assume that humans are rational agents who follow certain logical principles, we can impose constraints on how their beliefs should be structured. This leads to the well-known Dutch book argument, which demonstrates that rational agents engaged in betting must adhere to the Kolmogorov axioms to avoid sure losses~\citep{de1937foresight}. This result brings us back to probability theory as a framework for modeling rational credence, a concept often referred to as the structural norm in epistemic theories of probability. By imposing additional constraints on belief formation, we arrive at the concept of objective belief probability, which forms the foundation of the core interpretation of probability in this paper.

\paragraph{Objective belief probability.} A key principle in objective belief probability concerns how to calibrate one’s belief strength when presented with relevant evidence. Consider the following example: A person tells you, “It will rain tomorrow with probability $P_1(\text{rain}) = a$,” and you are asked to express your belief about the occurrence of this event. If you are agnostic about the problem, there is no reason not to calibrate your belief probability $P_B(\text{rain})$ to match $a$. However, if another person offers a different piece of evidence and claims, “It will rain tomorrow with probability $P_2(\text{rain}) = b$,” and you remain agnostic about both the problem and the credibility of the sources, how should you adjust your belief? One natural way is to express your belief of raining tomorrow as anything in between $a$ and $b$, i.e., assuming $a<b$, then $P_B(\text{rain})\in[a, b]$. Generalising this concept of probability interval naturally leads to the idea advocated in \citet{lewis1980subjectivist} and \citet{williamson_defence_2010} that objective epistemic ignorance should be represented as a convex hull of distributions, aka credal set.

\paragraph{Interpretation of credal tests.} The aforementioned concepts can be applied to the interpretation of credal tests. At its core, hypothesis testing is traditionally grounded in frequentist probability, relying on the concept of repetition to define quantities like the $p$-value, which is interpreted as:

\begin{center}
\emph{If I were to collect samples infinitely many times, how often would I observe a test statistic—computed to reflect certain desirable properties related to the \textbf{aleatoric variations} (probabilities) in the null hypothesis—as extreme as the one I have seen?}
\end{center}

This interpretation does not involve belief probability. In the case of credal tests, however, the interpretation of the $p$-value becomes:

\begin{center}
\emph{If I were to collect samples infinitely many times, and each time I express my epistemic ignorance through a credal set, how often would I observe a test statistic—computed to reflect certain desirable properties related to my \textbf{belief probabilities} in the null hypothesis—as extreme as the one I have seen?}
\end{center}

This interpretation clarifies the role of aleatoric variation (the observed samples) and belief disposition (epistemic ignorance represented as credal sets), which underpins the title of our work: \emph{Credal Two-sample Tests of Epistemic Ignorance}. To the best of our knowledge, the combination of using the frequentist interpretation of probability to reason about belief probability within a testing framework is rarely discussed. We believe this approach could open new avenues for future research at the intersection of these concepts.

%% file: sections/appendix/00_header.tex



\newpage

The appendix contains algorithmic details, formal proofs, and additional experiments, it is organised as follows.

\vspace{-1cm}

\renewcommand{\contentsname}{}
\tableofcontents

%% file: sections/appendix/A_algorithm.tex
\newpage
\section{Algorithms}
\label{appendix_sec: credal_algorithms}

This section provides the details of algorithms that are used as part of the specification, equality, and plausibility tests as well as their computational complexity.

\subsection{Algorithms}
The following algorithms are required for credal testing.

\begin{algorithm}
\caption{\texttt{split\_data}}
\label{algo: split_data}
\begin{algorithmic}[1]
    \Require Multiple samples $\bfS_X, \bfS_Y$, estimation-test split ratio $\rho$
    \For{$S^{(j)}_X$ in $\bfS_X$}
        \State $n_j \leftarrow$ the sample size for $S^{(j)}_X$
        \State Randomly assign $\lfloor n_j \times \rho \rfloor$ samples to $S^{(j),e}_X$ and the remaining samples to $S^{(j),t}_X$.
    \EndFor
    \For{$S^{(j)}_Y$ in $\bfS_Y$}
        \State $n_j \leftarrow$ the sample size for $S^{(j)}_Y$
        \State Randomly assign $\lfloor n_j \times \rho \rfloor$ samples to $S^{(j),e}_Y$ and the remaining samples to $S^{(j),t}_Y$.
    \EndFor
    \State Set $\bfS^e_X = \{S_X^{(j),e}\}_{j=1}^\ell$, $\bfS^t_X = \{S_X^{(j),t}\}_{j=1}^\ell$, $\bfS^e_Y = \{S_Y^{(j),e}\}_{j=1}^r$, $\bfS^t_Y = \{S_Y^{(j),t}\}_{j=1}^r$
    \State \Return $\bfS^e_X, \bfS^e_Y, \bfS^t_X, \bfS^t_Y$
\end{algorithmic}
\end{algorithm}

\begin{algorithm}
    \caption{\texttt{redraw\_samples}}
    \label{algo: redraw_samples}
    \begin{algorithmic}[1]
    \Require Multiple samples $\bfS_X, \bfS_Y$, convex weights $\bflambda, \bfeta$
    \State Initialise $\tilde{S}_{X, \bflambda} = \{\}, \tilde{S}_{Y, \bfeta} = \{\}$
    \State $n_X \leftarrow$ the minimum number of samples across $S_X$ in $\bfS_X$.
    \While{$|\tilde{S}_{X, \bflambda}| < n_X$}
    \State Draw $j \sim \operatorname{Multinomial}(\bflambda)$
    \State Draw $X^{(j)}$ from $S_X^{(j)}$ at random
    \State $\tilde{S}_{X, \bflambda} \leftarrow \tilde{S}_{X, \bflambda} \cup \{X^{(j)}\}$
    \State $S_X^{(j)} \leftarrow S_X^{(j)} \backslash \{X^{(j)}\}$
    \EndWhile
    \State $n_Y \leftarrow$ the minimum number of samples across $S_Y$ in $\bfS_Y$.
    \While{$|\tilde{S}_{Y, \bfeta}| < n_Y$}
    \State Draw $j \sim \operatorname{Multinomial}(\bfeta)$
    \State Draw $Y^{(j)}$ from $S_Y^{(j)}$ at random
    \State $\tilde{S}_{Y, \bfeta} \leftarrow \tilde{S}_{Y, \bfeta} \cup \{Y^{(j)}\}$
    \State $S_Y^{(j)} \leftarrow S_Y^{(j)} \backslash \{Y^{(j)}\}$
    \EndWhile
    \end{algorithmic}
\end{algorithm}

The procedure in Algorithm~\ref{algo: redraw_samples} ensures that we obtain independently and identically distributed samples from $\bflambda^\top \bfP_X$ and $\bfeta^\top \bfP_Y$ based on bootstrapping. Importantly, the sampling without replacement approach allows us to avoid obtaining the same observation, breaking the ``iid-ness'' of the redrawn samples.

\begin{algorithm}
    \caption{\texttt{kernel\_two\_sample\_test}}
    \begin{algorithmic}[1]
        \Require Samples $S_X$, $S_Y$, kernel $k$, number of simulated statistics $B$, test level $\alpha$
        \State Compute $M_0 = \operatorname{MMD}^2(S_X, S_Y)$ with kernel $k$
        \For{$b$ in $1,\dots, B$}
        \State Permute samples $S_X, S_Y$ to obtain $S_X^{(b)}, S_Y^{(b)}$
        \State Compute $M_b \leftarrow \operatorname{MMD}^2(S_X^{(b)}, S_Y^{(b)})$ with kernel $k$
        \EndFor
        \State Compute $p \leftarrow \frac{1}{B + 1} \left(\sum_{b=1}^{B} \mathbf{1}[M_b \geq M_0] + 1\right)$
        \State \Return ``reject'' if $p < \alpha$, otherwise \Return ``fail to reject''
    \end{algorithmic}
    \label{algo: kernel_two_sample_test}
\end{algorithm}

To implement Step 3 in Algorithm~\ref{algo: kernel_two_sample_test}, we employed the wild bootstrap permutation procedure for computational efficiency; see \citet[Section 3.2.2]{schrab2023mmd} for further details.

\begin{algorithm}
    \caption{\texttt{compute\_adaptive\_split\_ratio}}
    \label{algo: compute_adaptive_split_ratio}
    \begin{algorithmic}[1]    
    \Require Number of samples $n$, power $\beta$, tolerance $\epsilon$, maximum iteration $\zeta$
    \State Compute $n_e \leftarrow \lfloor n/2\rfloor$
    \For{$t=1,\dots,\zeta$}
    \State Compute $n_e' \leftarrow n_e - \frac{n_e + n_e^{(1+\beta)} - n}{1 + (1+\beta)n_e^{\beta}}$
    \If{$|n_e' - n_e| <\epsilon$}
    \State \Return split ratio $\frac{n_e}{n}$
    \EndIf
    \State Update $n_e \leftarrow n_e'$
    \EndFor
    \State \Return "error, the solution did not converge"
    \end{algorithmic}
\end{algorithm}

\subsection{Time Complexity}
We provide a brief description of the runtime complexities of our credal testing algorithms.
\begin{itemize}
    \item \textbf{Specification test.} Specification test consists of two stages: estimation (epistemic alignment) and testing. For the estimation stage, we are solving a convex quadratic program, which can be solved using the interior point method. Since we need to compute the KMEs for each distribution in $\bfP_Y$ and $P_X$, we need $O((1+r) n_e^2)$ runtime complexity. Solving the convex quadratic program per iteration often requires solving a system of linear equations with $r$ variables, which has $O(r^3)$ complexity, and it often converges at $O(\sqrt{r})$ steps. Overall the estimation stage has time complexity of ${O}((1+r)n_e^2 + r^{3.5})$. For the hypothesis testing part, the complexity is $O(B \times n_t^2)$ where $B$ is the number of simulated statistics we generate in the procedure. Therefore, combining the complexity of the two stages, we have the total runtime complexity of
    \begin{align*}
        {O}\left((1+r)n_e^2 + r^{3.5} + B \times n_t^2\right).
    \end{align*}
    \item \textbf{Inclusion test.} Since the inclusion test requires running multiple specification tests, the runtime complexity follows straightforwardly as
    \begin{align*}
        {O}\left(\ell\times \left((1+r)n_e^2 + r^{3.5} + B \times n_t^2\right)\right).
    \end{align*}
    \item \textbf{Equality test.} The complexity of the equality test, which requires running two times the inclusion tests, is
    \begin{align*}
        {O}\left(\ell\times \left((1+r)n_e^2 + r^{3.5} + B \times n_t^2\right) + r \times \left((1+\ell)n_e^2 + \ell^{3.5} + B \times n_t^2\right)\right).
    \end{align*}
    \item \textbf{Plausibility test.} The plausibility test requires solving an iterative biconvex minimisation problem. Let $D$ be the number of iterations, then the overall complexity of the algorithm is
    \begin{align*}
        { O} \left((\ell + r)n_e^2 + D \times (r^{3.5} + \ell^{3.5}) + B\times n_t^2\right).
    \end{align*}
\end{itemize}

%% file: sections/appendix/B_proofs.tex
\newpage
\section{Proofs}
\label{appendix_Sec: proofs}

This section contains the proofs of the main results presented in the main paper. Before proceeding to the proofs, we restate the assumptions used throughout the paper:

\begin{assumptiontrick}
\label{assumption 0}
The extreme points of the credal set are linearly independent.
\end{assumptiontrick}
\begin{assumptiontrick}
\label{assumption: 1}
    The kernel $k$ is continuous, bounded, positive definite, and characteristic.
\end{assumptiontrick}
\begin{assumptiontrick}
\label{assumption: 2}
    There exists some $n_0\in\NN$, such that for $n_t > n_0$, the function $\cL_{n_t}: \bflambda,\bfeta \mapsto \operatorname{MMD}^2(\tilde{S}_{X,\bflambda}, \tilde{S}_{Y,\bfeta})$ is continuous over $\Delta_\ell \times \Delta_r$ and twice continuously differentiable over its interior. Furthermore, the gradient $\nabla \cL_{n_t}$ satisfies Lipschitz continuity and a technical condition $\|\nabla \cL_{n_t} - \nabla L\|_{\infty} \leq C' \|\cL_{n_t} - L\|_{\infty}$ for some constant $C'$.
\end{assumptiontrick}
\begin{assumptiontrick}
    \label{assumption: 3}
    The Schur complement $M_{XX} - M_{XY}M_{YY}^{-1}M_{YX}$ is positive definite.
\end{assumptiontrick}

Assumption~\ref{assumption 0} facilitates our theoretical analysis, but even if this assumption is violated, our credal tests remain valid (c.f. Appendix~\ref{appendix subsub linearly dependent}). Assumption~\ref{assumption: 1} imposes regularity conditions on the RKHS and ensures the MMD serves as a valid divergence measure. These conditions are satisfied by commonly used kernels, such as the Gaussian kernel. Assumption~\ref{assumption: 2}, the smoothness condition, allows us to explicitly analyse the relationship between the estimation error and the test statistic. The smoothness assumption may be less reliable for small sample sizes, but it holds as sample size increases since $\cL_{n_t}$ (and $L_{n_e}$) converge uniformly to the population KCD, which is itself continuous over $\Delta_\ell \times \Delta_r$ and differentiable in its interior. The gradient conditions on $\nabla \cL_{n_t}$ pose a technical challenge when verifying them for the sample-based estimator $\cL_{n_t}$. However, these conditions can be confirmed for $L_{n_e}$ (see Proposition~\ref{prop: gradient_uniform_convergence} in Appendix~\ref{appendix_Sec: proofs}), as we have the analytical form of $L_{n_e}$. Given that both $L_{n_e}$ and $\cL_{n_t}$ are uniformly consistent estimators of $L$ (see Proposition~\ref{prop: uniform_convergence}), it is reasonable to extend this assumption to $\cL_{n_t}$ as well. Assumption~\ref{assumption: 3} is a technical assumption to analyse the convergence rate of KCD optimisers in the plausibility test.


\subsection{Proof for Proposition~\ref{prop: credal discrepancy}}

\CredalDiscrepancy*

\begin{proof}
    Let $d:\cP(\cX) \times \cP(\cX) \to \RR_{\geq 0}$  be a statistical divergence, i.e., $d(P, Q) = 0$ if and only if distributions $P, Q$ are equivalent. 
    
    For inclusion, recall that $\operatorname{Inc}(\cC_X,\cC_Y)$ is defined as $\sup_{P_X\in\cC_X}\inf_{P_Y\in\cC_Y} d(P_X, P_Y)$. If $\cC_X \subseteq \cC_Y$, then for all $P_X\in\cC_X$, we have $P_X \in \cC_Y$. Hence, for any $P_X\in\cC_Y$, there exists $P_Y\in\cC_Y$ such that $d(P_X, P_Y) = 0$, thus $\sup_{P_X\in\cC_X}\inf_{P_Y\in\cC_Y}d(P_X, P_Y) = 0$. On the other hand, if $\operatorname{Inc}(\cC_X, \cC_Y) = 0$, but $\cC_X \not\subseteq \cC_Y$, there exists a $P_X\in\cC_X$ where $P_X\not\in\cC_Y$. Consequently, there cannot be $P_Y\in\cC_Y$ such that $d(P_X, P_Y) = 0$. This implies that $\sup_{P_X\in\cC_X}\inf_{P_Y\in\cC_Y} d(P_X, P_Y) \neq 0$ because there is at least one $P_X\in\cC_Y$ for which the divergence is non-zero.

    For equality, the argument is straightforward. Two sets $\cC_X,\cC_Y$ are equal if and only if $\cC_X \subseteq \cC_Y$ and $\cC_Y \subseteq \cC_X$. To check these two conditions, we only need to show both $\operatorname{Inc}(\cC_X,\cC_Y) = 0$ and $\operatorname{Inc}(\cC_Y,\cC_X) = 0$, which is implied by $\max\left(\operatorname{Inc}(\cC_X,\cC_Y), \operatorname{Inc}(\cC_Y,\cC_X)\right) = 0$.

    For set intersection, recall that $\operatorname{Int}(\cC_X, \cC_Y) = \inf_{P_X\in\cC_X}\inf_{P_Y\in\cC_Y} d(P_X, P_Y)$. If $\cC_X \cap \cC_Y \neq \emptyset$, then there exists $P_X\in\cC_X, P_Y\in\cC_Y$ such that $P_X = P_Y$, implying that $\operatorname{Int}(\cC_X, \cC_Y) = \inf_{P_X\in\cC_X}\inf_{P_Y\in\cC_Y} d(P_X, P_Y) = 0$. If $\operatorname{Int}(\cC_X, \cC_Y) = 0$, but $\cC_X \cap \cC_Y = \emptyset$, then there exists no $P_X\in\cC_X, P_Y\in\cC_Y$ such that $d(P_X, P_Y) = 0$. However, since the credal set is closed, all infimums can be attained. Hence, the fact that $\operatorname{Int}(\cC_X, \cC_Y) = 0$ implies the existence of some $P_X\in\cC_X, P_Y\in\cC_Y$ that are equal, leading to a contradiction.
\end{proof}

\subsection{Proof for Proposition~\ref{prop: credal_mmd}}

\KernelCredalDiscrepancy*
\begin{proof}
    Recall that $\cC_X$ and $\cC_Y$ are finitely generated credal sets, i.e., they are convex hulls of discrete sets of probability distributions $\bfP_X$ and $\bfP_Y$. Hence, for any $P_X\in\cC_X$ and $P_Y\in\cC_Y$, there exists $\bflambda \in \Delta_\ell$ and $\bfeta \in \Delta_\ell$ where
    $$
    P_X = \bflambda^\top \bfP_X, \quad P_Y = \bfeta^\top \bfP_Y.
    $$
    Next, as discussed in Section~\ref{subsec: two-sample testing}, the squared MMD can be expressed as the RKHS distance between the kernel mean embedding of the distribution of interests, that is,
    \begin{align*}
        \operatorname{MMD}^2(P_X, P_Y) = \|\mu_{P_X} - \mu_{P_Y}\|_{\cH_k}^2
    \end{align*}
    where $\mu_{P_X}, \mu_{P_Y}$ are kernel mean embeddings of $P_X, P_Y$, defined as,
    \begin{align*}
        \mu_{P_X} = \int_\cX k(x,\cdot) \,dP_X, \quad \mu_{P_X} = \int_\cX k(y,\cdot) \,dP_Y.
    \end{align*}
    Focusing on $\mu_{P_X}$ for now, we then have,
    \begin{eqnarray*}
        \mu_{P_X}   &=& \int_\cX k(x,\cdot) \,dP_X \\
                    &=& \int_\cX k(x, \cdot) \,d\left(\bflambda^\top \bfP_X\right) \\
                    &=& \sum_{j=1}^\ell \lambda_j \int_\cX k(x,\cdot) \,dP^{(j)}_X\\
                    &=& \sum_{j=1}^\ell \lambda_j \mu_{P^{(j)}_X} \\
                    &=& \bflambda^\top \vec{\mu}_{\bfP_X}
    \end{eqnarray*}
    where $\vec{\mu}_{\bfP_X} := [\mu_{P^{(1)}_X},\ldots, \mu_{P^{(\ell)}_X}]^\top$ denotes the vector of kernel mean embeddings of the extreme point distributions. Similarly, we can express $\mu_{P_Y} = \bfeta^\top \vec{\mu}_{\bfP_Y}$. In the second equality, we used linearity of integral with respect to the measures, and in the third equality, we swapped integral with the summation, which is allowed since the kernel mean embeddings at each corner distribution are well defined due to the boundedness of the kernel. This characterisation of kernel mean embeddings for $\mu_{P_X}$ and $\mu_{P_Y}$ allows us to express the square of the MMD as
    \begin{eqnarray*}
        \operatorname{MMD}^2(P_X,P_Y)   &=& L(\bflambda,\bfeta) \\
                                        &=& \|\mu_{P_X} - \mu_{P_Y}\|_{\cH_k}^2 \\
                                        &=& \|\bflambda^\top \vec{\mu}_{\bfP_X} - \bfeta^\top \vec{\mu}_{\bfP_Y}\|_{\cH_k}^2 \\
                                        &=& \langle \bflambda^\top \vec{\mu}_{\bfP_X}, \bflambda^\top \vec{\mu}_{\bfP_X}\rangle + \langle \bfeta^\top \vec{\mu}_{\bfP_Y}, \bfeta^\top \vec{\mu}_{\bfP_Y}\rangle - 2 \langle \bflambda^\top \vec{\mu}_{\bfP_X}, \bfeta^\top \vec{\mu}_{\bfP_Y}\rangle \\
                                        &=& \sum_{j=1}^\ell \sum_{j'=1}^\ell \lambda_j\lambda_{j'} \langle \mu_{P^{(j)}_X}, \mu_{P^{(j')}_X}\rangle + \sum_{j=1}^r \sum_{j'=1}^r \eta_j\eta_{j'} \langle \mu_{P^{(j)}_Y}, \mu_{P^{(j')}_Y}\rangle 
                                        -2 \sum_{j=1}^\ell \sum_{j'=1}^r \lambda_j\eta_{j'} \langle \mu_{P^{(j)}_X}, \mu_{P^{(j')}_Y}\rangle \\
                                        &=& \bflambda^\top \bfM_{XX} \bflambda + \bfeta^\top\bfM_{YY} \bfeta -2 \bflambda^\top \bfM_{XY} \bfeta
    \end{eqnarray*}
    where the $i,j$ entry of $\bfM_{XX}$ is $\langle \mu_{P^{(i)}_X}, \mu_{P^{(j)}_X}\rangle = \EE[k(X^{(i)}, X^{(j)})]$, and $\bfM_{XY}, \bfM_{YY}$ defined analogously. This concludes the proposition.
\end{proof}

\subsection{Proof for Proposition~\ref{prop: uniform_convergence}}

\KCDUniformConvergence*

\begin{proof}[Proof for Proposition~\ref{prop: uniform_convergence}]
    Starting with $L_{n_e}$, recall that the empirical KCD $L_{n_e}(\bflambda, \bfeta)$ is expressed as
    \begin{align*}
        L_{n_e}(\bflambda, \bfeta) = \bflambda^\top \widehat{\bfM}_{XX} \bflambda + \bfeta^\top\widehat{\bfM}_{YY}\bfeta -2 \bflambda^\top \widehat{\bfM}_{XY}\bfeta.
    \end{align*}
    \paragraph{Uniform convergence of $L_{n_e}$.} For fixed $\bflambda, \bfeta$, the difference between the empirical and population level objectives can then be written as
    \begin{align*}
        |L_{n_e}(\bflambda, \bfeta) - L(\bflambda, \bfeta)| 
            &= \left| \bflambda^\top \widehat{\bfM}_{XX} \bflambda + \bfeta^\top\widehat{\bfM}_{YY}\bfeta -2 \bflambda^\top \widehat{\bfM}_{XY}\bfeta - \left(\bflambda^\top \bfM_{XX} \bflambda + \bfeta^\top\bfM_{YY} \bfeta -2 \bflambda^\top \bfM_{XY} \bfeta\right)\right| \\
            &= \left| \bflambda^\top (\widehat{\bfM}_{XX} - \bfM_{XX}) \bflambda + \bfeta^\top (\widehat{\bfM}_{YY} - \bfM_{YY}) \bfeta -2 \bflambda^\top (\widehat{\bfM}_{XY} - \bfM_{XY}) \bfeta\right| \\
            &\leq \left|\bflambda^\top (\widehat{\bfM}_{XX} - \bfM_{XX}) \bflambda \right| + \left|\bfeta^\top (\widehat{\bfM}_{YY} - \bfM_{YY}) \bfeta\right| + 2 \left|\bflambda^\top (\widehat{\bfM}_{XY} - \bfM_{XY}) \bfeta\right|.
    \end{align*}

Next, notice that,
\begin{align*}
    \left|\bflambda^\top (\widehat{\bfM}_{XX} - \bfM_{XX}) \bflambda \right| &= \left|\sum_{i=1}^\ell \sum_{j=1}^\ell \lambda_i\lambda_j (\langle \hat{\mu}_{P_X^{(i)}}, \hat{\mu}_{P_X^{(j)}}\rangle - \langle {\mu}_{P_X^{(i)}}, {\mu}_{P_X^{(j)}}\rangle)\right| \\
    &\leq \sum_{i=1}^\ell \sum_{j=1}^\ell \lambda_i\lambda_j \left|\langle \hat{\mu}_{P_X^{(i)}}, \hat{\mu}_{P_X^{(j)}}\rangle - \langle {\mu}_{P_X^{(i)}}, {\mu}_{P_X^{(j)}}\rangle\right| \\
    &= \sum_{i=1}^\ell \sum_{j=1}^\ell \lambda_i\lambda_j \left|\langle \hat{\mu}_{P_X^{(i)}}, \hat{\mu}_{P_X^{(j)}}\rangle - \langle \hat{\mu}_{P_X^{(i)}}, \mu_{P_X^{(j)}}\rangle + \langle \hat{\mu}_{P_X^{(i)}}, \mu_{P_X^{(j)}}\rangle - \langle {\mu}_{P_X^{(i)}}, {\mu}_{P_X^{(j)}}\rangle\right| \\
    &\leq \sum_{i=1}^\ell \sum_{j=1}^\ell \lambda_i\lambda_j \left(\left|\langle \hat{\mu}_{P_X^{(i)}}, \hat{\mu}_{P_X^{(j)}}\rangle - \langle \hat{\mu}_{P_X^{(i)}}, \mu_{P_X^{(j)}}\rangle \right| + \left|\langle \hat{\mu}_{P_X^{(i)}}, \mu_{P_X^{(j)}}\rangle - \langle {\mu}_{P_X^{(i)}}, {\mu}_{P_X^{(j)}}\rangle\right|\right) \\
    &= \sum_{i=1}^\ell \sum_{j=1}^\ell \lambda_i\lambda_j \left(\left|\langle \hat{\mu}_{P_X^{(i)}}, \hat{\mu}_{P_X^{(j)}} - \mu_{P_X^{(j)}}\rangle \right| + \left|\langle \hat{\mu}_{P_X^{(i)}}-{\mu}_{P_X^{(i)}}, \mu_{P_X^{(j)}}\rangle\right|\right) \\
    & \stackrel{(\clubsuit)}{\leq} \sum_{i=1}^\ell \sum_{j=1}^\ell \lambda_i\lambda_j \left( \|\hat{\mu}_{P_X^{(i)}}\|_{\cH_k} \|\hat{\mu}_{P_X^{(j)}} - \mu_{P_X^{(j)}}\|_{\cH_k}  + \|\hat{\mu}_{P_X^{(i)}}-{\mu}_{P_X^{(i)}}\|_{\cH_k} \|\mu_{P_X^{(j)}}\|_{\cH_k} \right)\\
    & \stackrel{(\spadesuit)}{\leq} \sum_{i=1}^\ell \sum_{j=1}^\ell \lambda_i\lambda_j \left(c  \|\hat{\mu}_{P_X^{(j)}} - \mu_{P_X^{(j)}}\|_{\cH_k} + c \|\hat{\mu}_{P_X^{(i)}} - \mu_{P_X^{(i)}}\|_{\cH_k}\right) \\
    &= {O}\left(\frac{1}{\sqrt{n_e}}\right)
\end{align*}
where we used Cauchy-Schwarz inequality in $(\clubsuit)$, and in $(\spadesuit)$ we used Assumption~\ref{assumption: 1}, which states that the kernel is bounded, implying that any kernel mean embedding is bounded by some constant $c$. The last step follows from the standard result regarding convergence of empirical kernel mean embedding to its population counterpart from \citep[Theorem 3.4]{muandet2017kernel}. It is important to note that this bound is not affected by the coefficients $\lambda_i,\lambda_j$ because they are bounded above by $1$. Similar results for $|\bflambda^\top (\widehat{\bfM}_{XY} - \bfM_{XY})\bfeta|$ and $|\bfeta^\top (\widehat{\bfM}_{YY} - \bfM_{YY})\bfeta|$ can be proven analogously.

Using this result, we can then bound the supremum difference between the empirical KCD and its population counterpart as,
\begin{eqnarray*}
    \sup_{\bflambda\in\Delta_\ell, \bfeta\in\Delta_r}|L_{n_e}(\bflambda, \bfeta) - L(\bflambda, \bfeta)| &\leq& \sup_{\bflambda\in\Delta_\ell, \bfeta\in\Delta_r}\Big(\left|\bflambda^\top (\widehat{\bfM}_{XX} - \bfM_{XX}) \bflambda \right| + \left|\bfeta^\top (\widehat{\bfM}_{YY} - \bfM_{YY}) \bfeta\right| \\
        && \qquad\qquad\qquad + 2 \left|\bflambda^\top (\widehat{\bfM}_{XY} - \bfM_{XY}) \bfeta\right|\Big) \\
    &=& {O}\left(\frac{1}{\sqrt{n_e}}\right).
\end{eqnarray*}

Since the bound holds uniformly across all possible $\bflambda\in\Delta_\ell,\bfeta\in\Delta_r$, we have established the uniform convergence of the objective and the corresponding convergence rate.

\paragraph{Uniform convergence of $\cL_{n_t}$.} Let $\vec{\mu}_{\bfP_X} :=[\mu_{P_X^{(1)}},\dots, \mu_{P_X^{(\ell)}}]$ for the credal samples $\bfP_X$ and similarly $\vec{\mu}_{\bfP_Y}$ as the vector of kernel mean embeddings for $\bfP_Y$. Then, we have
\begin{eqnarray*}
    |\cL_{n_t}(\bflambda, \bfeta) - L(\bflambda,\bfeta)| &=& |\|\hat{\mu}_{\bflambda^\top \bfP_X} - \hat{\mu}_{\bfeta^\top \bfP_Y} \|_{\cH_k}^2 - \|\bflambda^\top \vec{\mu}_{\bfP_X} - \bfeta^\top \vec{\mu}_{\bfP_Y} \|_{\cH_k}^2 |\\
    &\stackrel{(\clubsuit)}{=}& |\|\hat{\mu}_{\bflambda^\top \bfP_X} - \hat{\mu}_{\bfeta^\top \bfP_Y} \|_{\cH_k}^2 - \| \mu_{\bflambda^\top\bfP_X} - \mu_{\bfeta^\top\bfP_Y} \|_{\cH_k}^2 |\\
    &=& |\|\hat{\mu}_{\bflambda^\top \bfP_X} - \hat{\mu}_{\bfeta^\top \bfP_Y} \|_{\cH_k}^2 - \|\mu_{\bflambda^\top \bfP_X} - \hat{\mu}_{\bfeta^\top \bfP_Y} \|_{\cH_k}^2 \\ 
    && \quad + \|\mu_{\bflambda^\top \bfP_X} - \hat{\mu}_{\bfeta^\top \bfP_Y} \|_{\cH_k}^2    
    -\| \mu_{\bflambda^\top\bfP_X} - \mu_{\bfeta^\top\bfP_Y} \|_{\cH_k}^2|.
\end{eqnarray*}
In $(\clubsuit)$, we replace $\bflambda^\top \vec{\mu}_{\bfP_X}$ with $\mu_{\bflambda^\top\bfP_X}$ because
\begin{align*}
    \bflambda^\top \vec{\mu}_{\bfP_X} = \sum_{j=1}^\ell \lambda_j \mu_{P_X^{(j)}} = \int \sum_j \lambda_j k(X, \cdot) \, dP_X^{(j)} = \int k(X,\cdot) d\left(\sum_j P_X^{(j)}\right) = \mu_{\bflambda^\top \bfP_X}.
\end{align*}
Next, let $A := \|\hat{\mu}_{\bflambda^\top \bfP_X} - \hat{\mu}_{\bfeta^\top \bfP_Y} \|_{\cH_k}^2 - \|\mu_{\bflambda^\top \bfP_X} - \hat{\mu}_{\bfeta^\top \bfP_Y} \|_{\cH_k}^2$. Then, by expanding out the expressions and rearranging the terms, it follows that
\begin{eqnarray*}
    |A| &=& \left|\|\hat{\mu}_{\bflambda^\top\bfP_X}\|^2_{\cH_k} - 2\langle\hat{\mu}_{\bflambda^\top\bfP_X}, \hat{\mu}_{\bfeta^\top\bfP_Y}\rangle + \|\hat{\mu}_{\bfeta^\top\bfP_Y}\|^2_{\cH_k} - \left(\|{\mu}_{\bflambda^\top\bfP_X}\|^2_{\cH_k} - 2\langle{\mu}_{\bflambda^\top\bfP_X}, \hat{\mu}_{\bfeta^\top\bfP_Y}\rangle + \|{\mu}_{\bfeta^\top\bfP_Y}\|^2_{\cH_k}\right)\right| 
    \\ 
        &=& \left|\|\hat{\mu}_{\bflambda^\top\bfP_X}\|^2_{\cH_k} -\|\mu_{\bflambda^\top\bfP_X}\|^2_{\cH_k} + 2\langle {\mu}_{\bflambda^\top\bfP_X} - \hat{\mu}_{\bflambda^\top\bfP_X}, \hat{\mu}_{\bfeta^\top\bfP_Y}\rangle \right| \\
        &\leq& \left|\|\hat{\mu}_{\bflambda^\top\bfP_X}\|^2_{\cH_k} -\|\mu_{\bflambda^\top\bfP_X}\|^2_{\cH_k}\right| + 2\|{\mu}_{\bflambda^\top\bfP_X} - \hat{\mu}_{\bflambda^\top\bfP_X}\| \|\hat{\mu}_{\bfeta^\top\bfP_Y}\| \\
        &\leq& \left|\|\hat{\mu}_{\bflambda^\top\bfP_X}\|^2_{\cH_k} -\|\mu_{\bflambda^\top\bfP_X}\|^2_{\cH_k}\right| + 2C\|{\mu}_{\bflambda^\top\bfP_X} - \hat{\mu}_{\bflambda^\top\bfP_X}\| \\
        &=& {O}\left(\frac{1}{\sqrt{n_t}}\right).
\end{eqnarray*}
The third step follows from the Cauchy-Schwartz inequality. The constant term $C$ in the third step follows from Assumption~\ref{assumption: 1} that the kernel is bounded, so the kernel mean embeddings are also bounded in RKHS norm. Then, the last step follows from the fact that empirical kernel mean embeddings converge to their population counterpart at ${O}(\nicefrac{1}{\sqrt{n_t}})$. Analogously, the difference $B := \|\mu_{\bflambda^\top \bfP_X} - \hat{\mu}_{\bfeta^\top \bfP_Y} \|_{\cH_k}^2    
    -\| \mu_{\bflambda^\top\bfP_X} - \mu_{\bfeta^\top\bfP_Y} \|_{\cH_k}^2$ can be shown to converge to zero at rate ${O}(\nicefrac{1}{\sqrt{n_t}})$.
Therefore, continuing from before, we have
\begin{equation*}
    |\cL_{n_t}(\bflambda,\bfeta) - L(\bflambda,\bfeta)| = |A + B| \leq |A|+|B| =  {O}\left(\frac{1}{\sqrt{n_t}}\right).
\end{equation*}
Since the convergence is not affected by the arguments, we have established uniform convergence for $\cL_{n_t}$.
\end{proof}

\subsection{Proofs for Theorem~\ref{thm: main_theorem_h0}}

To prove Theorem~\ref{thm: main_theorem_h0}, we need a few auxiliary results.

\begin{proposition}
\label{prop: positive definiteness}
Given positive definite kernel $k$, the Gram matrix $\bfM_{XX} \in \RR^{\ell \times \ell}$ with $i,j$ entries given by $\EE[k(X^{(i)}, X^{(j)})]$ is also positive definite.
\end{proposition}
\begin{proof}
    Let $\bm{\alpha}\in\RR^\ell$ be any vector of size $\ell$. Then, we have
    \begin{equation*}
        \bm{\alpha}^\top \bfM_{XX} \bm{\alpha} = \sum_{i=1}^\ell\sum_{j=1}^\ell \alpha_i\alpha_j \EE[k(X^{(i)}, X^{(j)})] 
        = \EE\left[\sum_{i=1}^\ell\sum_{j=1}^\ell \alpha_i\alpha_j k(X^{(i)}, X^{(j)})\right] > 0.
    \end{equation*}
    The last inequality follows from the positive definiteness of the kernel $k$. Consequently, the Gram matrix $\bfM_{XX}$ is also positive definite.
\end{proof}

\begin{proposition}
\label{prop: gradient_uniform_convergence}
    Under Assumption \ref{assumption: 1}, for any $\bflambda,\bfeta, \|\nabla L_{n_e}(\bflambda, \bfeta) - \nabla L(\bflambda, \bfeta)\| = {O}\left(\frac{1}{\sqrt{n_e}}\right)$.
\end{proposition}
\begin{proof}
    Differentiating $L_{n_e}$ with respect to its arguments yields 
    \begin{align*}
        \nabla L_{n_e}(\bflambda,\bfeta) = \begin{bmatrix}
            2\widehat{\bfM}_{XX}\bflambda -2 \widehat{\bfM}_{XY}\bfeta \\
            2\widehat{\bfM}_{YY}\bfeta - 2 \widehat{\bfM}_{YX}\bflambda 
        \end{bmatrix},
    \end{align*}
    whereas by differentiating $L(\bflambda,\bfeta)$, we have
    \begin{align*}
        \nabla L(\bflambda,\bfeta) = \begin{bmatrix}
            2{\bfM}_{XX}\bflambda -2 {\bfM}_{XY}\bfeta \\
            2{\bfM}_{YY}\bfeta - 2 {\bfM}_{YX}\bflambda 
        \end{bmatrix}.
    \end{align*}
    By following the same arguments as in the proof for Proposition~\ref{prop: uniform_convergence}, i.e., that empirical kernel mean embeddings converge to population counterparts at rate $O(\nicefrac{1}{\sqrt{n_e}})$, we can see that $\|\nabla L_{n_e}(\bflambda, \bfeta) - \nabla L(\bflambda, \bfeta)\| = {O}\left(\nicefrac{1}{\sqrt{n_e}}\right)$.
\end{proof}

\begin{proposition}
\label{proposition: eta_estimator_convergence}
    Under Assumption~\ref{assumption 0}, ~\ref{assumption: 1}, and under the null $H_{0,\in}: P_X \in \cC_Y$, $\bfeta^e$, the minimiser of the empirical KCD $L_{n_e}(1, \bfeta)$, converges to $\bfeta_0$, the minimiser of the population KCD $L(1, \bfeta)$, at the rate of ${O}(\nicefrac{1}{\sqrt{n_e}})$.
\end{proposition}
\begin{proof}
    Since $L_{n_e}(1, \bfeta)$ is twice continuously differentiable in $\bfeta$, we can apply the Taylor expansion to $\nabla L_{n_e}(1, \bfeta)$ around $\bfeta_0$ and obtain,
    \begin{align*}
        \nabla L_{n_e}(1, \bfeta^e) = \nabla L_{n_e}(1, \bfeta_0) + \nabla^2 L_{n_e}(1, \bfeta_0)^\top (\bfeta^e - \bfeta_0) + o(\|\bfeta^e - \bfeta_0\|).
    \end{align*}
    For simplicity, we will write $L_{n_e}(1, \bfeta)$ as $L_{n_e}(\bfeta)$, omitting the first argument. Hence, we have instead,
    \begin{align*}
        \nabla L_{n_e}(\bfeta^e) = \nabla L_{n_e}(\bfeta_0) + \nabla^2 L_{n_e}(\bfeta_0)^\top (\bfeta^e - \bfeta_0) + o(\|\bfeta^e - \bfeta_0\|).
    \end{align*}
    Next, since $\bfeta^e$ minimises $L_{n_e}$, $\nabla L_{n_e}(\bfeta^e)=0$. Furthermore, recall that
    \begin{align*}
        L(1,\bfeta) = \bfeta^\top \bfM_{YY}\bfeta -2\bfM_{XY}\bfeta + c
    \end{align*}
    for some positive constant $c$. Since $\bfM_{YY}$ is positive definite (Proposition~\ref{prop: positive definiteness}), the quadratic form $L$ is strongly convex. This implies that there is a lower bound $c_0 > 0$ for the singular values of $\nabla^2 L(\bfeta_0)$ and that the operator norm $\|\nabla^2 L(\bfeta_0)\| \geq c_0 I$ for some identity matrix $I$. Furthermore, since $L_{n_e}$ converges to $L$ uniformly, for large enough $n_e$, there also exists a positive constant $c_1$ such that $\|\nabla^2 L_{n_e}(\bfeta_0)\| \geq c_1 I$. Hence, for large enough $n_e$, combining all the results yields
    \begin{align*}
        \nabla L_{n_e}(\bfeta^e) &= \nabla L_{n_e}(\bfeta_0) + \nabla^2 L_{n_e}(\bfeta_0)^\top (\bfeta^e - \bfeta_0) + o(\|\bfeta^e - \bfeta_0\|) \\
        \implies 0 &= \nabla L_{n_e}(\bfeta_0) + \nabla^2 L_{n_e}(\bfeta_0)^\top (\bfeta^e - \bfeta_0) + o(\|\bfeta^e - \bfeta_0\|) \\
        \implies \nabla^2 L_{n_e}(\bfeta_0)^\top (\bfeta^e - \bfeta_0) &= -\nabla L_{n_e}(\bfeta_0) + o(\|\bfeta^e - \bfeta_0\|) \\
        \implies \|\bfeta^e - \bfeta_0\| &\leq \frac{1}{c_1}\|\nabla L_{n_e}(\bfeta_0)\| + h.o. \\
        &= {O}\left(\frac{1}{\sqrt{n_e}}\right)
    \end{align*}
    where the last step follows from the fact that $\|\nabla L_{n_e}(\bfeta)\| = \|\nabla L_{n_e}(\bfeta_0) - \nabla L(\bfeta_0)\|$ since $\bfeta_0$ is the minimiser of $L$. This error converges at rate ${O}(\nicefrac{1}{\sqrt{n_e}})$ as proven in Proposition~\ref{prop: gradient_uniform_convergence}. This concludes the proof of $\sqrt{n_e}$-consistency for the estimator $\bfeta^e$.
\end{proof}


\MainTheoremOne*
\begin{proof}[Proof for Theorem~\ref{thm: main_theorem_h0}] 

    \textbf{Error on the test statistic.} Under Assumptions \ref{assumption: 2}, there exists some $n_0$, such that for $n_t > n_0$, $\cL_{n_t}(1,\bfeta^e)$ as a function of $\bfeta^e$ is continuous in $\Delta_r$ and continuously differentiable in its interior. Therefore, by invoking the mean value theorem, we have
    \begin{align*}
        n_t\cL_{n_t}(1,\bfeta^e) = n_t\cL_{n_t}(1,\bfeta^0) + n_t\langle\bfeta^e - \bfeta_0, \nabla\cL_{n_t}(1,\tilde{\bfeta})\rangle
    \end{align*}
    for some $\tilde{\bfeta}$ lying on the line segment between $\bfeta_0$ and $\bfeta^e$. Rearranging terms yields
    \begin{align}
        n_t\cL_{n_t}(1,\bfeta^e) - n_t\cL_{n_t}(1,\bfeta_0) &= n_t\langle \bfeta^e - \bfeta_0, \nabla \cL_{n_t}(1,\tilde{\bfeta})\rangle \nonumber \\
        \implies |n_t\cL_{n_t}(1,\bfeta^e) - n_t\cL_{n_t}(1,\bfeta_0)| &\leq n_t \|\bfeta^e - \bfeta_0\| \|\nabla \cL_{n_t}(1,\tilde{\bfeta}) \label{eq: error_bound_for_mmd_eta}\|
    \end{align}
where we used the Cauchy-Schwarz inequality on the inner product. Notice that 
\begin{eqnarray*}
    \|\nabla \cL_{n_t}(1, \tilde{\bfeta})\| 
        &=& \|\nabla \cL_{n_t}(1, \tilde{\bfeta}) - \nabla \cL_{n_t}(1, \bfeta_0) + \nabla \cL_{n_t}(1, \bfeta_0)\|\\
        &\stackrel{(\clubsuit)}{=}& \|\nabla \cL_{n_t}(1, \tilde{\bfeta}) - \nabla \cL_{n_t}(1, \bfeta_0) + \nabla \cL_{n_t}(1, \bfeta_0) - \nabla L(1, \bfeta_0) \| \\
        &\leq& \|\nabla \cL_{n_t}(1, \tilde{\bfeta}) - \nabla \cL_{n_t}(1, \bfeta_0)\| + \|\nabla \cL_{n_t}(1, \bfeta_0) - \nabla L(1, \bfeta_0) \| \\
        &\stackrel{(\spadesuit)}{\leq}& C'\|\tilde{\bfeta} - \bfeta_0\| + C''\|\cL_{n_t} - L\|_{\infty} \\
        &\stackrel{(\heartsuit)}{\leq}& C'\|\bfeta^e - \bfeta_0\| + C''\|\cL_{n_t} - L\|_{\infty} \\
        &=& {O}\left(\frac{1}{\sqrt{n_e}} + \frac{1}{\sqrt{n_t}}\right).
\end{eqnarray*}
In $(\clubsuit)$, we used the fact that under the null and Assumption~\ref{assumption: 1}, $\bfeta_0$ is the minimiser of the population KCD. That is, since $P_X = \bfeta_0^\top \bfP_Y$, $\nabla L(1,\bfeta_0) = 0$. In $(\spadesuit)$, we used the Lipschitz conditions on the gradient terms stated in Assumption~\ref{assumption: 2}, i.e., $\|\nabla \cL_{n_t}(1, \tilde{\bfeta}) - \nabla \cL_{n_t}(1, \bfeta_0)\| \leq C'\|\tilde{\bfeta} - \bfeta_0\|$ for some constant $C'$, and $\|\nabla \cL_{n_t}(1, \bfeta_0) - \nabla L(1, \bfeta_0) \| \leq \|\nabla \cL_{n_t} - \nabla L\|_{\infty} \leq C''\|\cL_{n_t} - L\|_{\infty}$. In $(\heartsuit)$, we used the fact that $\tilde{\bfeta} = t\bfeta^e + (1-t)\bfeta_0$ for some $t\in[0,1]$, therefore by sandwiching argument, $\|\tilde{\bfeta} - \bfeta_0\| \leq \|\bfeta^e - \bfeta_0\|$. Finally, the last equality follows from the convergence rate results in Proposition~\ref{proposition: eta_estimator_convergence} and Proposition~\ref{prop: uniform_convergence}.

Next, continuing from Equation~\eqref{eq: error_bound_for_mmd_eta}, 
\begin{eqnarray*}
    |n_t\cL_{n_t}(1,\bfeta^e) - n_t\cL_{n_t}(1,\bfeta_0)|   
        &\leq& n_t \|\bfeta^e - \bfeta_0\| \|\nabla \cL_{n_t}(1,\tilde{\bfeta}) \| \\
        &\leq& n_t \frac{C}{\sqrt{n_e}} \left(\frac{1}{\sqrt{n_e}} + \frac{1}{\sqrt{n_t}}\right) \\
        &=& {O}\left(\frac{n_t}{n_e} + \sqrt{\frac{n_t}{n_e}}\right) \\
        &=& {O}\left(\sqrt{\frac{n_t}{n_e}}\right),
\end{eqnarray*}
where $C$ is some constant term.

\paragraph{Limiting distribution under the null.} We can now apply the Slutsky's theorem. For splitting ratio $\rho$ chosen such that $\nicefrac{n_t}{n_e} \to 0$, we have
\begin{eqnarray*}
    \lim_{n\to\infty}n_t\cL_{n_t}(1,\bfeta^e) 
        &=& \lim_{n\to\infty}n_t\cL_{n_t}(1,\bfeta^e) - \lim_{n\to\infty}n_t\cL_{n_t}(1,\bfeta_0) + \lim_{n\to\infty}n_t\cL_{n_t}(1,\bfeta_0) \\
        &=& \lim_{n\to\infty} (n_t\cL_{n_t}(1,\bfeta^e)-n_t\cL_{n_t}(1, \bfeta_0)) + \lim_{n_t\to\infty} n_t\cL_{n_t}(1,\bfeta_0) \\
        &\overset{D}{\rightarrow}& 0 + \sum_{i=1}^\infty \zeta_iZ_i^2,
\end{eqnarray*}
for standard normal random variables $Z_i \overset{i.i.d}{\sim} N(0,1)$, and a certain eigenvalue $\zeta_i$ depending on the choice of kernel and $P_X$, with $\sum_{i=1}^\infty \zeta_i < \infty$. For exact details, see \citet[Theorem 12]{gretton2012kernel}. This result shows that as long as we choose an adaptive splitting ratio such that $\nicefrac{n_t}{n_e} \to 0$, the null distribution of our test statistic will converge in distribution to the null distribution of the test statistic as if the oracle parameter is known. Appendix~\ref{appendix_subsubsec: null_dist_specification} provides an empirical demonstration of this result.

\paragraph{Consistency against the fixed alternative.} The proof strategy follows closely from \citep[Theorem 2]{key_composite_2024}. Under $H_{A,\in}$, we first show that $\lim\inf_{n\to\infty}\cL_{n_t}(1, \bfeta^e) > 0$. Recall that
\begin{eqnarray*}
    \cL_{n_t}(1,\bfeta^e) 
        &=& \|\hat{\mu}_{P_X} - \hat{\mu}_{{\bfeta^e}^\top\bfP_Y}\|^2_{\cH_k} \\
        &=& \|\hat{\mu}_{P_X} - \mu_{P_X} + \mu_{P_X} - \mu_{{\bfeta^e}^\top \bfP_Y} + \mu_{{\bfeta^e}^\top \bfP_Y} - \hat{\mu}_{{\bfeta^e}^\top\bfP_Y} \|^2_{\cH_k} \\
        &\geq& \|\mu_{P_X} - \mu_{{\bfeta^e}^\top\bfP_Y}\|^2_{\cH_k} - \|\hat{\mu}_{P_X} - \mu_{P_X}  + \mu_{{\bfeta^e}^\top \bfP_Y} - \hat{\mu}_{{\bfeta^e}^\top\bfP_Y} \|^2_{\cH_k} \\
        &\geq& \|\mu_{P_X} - \mu_{{\bfeta^e}^\top\bfP_Y}\|^2_{\cH_k} - \|\hat{\mu}_{P_X} - \mu_{P_X}\|  - \|\mu_{{\bfeta^e}^\top \bfP_Y} - \hat{\mu}_{{\bfeta^e}^\top\bfP_Y} \|^2_{\cH_k}
\end{eqnarray*}
where the last two steps follow from the triangle inequality. Note that as $n\to\infty$, both $\|\hat{\mu}_{P_X} - \mu_{P_X}\|$ and $\|\mu_{{\bfeta^e}^\top \bfP_Y} - \hat{\mu}_{{\bfeta^e}^\top\bfP_Y}\|$ are zeros \emph{almost surely} by the standard law of large numbers for RKHS~\citep{berlinet2011reproducing}. We just need to show that
\begin{align*}
    \lim\inf_{n\to\infty} \|\mu_{P_X} - \mu_{{\bfeta^e}^\top\bfP_Y}\| > 0.
\end{align*}
We proceed by contradiction. To emphasize the dependence of $\bfeta^e$ on the sample size, we write $\mu_{{\bfeta^e}^\top\bfP_Y}$ as $\mu_{\bfeta(n_e)}$. Suppose that $\lim\inf_{n\to\infty} \|\mu_{P_X} - \mu_{\bfeta(n_e)}\| = 0$, then by definition of the limit infimum, there exists a subsequence of estimators $\bfeta(a(n_e))$ such that,
\begin{align*}
    \lim_{n\to\infty}\|\mu_{P_X} - \mu_{\bfeta(a(n_e))}\|_{\cH_k}^2 = 0.
\end{align*}
Furthermore, since $\Delta_r$ is compact, there exists a subsequence $\bfeta(b(a(n_e)))$ and $\bfeta^\star \in \Delta_r$ such that $\lim_{n\to\infty}\|\bfeta(b(a(n_e))) - \bfeta^\star\| = 0$. Moreover, since the kernel is bounded by Assumption~\ref{assumption: 1}, we deduce,
\begin{align*}
    \MoveEqLeft \|\mu_{\bfeta(b(a(n_e)))} - \mu_{\bfeta^\star}\|^2_{\cH_k} \\ 
    &= \int\int k(x,y) \left((\bfeta(b(a(n_e)))^\top\bfP_Y)(x) - ({\bfeta^\star}^\top\bfP_Y)(x) \right)\left((\bfeta(b(a(n_e)))^\top\bfP_Y)(y) - ({\bfeta^\star}^\top\bfP_Y)(y) \right) dxdy \\
    &\to 0
\end{align*}
as $n\to \infty$. Therefore, by triangle inequality, we have
\begin{align*}
    \|\mu_{\bfeta^\star} - \mu_{P_X}\| \leq \|\mu_{\bfeta^\star} - \mu_{\bfeta(b(a(n_e))}\| + \|\mu_{\bfeta(b(a(n_e))} - \mu_{P_X}\| \to 0.
\end{align*}
Therefore, $\|\mu_{\bfeta^\star} - \mu_{P_X}\| = \operatorname{MMD}(P_X, {\bfeta^\star}^\top\bfP_Y) = 0$, but since the kernel is characteristic by Assumption~\ref{assumption: 1}, this implies there exists $\bfeta^\star \in \Delta_r$ such that ${\bfeta^\star}^\top\bfP_Y = P_X$. This is a contradiction under the alternative hypothesis $H_{A,\in}: \not\exists \bfeta\in\Delta_r, \bfeta^\top\bfP_Y = \bfP_X$. Finally, since $\lim\inf_{n\to\infty} \cL_{n_t}(1,\bfeta^e) > 0$, this implies,
\begin{align*}
    n_t\cL_{n_t}(1,\bfeta^e) \to \infty.
\end{align*}
This concludes the consistency proof for the test.
\end{proof}

\subsection{Proofs for Theorem~\ref{thm: main_theorem_2}}

\begin{proposition} 
\label{prop: local_is_global}
Under the null $H_{0,\cap}$ and Assumption~\ref{assumption: 1}, any local minimiser of $L$ is a global minimiser.
\end{proposition}
\begin{proof}
    Let $\Theta = \arg\min_{\bflambda\in\Delta_\ell, \bfeta\in\Delta_r} L(\bflambda,\bfeta)$ be the set of global minimisers for $L$. Since $L$ is  biconvex, standard results state that iterative minimisation guarantees arrival at local minima, i.e., we obtain $(\bflambda^e,\bfeta^e)$ such that,
    \begin{align*}
        \left.\nabla L(\bflambda,\bfeta)\right|_{\bflambda=\bflambda^e, \bfeta=\bfeta^e} = 0.
    \end{align*}
    Now, notice that
    \begin{align*}
        \nabla L(\bflambda,\bfeta) = \begin{bmatrix}
            2\bfM_{XX}\bflambda -2 \bfM_{XY}\bfeta \\
            2\bfM_{YY}\bfeta - 2 \bfM_{YX}\bflambda
        \end{bmatrix}.
    \end{align*}
    Setting $\nabla L(\bflambda,\bfeta) = 0$ and focusing on the top block matrix, we have,
    \begin{align*}
        \bfM_{XX}\bflambda - \bfM_{XY}\bfeta = 0.
    \end{align*}
    Specifically, for the $i^{th}$ entry of $\bfM_{XX}\bflambda - \bfM_{XY}\bfeta$, we have,
    \begin{align*}
        \sum_{a=1}^\ell (\bfM_{XX})_{i,a} \lambda_a - \sum_{b=1}^r (\bfM_{XY})_{i, b}\eta_b &= 0 \\ 
        \implies \sum_{a=1}^\ell \langle\mu_{P_X^{(i)}}, \mu_{P_X^{(a)}}\rangle \lambda_a - \sum_{b=1}^r \langle\mu_{P_X^{(i)}}, \mu_{P_Y^{(b)}}\rangle \eta_b &= 0 \\
        \implies \left\langle \mu_{P_X^{(i)}}, \sum_{a=1}^\ell \lambda_a \mu_{P_X^{(a)}} - \sum_{b=1}^r \eta_b \mu_{P_Y^{(b)}}\right\rangle = 0.
    \end{align*}
    Similarly, for the $j^{th}$ entry of $\bfM_{YY}\bfeta - \bfM_{YX}\bflambda$, we have
    \begin{align*}
        \left\langle \mu_{P_Y^{(j)}}, \sum_{a=1}^\ell \lambda_a \mu_{P_X^{(a)}} - \sum_{b=1}^r \eta_b \mu_{P_Y^{(b)}}\right\rangle = 0.
    \end{align*}
    Since $\sum_{a=1}^\ell \lambda_a \mu_{P_X^{(a)}} - \sum_{b=1}^r \eta_b \mu_{P_Y^{(b)}}$ are in the span of $\Xi = \{\mu_{P_X^{(1)}},\dots, \mu_{P_X^{(\ell)}}, \mu_{P_Y^{(1)}},\dots, \mu_{P_Y^{(r)}} \}$ but it is orthogonal to every element in $\Xi$, we can deduce by standard geometry argument that
    \begin{align*}
        \sum_{a=1}^\ell \lambda_a \mu_{P_X^{(a)}} - \sum_{b=1}^r \eta_b \mu_{P_Y^{(b)}} = \mathbf{0}.
    \end{align*}
    Under Assumption $\ref{assumption: 1}$, since the kernel is characteristic, we have
    \begin{align*}
        & \sum_{a=1}^\ell \lambda_a \mu_{P_X^{(a)}} - \sum_{b=1}^r \eta_b \mu_{P_Y^{(b)}} = \mu_{\bflambda^\top \bfP_X} - \mu_{\bfeta^\top\bfP_Y} = \mathbf{0} \\
        &\implies \bflambda^\top \bfP_X = \bfeta^\top \bfP_Y.
    \end{align*}
    As a result, any local minimiser $(\bflambda^e, \bfeta^e)$ to the optimisation of population KCD such that $\nabla L(\bflambda^e,\bfeta^e) = 0$ implies $\bfeta^e\bfP_Y = \bflambda^e\bfP_X$, therefore $\bfeta^e, \bflambda^e \in \Theta$, the set of global minimisers that satisfies the null hypothesis.
\end{proof}



\begin{proposition}
\label{prop: plausibility_estimator_convergence}
Under the plausibility null $H_{0,\cap}$, Assumptions \ref{assumption 0}, \ref{assumption: 1}, \ref{assumption: 2}, and \ref{assumption: 3}, let $\theta^e = (\bflambda^e, \bfeta^e)$ be a pair of local minimisers of the empirical KCD objective, i.e., $\theta^e = (\bflambda^e, \bfeta^e) \in \arg\min_{\bflambda\in\Delta_\ell,\bfeta\in \Delta_r} L_{n_e}(\bflambda,\bfeta)$. Then, there exists some $n_0 \in \NN$, such that for $n_e > n_0$, there exists $\theta_0=(\bflambda_0,\bfeta_0) \in \arg\min_{\bflambda\in\Delta_\ell,\bfeta\in \Delta_r} L(\bflambda,\bfeta)$ such that $\|\theta^e - \theta_0\| = {O}(\nicefrac{1}{\sqrt{n_e}})$.
\end{proposition}

\begin{proof}
    Since $L_{n_e}$ uniformly converges to $L$ in a compact domain, $L_{n_e}$ epi-converges to $L$~\citep[Theorem 2]{kall1986approximation}. Hence, all converging subsequence's limit points are part of the solution set $\arg\min L(\bflambda,\bfeta)$, i.e., $\lim_{n\to\infty}\arg\min L_{n_e} \subseteq \arg\min L = \Theta$. This means, for large enough $n$, any local minimiser $\theta^e$ is on some subsequence that is converging to some $\theta_0\in\Theta$. Applying the Taylor expansion to the gradient $\nabla L_{n_e}(\bflambda,\bfeta)$ with respect to this $\theta_0$ yields
    \begin{align*}
        \nabla L_{n_e}(\bflambda^e,\bfeta^e) = \nabla L_{n_e}(\bflambda_0,\bfeta_0) + \nabla^2 L_{n_e}(\bflambda_0, \bfeta_0)\begin{bmatrix}
            \bflambda^e - \bflambda_0 \\
            \bfeta^e - \bfeta_0
        \end{bmatrix} + \text{h.o.}
    \end{align*}
    For simplicity, we rewrite the expression in terms of $\theta$, then
    \begin{align}
        \nabla L_{n_e}(\theta^e) &= \nabla L_{n_e}(\theta_0) + \nabla^2 L_{n_e}(\theta_0)(\theta^e - \theta_0) + \text{h.o.} \nonumber \\
        0 &= \nabla L_{n_e}(\theta_0) + \nabla^2 L_{n_e}(\theta_0)(\theta^e - \theta_0) + \text{h.o.} \label{eq: kcd_estimator_uniform_convergence}
    \end{align}
    Note that $\nabla^2 L_{n_e}(\theta_0)$ can also be expressed as
    \begin{align*}
        \nabla^2 L_{n_e}(\theta_0) = \begin{bmatrix}
            \bfM_{XX} & -\bfM_{XY} \\
            -\bfM_{YX} & \bfM_{YY}
        \end{bmatrix}
    \end{align*}
    which by Assumption \ref{assumption: 3}, has a positive definite Schur complement, implying that the block matrix $\nabla^2 L(\theta_0)$ is positive definite and has a positive lower bound for its singular values. Using a similar argument as in the proof for Proposition~\ref{proposition: eta_estimator_convergence}, it follows that $\nabla^2 L_{n_e}$ uniformly converges to $\nabla^2 L$, therefore for large enough samples, the Schur complement for $\nabla^2 L_{n_e}$ will be positive definite as well, therefore the singular values for $\nabla^2 L_{n_e}$ will be lower bounded by some constant $c>0$. Furthermore, $\nabla^2 L_{n_e}$  will be invertible. Continuing from Equation~\eqref{eq: kcd_estimator_uniform_convergence},
    \begin{eqnarray*}
        \|\theta^e - \theta_0\| &\leq& \|\nabla^2 L_{n_e}(\theta_0)^{-1}\|\|\nabla L_{n_e}(\theta_0)\| + \text{h.o.} \\
            &\leq& \frac{1}{c}{O}\left(\frac{1}{\sqrt{n_e}}\right) \\
            &=& {O}\left(\frac{1}{\sqrt{n_e}}\right)
    \end{eqnarray*}
    since $\|\nabla L_{n_e}(\theta_0)\| = {O}\left(\nicefrac{1}{\sqrt{n_e}}\right)$ as proven in Proposition~\ref{prop: gradient_uniform_convergence}. This concludes the proposition for the convergence rate of our estimators for the plausibility test.
\end{proof}

\MainTheoremTwo*
\begin{proof}
    The overall proof strategy is analogous to the proof for Theorem~\ref{thm: main_theorem_h0}. For a fixed $n_t, n_e$, pick an optimiser $\theta^e=(\bflambda^e,\bfeta^e)$ from $\arg\min L_{n_e}(\bflambda,\bfeta)$. Since $L_{n_e}$ converges to $L$ uniformly over a compact domain, epi-convergence~\citep{kall1986approximation} implies $\lim_{n\to\infty}\arg\min L_{n_e}(\bflambda,\bfeta) \subseteq \arg\min L = \Theta$. This means $\theta^e$ is on some subsequence that converges to some $\theta_0 \in \Theta$. Let $\theta_0$ be such limit for the subsequence $\theta^e$ is on. As such, based on Assumption $\ref{assumption: 2}$, there exits some $n_0\in\NN$ such that for $n_t>n_0$, $n_t\cL_{n_t}(\theta^e)$ is continuous over $\Delta_\ell \times \Delta_r$ and differentiable over the interior. We can then invoke the mean value theorem,
    \begin{align*}
        n_t\cL_{n_t}(\theta^e) = n_t\cL_{n_t}(\theta_0) + n_t\langle\theta^e - \theta_0, \nabla \cL_{n_t}(\tilde{\theta}) \rangle
    \end{align*}
    where $\tilde{\theta}$ is some interpolation between $\theta_0$ and $\theta^e$. Rearranging the terms and applying Cauchy-Schwartz yield
    \begin{align*}
        |n_t\cL_{n_t}(\theta^e) - n_t\cL_{n_t}(\theta_0)| \leq n_t \|\theta^e - \theta_0\|\|\nabla\cL_{n_t}(\tilde{\theta})\|
    \end{align*}
    Utilising Proposition~\ref{prop: plausibility_estimator_convergence}, Assumption~\ref{assumption: 2}, and Proposition~\ref{prop: uniform_convergence}, we can express the error as
    \begin{align}
        |n_t\cL_{n_t}(\theta^e) - n_t\cL_{n_t}(\theta_0)| = {O}\left(\sqrt{\frac{n_t}{n_e}}\right).
    \end{align}
    As a result, if the split ratio is chosen such that $\nicefrac{n_t}{n_e}\to 0$ as $n\to\infty$, this error decays to zero. Next, for any $\epsilon > 0$, choose $\epsilon_1, \epsilon_2 > 0$ such that $\epsilon = \epsilon_2 + \epsilon_3$, we know there exists $n_2 \in \NN$ such that for all $n > n_2$, there exists $\theta_0 \in \Theta$, such that 
    \begin{align*}
        | F_{n_t\cL_{n_t}(\theta^e)}(x) - F_{n_t\cL_{n_t}(\theta_0)}(x)| < \epsilon_2
    \end{align*}
    since convergence almost surely implies convergence in distribution. Here, $F$ is the cumulative distribution function. Now, there also exists $n_3\in\NN$ such that for all $n> n_3$, there exists a $Z \in \cZ$ such that,
    \begin{align*}
        |F_{n_t\cL_{n_t}(\theta_0)}(x) - F_Z(x)| < \epsilon_3,
    \end{align*}
    where $Z = \sum_{i=1}^\infty \zeta_{i, \bflambda_0,\bfeta_0} Z_i^2$ is the infinite sum of chi-squared distributions that is indexed by the parameter $(\bflambda_0,\bfeta_0)$. Combining the two statements, we arrive at the main result. For $n> n_1 = \max(n_2, n_3)$, there exists $\theta_0\in\Theta$, such that,
    \begin{eqnarray*}
        |F_{n_t\cL_{n_t}(\theta^e)(x)} - F_Z(x)| &\leq& | F_{n_t\cL_{n_t}(\theta^e)}(x) - F_{n_t\cL_{n_t}(\theta_0)}(x)| + |F_{n_t\cL_{n_t}(\theta_0)}(x) - F_Z(x)| \\ 
        &<& \epsilon_2 + \epsilon_3 \\
        &=& \epsilon .
    \end{eqnarray*}
    The proof for showing $n_t\cL_{n_t}(\bflambda^e, \bfeta^e)\to\infty$ under the fixed alternative $H_{A,\cap}$ is identical to the proof in Theorem~\ref{thm: main_theorem_2} showing that $n_t\cL_{n_t}(1,\bfeta^e) \to \infty$ under the fixed alternative $H_{A,\in}$, thus is ommited.
\end{proof}

%% file: sections/appendix/C_experiments.tex
\newpage
\section{Additional Experiments}
\label{appendix_sec: further_experiments}

All the experiments were conducted on the Google Cloud Platform using a single NVIDIA V100 GPU. We provide an overview of the additional experiments below:

\begin{itemize}
    \item In Appendix~\ref{appendix_subsec: null_dist}, we simulate and plot the empirical null distribution of our test statistic using oracle parameters, as well as when parameters are estimated with both fixed and adaptive splitting ratios. The aim of these experiments is to visualize the non-diminishing bias in the null distribution when using a fixed splitting ratio and to compare this with the null distribution from the adaptive splitting ratio, which closely resembles the distribution obtained with oracle parameters. Aligning with our theoretical analysis from Theorem~\ref{thm: main_theorem_h0} and Theorem~\ref{thm: main_theorem_2}.
    \item In Appendix~\ref{appendix_subsec: more synthetic data}, we conduct ablation studies on our credal tests and assess their performance under different configurations:
        \begin{itemize}
            \item In Appendix~\ref{appendix_subsubsec: different convex weights}, we simulate 10 different convex weights $\bfeta_0$ uniformly from $\Delta_r$ to test the sensitivity of the specification test with respect to $\bfeta_0$.
            \item In Appendix~\ref{appendix_subsubsec: varying number of credal samples}, we increase the number of extreme points in the credal sets from 3 (used in the main text experiments) to 5 and 10, and study the impact on the Type I error convergence behavior.
            \item In Appendix~\ref{appendix_subsubsec: tradeoff between different adaptive splitting ratios}, we expand the set of adaptive splitting ratios from $\beta \in \{0, 0.25, 0.33\}$ to $\{0, 0.1, 0.2, 0.3, 0.4, 0.5, 0.6, 0.7\}$ to highlight the trade-off between the Type I error convergence rate and test power when different split configurations are chosen.
            \item In Appendix~\ref{subsubsec:largescale}, we conduct a large-scale experiment with up to 45,000 samples on a challenging two-sample problem, comparing a mixture of Gaussians with a mixture of Student's t-distributions (with 10 degrees of freedom) to demonstrate the scalability of our method.
            \item In Appendix~\ref{appendix_subsubsec: comparing_with_double_dipping}, we compare our standard sample splitting scheme with the "double-dipping" approaches considered in \citet{key_composite_2024} and \citet{bruck_distribution_2023}. We discuss why analysing double-dipping theoretically is challenging, and show that using double-dipping with a fixed splitting ratio may not always achieve correct Type I error control, making it unreliable for practical use.
            \item In Appendix~\ref{appendix subsub linearly dependent}, we present the results of the specification test when Assumption~\ref{assumption 0}, concerning the linear independence of extreme points, is violated. We provide a sketch of the proof explaining why this does not pose a problem, as the arguments are analogous to those addressing multiplicity in the optimisation problem for the plausibility test, as proven in Theorem~\ref{thm: main_theorem_2}.
        \end{itemize}
    \item In Appendix~\ref{appendix_subsec: mnist}, we conduct semi-synthetic experiments with credal tests using the MNIST data to demonstrate our test can handle structured data such as images. 
\end{itemize}



\subsection{Examining the Empirical Null Distributions of the Test Statistics}
\label{appendix_subsec: null_dist}

\subsubsection{Null Distributions for the Specification Test}
\label{appendix_subsubsec: null_dist_specification}
To illustrate the impact of estimation on test validity, we utilise the simulation setup described in Section~\ref{sec: experiments} to simulate the empirical null test statistic distribution. We set $n=3000$ for the following illustrations. 
\begin{itemize}
    \item \textbf{Oracle:} No estimation is involved. For each round in the $500$ repeated experiment, we draw the credal samples $\bfS_X,\bfS_Y$ and compute the test statistic $n_t\cL_{n_t}(1, \bfeta_0)$ following the sample-splitting procedure described in the main paper. Since the estimation weight is provided in this procedure, we call this the oracle set-up.
    \item \textbf{CMMD(0):} Estimation is involved. For each round in the $500$ repeated experiments, we draw the credal samples $\bfS_X,\bfS_Y$ and compute $n_t\cL_{n_t}(1, \bfeta^e)$ with a fixed splitting ratio $\rho$ such that $\nicefrac{n_t}{n_e} = 1$. We call this the CMMD(0) approach.
    \item \textbf{CMMD($\nicefrac{1}{3})$:}. Estimation is involved. For each round in the $500$ repeated experiment, we draw the credal samples $\bfS_X,\bfS_Y$ and compute $n_t\cL_{n_t}(1, \bfeta^e)$ with an adaptive splitting ratio $\rho$ such that $\nicefrac{n_t}{n_e} = \nicefrac{1}{n_e^{0.33}}$. We call this the CMMD(\nicefrac{1}{3}) approach.
\end{itemize}
\begin{figure}[!h]
    \centering
    \includegraphics[width=\linewidth]{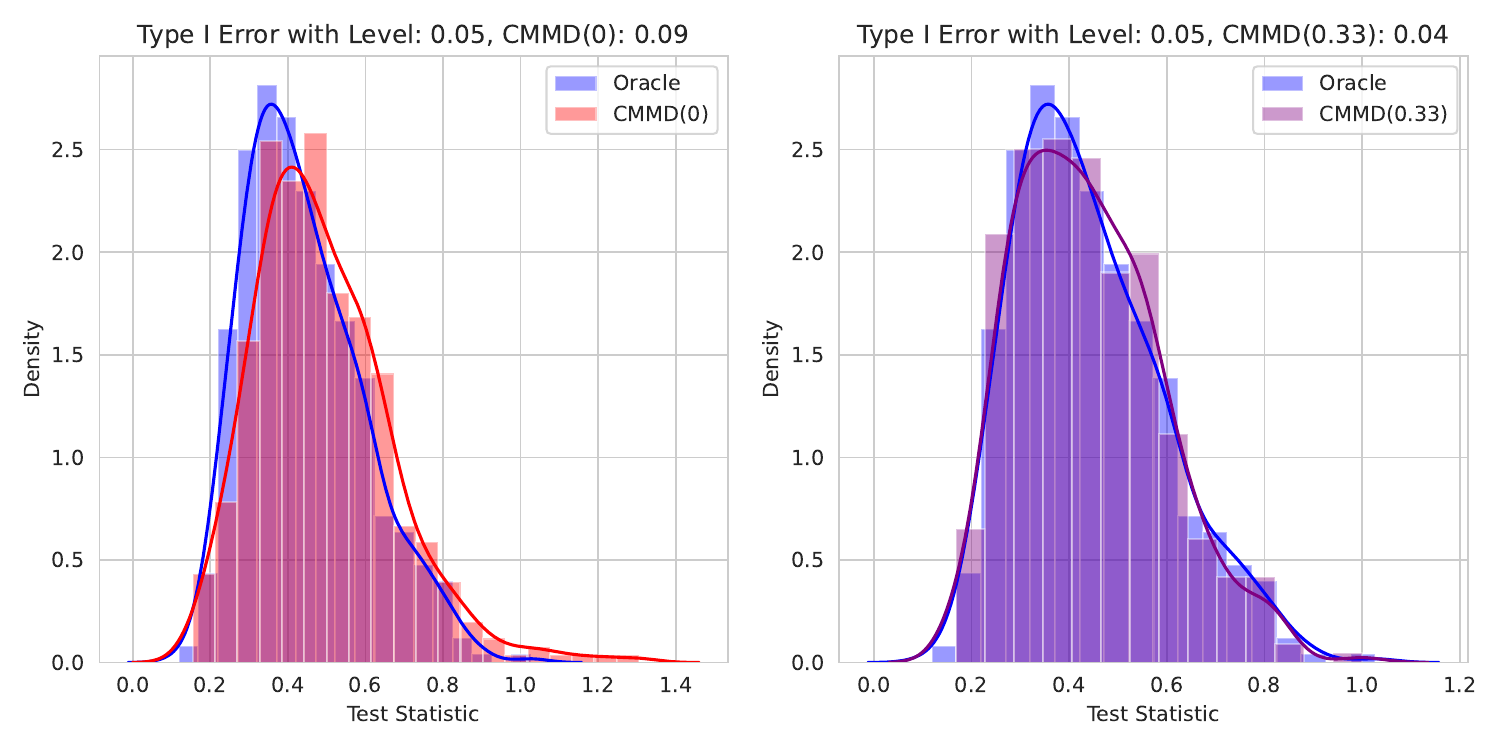}
    \caption{We visualise the impact of estimation on the null statistic distribution when using fixed splitting ratio CMMD$(0)$ and adaptive splitting ratio CMMD$(0.33)$. Using a fixed sample splitting scheme results in an empirical null distribution that presents an observable shift compared to the null distribution based on the oracle parameter. On the other hand, the adaptive sample splitting scheme results in an empirical null distribution that resembles the shape of the oracle version.}
    \label{fig:impact_of_estimation_to_null_distribution}
\end{figure}

After computing all the test statistics, we fit a kernel density estimator to visualise the distribution for each method. Figure~\ref{fig:impact_of_estimation_to_null_distribution} illustrates the impact of estimation on the null statistic distribution when using a fixed splitting ratio where $\nicefrac{n_t}{n_e} = 1$, compared to an adaptive splitting ratio where $\nicefrac{n_t}{n_e} = \nicefrac{1}{n_e^{0.33}}$. Even with 3,000 samples, the estimation error has a persistent effect on the test statistic, leading to incorrect Type I error control. As shown in the left panel, the null statistic distribution from CMMD$(0)$ is shifted to the right relative to the oracle distribution. In contrast, with adaptive sample splitting, as demonstrated in Theorem~\ref{thm: main_theorem_h0}, the estimation error decays faster relative to the decay of the test statistic, allowing for asymptotically correct Type I control. This is evident in the right panel, where the null statistic distribution of CMMD$(0.33)$ nearly overlaps with that of the oracle procedure.

\subsubsection{Null Distributions for the Plausibility Test}
\label{subsubsec: null_dist_plausibility_multiplicity}

For plausibility tests, there is an interesting phenomenon regarding the null statistic's distribution. Recall that the plausibility test is based on the following null:
\begin{align*}
    H_{0,\cap}: \cC_X\cap\cC_Y \neq \emptyset.
\end{align*}
Consider $\cC_X = \operatorname{CH}[P_1, P_2, P_3]$ and $\cC_Y = \operatorname{CH}[Q_1, P_2, P_3]$, then any convex weights $\bflambda \in \Delta_\ell$ and $\bfeta \in \Delta_r$ of the form $(0, \lambda_2, \lambda_3)$ and $(0,\eta_2, \eta_3)$ satisfies the null.  Consequently, unlike the specification test, under Assumption~\ref{assumption 0}, which has only one specific convex weight and, therefore, a single null distribution of test statistics, the plausibility test can exhibit a set of null distributions, each indexed by a pair of plausible convex weights. This raises concerns about the approximated Type I error reported in the experimental section. To investigate further, we replicated the plausibility test setup described in Section~\ref{sec: experiments}. In each repetition, we drew credal samples $\bfS_X, \bfS_Y$ from the population, estimated a pair of convex weights $\bflambda^e, \bfeta^e$, computed the test statistic $n_t \cL_{n_t}(\bflambda^e, \bfeta^e)$, and stored this value in a list of test statistics.

In the standard specification test, this list can be used to estimate the underlying null statistic distribution, as each element represents a draw from the same null distribution. However, in the plausibility test, due to the random initialisation of the iterative optimisation algorithm, each draw of credal samples $\bfS_X, \bfS_Y$ can result in a different pair of convex weights $\bflambda^e, \bfeta^e$. As a result, even with the adaptive splitting ratio, the test statistic for each round may be drawn from a null distribution that differs from previous rounds due to the randomness in the optimisation process. Intuitively, the null distribution we are observing in the experiment for the plausibility test follows the generative process:
\begin{align*}
    \PP(\text{Test Statistic}) = \int \PP(\text{Test Statistic}\mid \text{Optimisation identified $\bflambda_0,\bfeta_0$}) \,d\PP(\text{Optimisation identified $\bflambda_0,\bfeta_0$}).
\end{align*}
This scenario happens regardless of whether we use a fixed or adaptive splitting strategy, as illustrated in Figure~\ref{fig:plausibility_test_cmmd_0} and Figure~\ref{fig:plausibility_test_cmmd_0.33}.
\begin{figure}[!h]
    \centering
    \includegraphics[width=\linewidth]{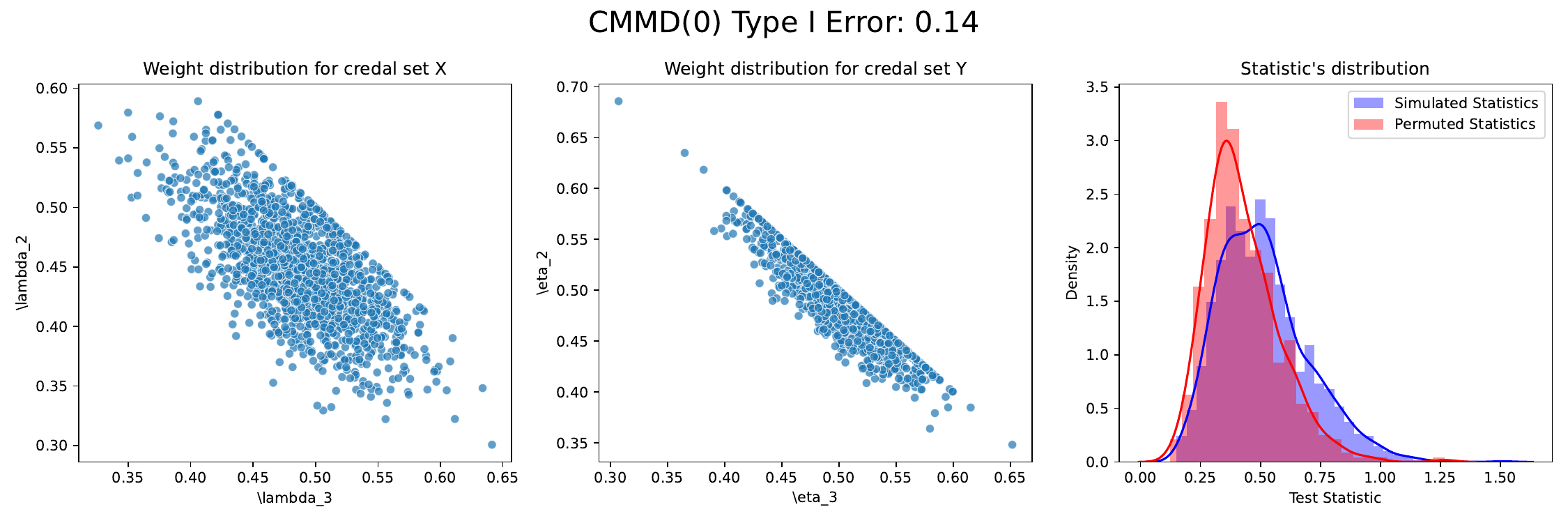}
    \caption{(Left) and (Middle): Distribution of the estimated parameters $\bflambda^e$ and $\bfeta^e$. Due to the existence of multiple pairs of weights under which the null hypothesis holds, our randomised optimisation procedure may identify a different pair of weights in each round during the repeated data sampling used to approximate the Type I error distribution in the experiments. (Right) The null statistic distribution for CMMD$(0)$ in the plausibility test, is denoted as "Simulated Statistics". The “Permuted Statistic” refers to the statistics generated through permutation during a specific round of the repeated experiment using the permutation test.}
    \label{fig:plausibility_test_cmmd_0}
\end{figure}
\begin{figure}[!h]
    \centering
    \includegraphics[width=\linewidth]{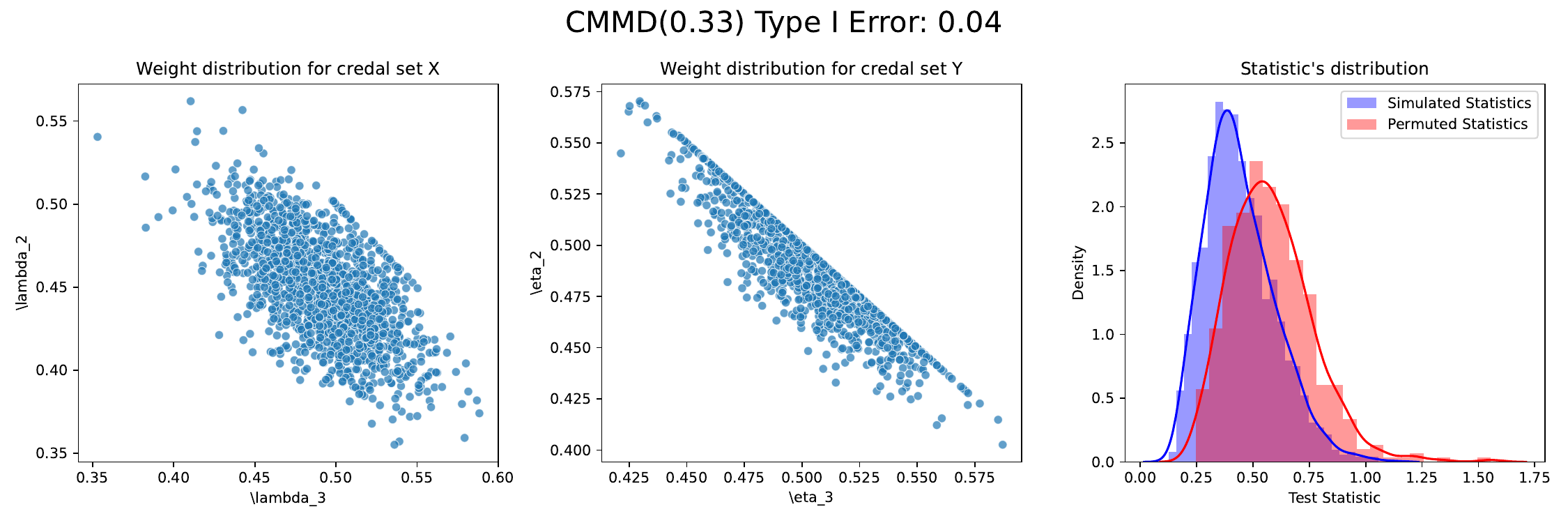}
    \caption{(Left) and (Middle): Distribution of the estimated parameters $\bflambda^e$ and $\bfeta^e$. Due to the existence of multiple pairs of weights under which the null hypothesis holds, our randomised optimisation procedure may identify a different pair of weights in each round during the repeated data sampling used to approximate the Type I error distribution in the experiments.  (Right) The null statistic distribution for CMMD$(0.33)$ in the plausibility test, is denoted as "Simulated Statistics". The “Permuted Statistic” refers to the statistics generated through permutation during a specific round of the repeated experiment using the permutation test.}
    \label{fig:plausibility_test_cmmd_0.33}
\end{figure}

Specifically, in both figures, we observe that the null statistic distribution derived from repeated data sampling and the estimated null statistic distribution obtained through the permutation procedure (for a fixed round of observation) do not overlap significantly. In fact, the null statistic distribution from repeated sampling follows a mixture of chi-square distributions, each indexed by the weights identified during that specific round of repetition. Nevertheless, the overall rejection rate for CMMD$(0.33)$ is 0.04, while for CMMD$(0)$ it is 0.14—substantially inflated compared to the nominal level of 0.05. This discrepancy is not surprising because the Type I error for this plausibility test follows the following generating process:
\begin{align*}
\PP(\text{Rejection}) = \int \PP(\text{Rejection} \mid \text{Optimisation identified }\bflambda_0, \bfeta_0) \,d\PP(\text{Optimisation identified }\bflambda_0, \bfeta_0).
\end{align*}
However, as shown in Theorem~\ref{thm: main_theorem_2}, as long as a specific pair of weights are identified and we use the adaptive splitting ratio, asymptotically:
\begin{align*}
\PP(\text{Rejection} \mid \text{Optimisation identified }\bflambda_0, \bfeta_0) = 0.05.
\end{align*}
This implies that, regardless of how the randomisation in the optimisation affects the distribution of weights obtained in each round of the repeated experiment, the overall Type I error rate will still converge to 0.05. This holds because
\begin{align*}
\PP(\text{Rejection}) &= \int \PP(\text{Rejection} \mid \text{Optimisation identified }\bflambda_0, \bfeta_0) \,d\PP(\text{Optimisation identified }\bflambda_0, \bfeta_0) \\
&= \int 0.05 \ \,d\PP(\text{Optimisation identified }\bflambda_0, \bfeta_0) \\
&= 0.05.
\end{align*}

Therefore, the multiplicity of null distributions under the plausibility hypothesis does not cause any issue. In the end, we still have the correct Type I control, since \emph{if I were to repeatedly sample the observations and conduct my test, although the procedure might identify different solutions every time, on average I am still wrong $5\%$ all the time}.

\subsection{Ablation Studies with Synthetic Data}
\label{appendix_subsec: more synthetic data}
We now perform ablation studies on our tests using the synthetic data set-up we outlined in Section~\ref{sec: experiments}. Unless specified, all experiments share the same kernel parameter selection, number of permutations used to determine critical values, and number of repetitions used to determine the rejection rate, as the main experiments in Section~\ref{sec: experiments}.
\label{appendix_subsec: ablation} 

\subsubsection{Experimenting with Different Convex Weights for Specification Test}
\label{appendix_subsubsec: different convex weights}

In the main experiment section, we chose not to include error bars, as they are generally not relevant in the context of hypothesis testing. We conducted 500 repetitions of the experiments and reported the average rejection rate as an approximation of the Type I error probability. If we were to repeat this setup 10 more times, it would essentially amount to running the experiment 5000 times, then splitting the results into 10 groups, averaging the rejection rates, and plotting the error bars—an approach that wouldn’t provide additional meaningful insights.

However, for the specification (and inclusion) test, we can generate multiple sets of convex weights and observe how the tests perform across these different weights. This is not possible with the equality and plausibility tests, as there is no additional randomness to exploit in these cases. 

To illustrate the sensitivity of our tests to convex weights in the simulation set-up, we randomly draw $10$ sets of convex weights and perform the specification experiment described in the main text. The result is presented in Figure~\ref{fig:many_corners}. The observation is the same as in previous sections, fixed sample splitting results in Type I inflation, while adaptive sample splitting results in Type I control asymptotically. Our permutation-based methods significantly outperform studentised statistic-based approaches. We also see as sample size increases, the fluctuation in Type I for CMMD$(0.33)$ and CMMD$(0.25)$ decreases as well.

\begin{figure}[!h]
    \centering
    \includegraphics[width=\linewidth]{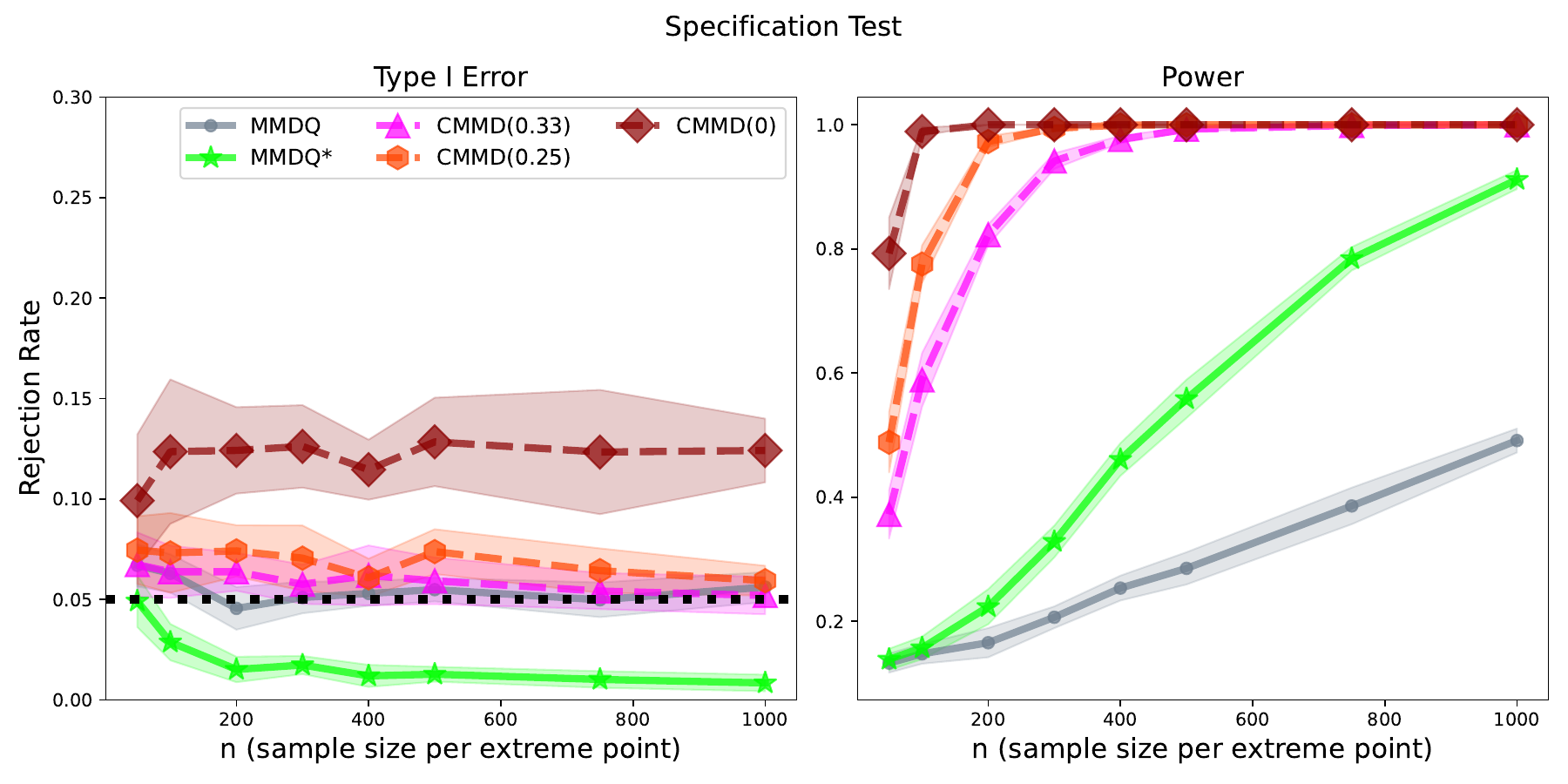}
    \caption{We average the rejection rate over $10$ configurations of convex weight (with $1$ standard deviation reported) for the specification test to demonstrate the sensitivity of our tests to the convex weights. The conclusion is the same as in previous sections, fixed sample splitting results in Type I inflation, while adaptive sample splitting results in Type I control asymptotically. Our permutation-based methods significantly outperform studentised statistic based approaches.}
    \label{fig:many_corners}
\end{figure}

\subsubsection{Varying Number of Credal Samples}
\label{appendix_subsubsec: varying number of credal samples}
In this section, we study how varying the number of extreme points in our simulation affects Type I control in our algorithms. 

\paragraph{Experimental setup.} The set-up for the specification test follows from the one described in the main text but this time we use $5$ and $10$ number of extreme points for $\cC_Y$ instead of $3$. For inclusion test, we test whether $\cC_X \subseteq \cC_Y$ for $\cC_X, \cC_Y$ both having $5$ and $10$ number of extreme points. The same setting applies to the equality test. For plausibility tests, $\cC_X$ and $\cC_Y$ only share two common extreme points, and the rest differ.

\paragraph{Analysis.} Figure~\ref{fig:abalation_number_of_corners_5} and Figure~\ref{fig:abalation_number_of_corners_10} illustrate synthetic experiment results for credal sets with $5$ extreme points and $10$ extreme points respectively. The overall behaviour is analogous to the one presented in the main text. Fixed splitting ratio results in inflated Type I control, thus rendering it invalid. We see also that by increasing the number of extreme points, the estimation problem becomes more challenging, therefore the convergence to Type I for CMMD($\nicefrac{1}{4})$ is observably slower than that of CMMD($\nicefrac{1}{3})$. This is particularly obvious for the equality test, which is the most challenging tests since it requires performing multiple testing to check whether a corner distribution belongs to another set. 

\begin{figure}[!h]
    \centering
    \includegraphics[width=0.49\linewidth]{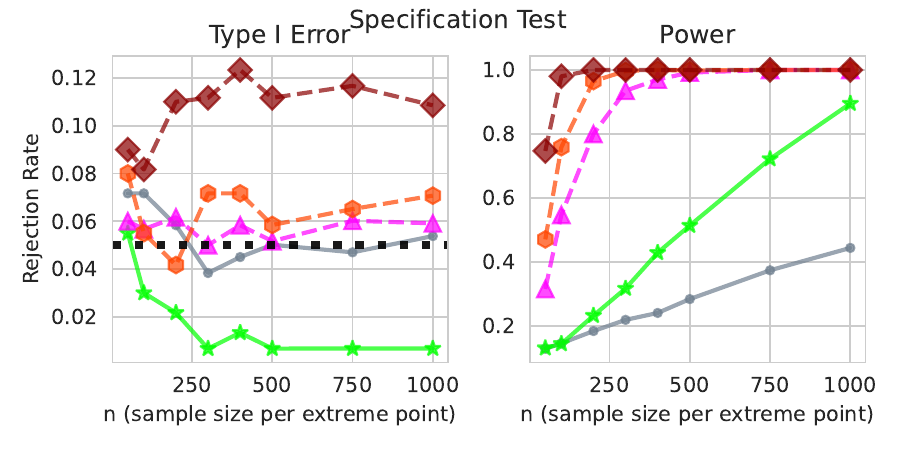}
    \includegraphics[width=0.49\linewidth]{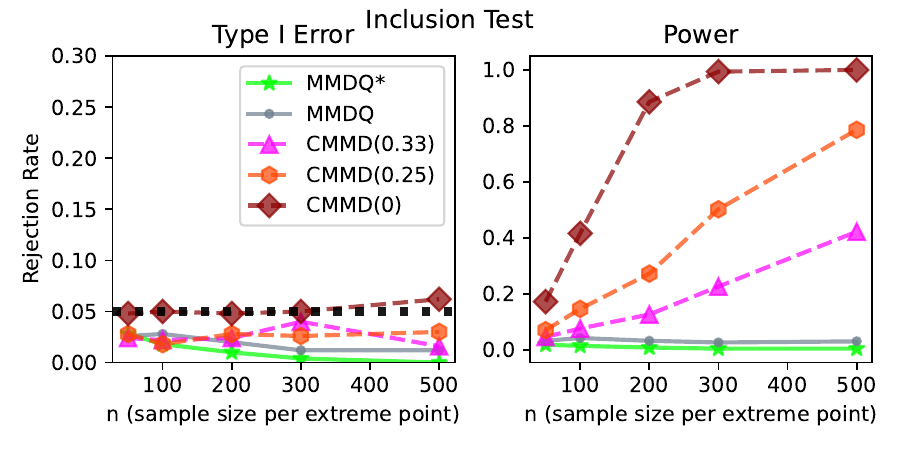}
    \includegraphics[width=0.49\linewidth]{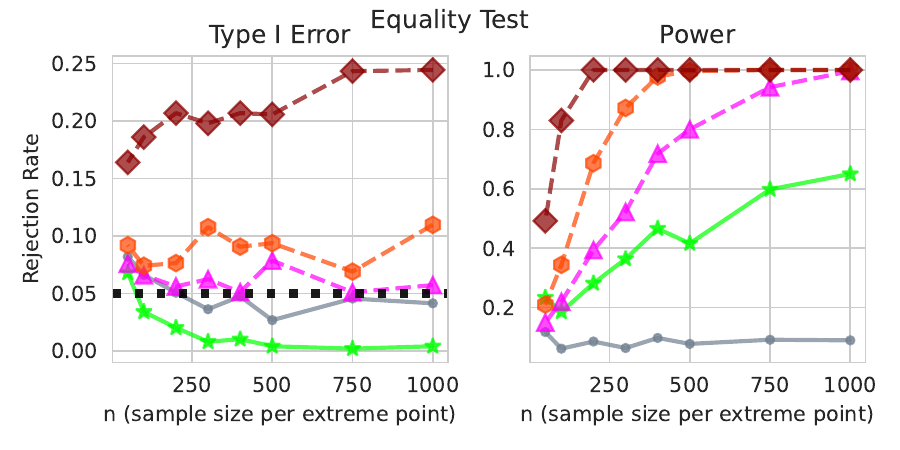}
    \includegraphics[width=0.49\linewidth]{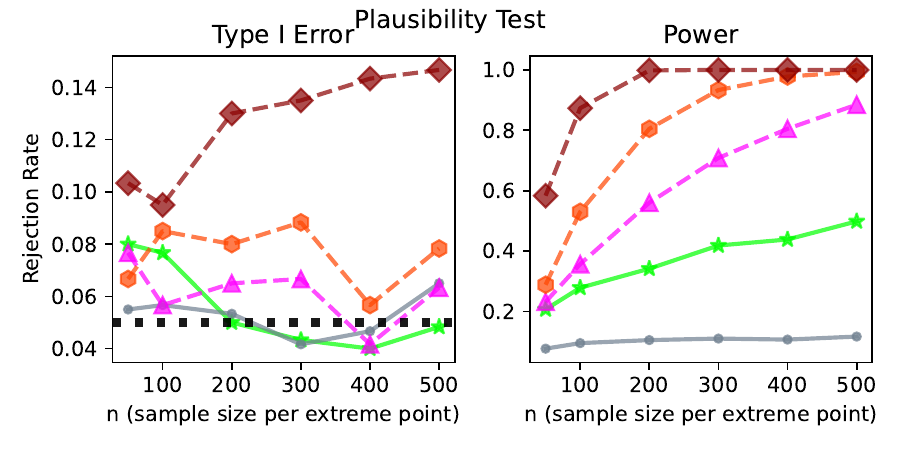}
    \caption{Synthetic experiment results for credal sets with $5$ extreme points. We again observe a persistent Type I error for fixed sample splitting and for adaptive sample splitting approaches we see Type I convergence. The power of studentised tests is significantly weaker than our proposed permutation-based approaches. }
    \label{fig:abalation_number_of_corners_5}
\end{figure}

\begin{figure}[!h]
    \centering
    \includegraphics[width=0.49\linewidth]{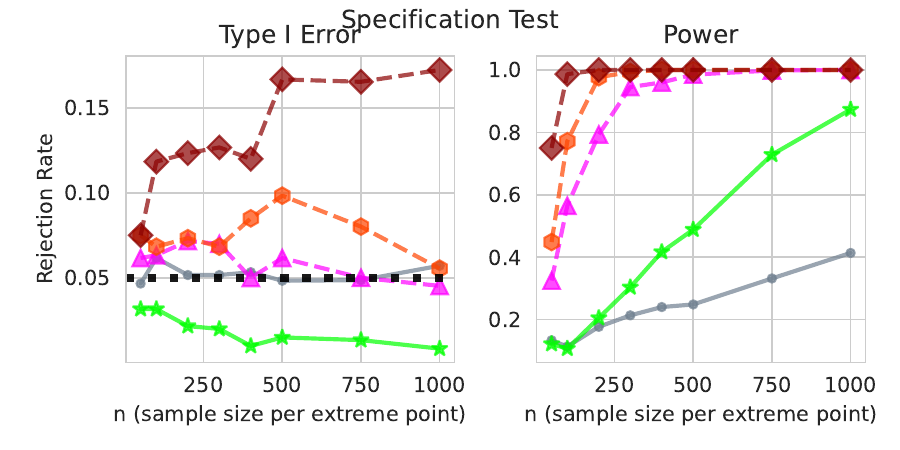}
    \includegraphics[width=0.49\linewidth]{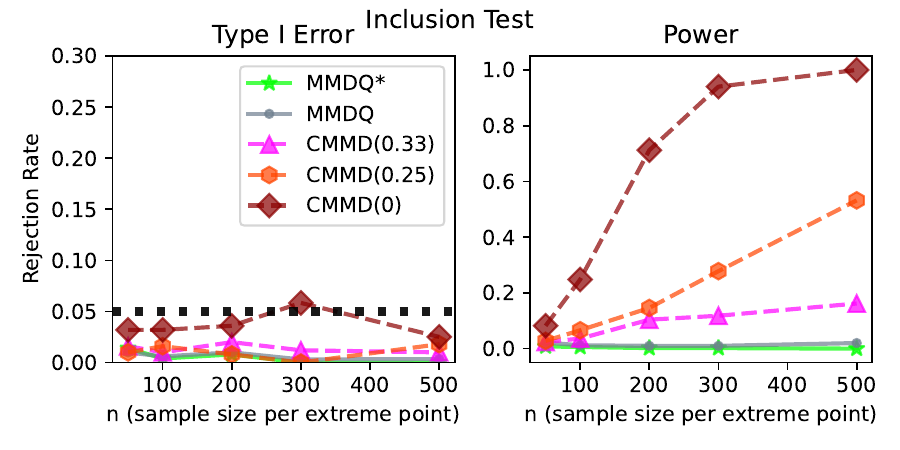}
    \includegraphics[width=0.49\linewidth]{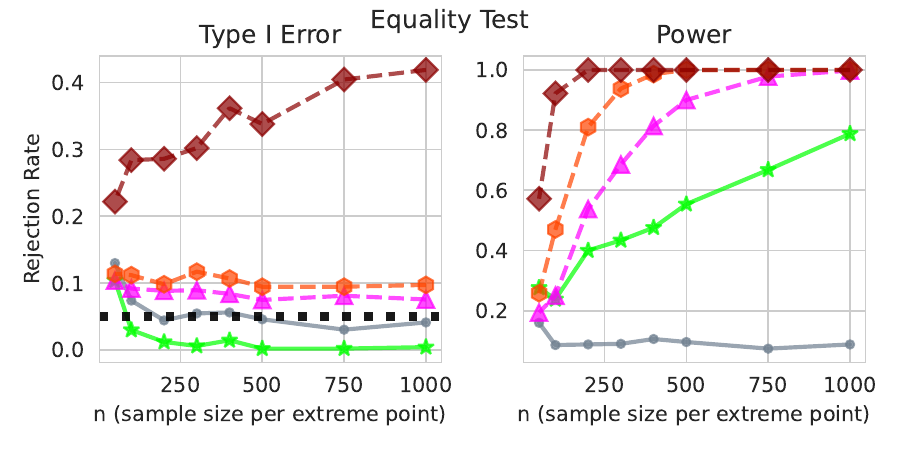}
    \includegraphics[width=0.49\linewidth]{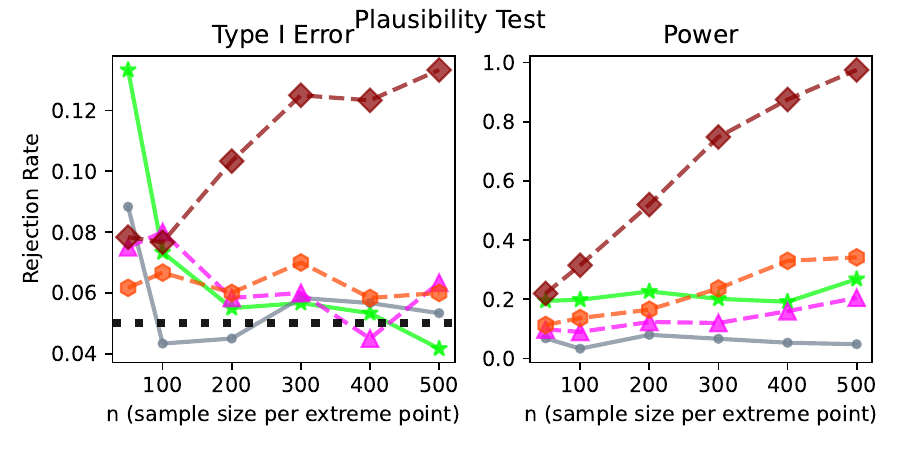}
    \caption{Synthetic experiment results for credal sets with $10$ extreme points. Besides the same observations as before, we see that as we increase the number of credal samples, the Type I inflation becomes more serious compared to when using $5$ or $3$ number of credal samples. This suggests in practice one should consider the tradeoff between the difficulty of the problem with the right adaptive sample-size splitting ratio.}
    \label{fig:abalation_number_of_corners_10}
\end{figure}

\subsubsection{Tradeoff between Different Adaptive Splitting Ratios}
\label{appendix_subsubsec: tradeoff between different adaptive splitting ratios}

We use the specification test to demonstrate the trade-off between the Type I error convergence rate and test power. In the following, we replicate the specification experiment from the main paper, but with a broader range of configurations controlling the split ratios. Recall that the adaptive split ratio $\rho$ is chosen for $n$ such that $\nicefrac{n_t}{n_e} = \nicefrac{1}{n_e^{\beta}}$, where we take $\beta \in \{0.0, 0.1, 0.2, 0.3, 0.4, 0.5, 0.6, 0.7\}$. For each configuration, we draw 10 sets of convex weights and run 500 experiments per weight to obtain the rejection rate. We then average the results across the configurations and report them in Figure~\ref{fig:more_split}.

As shown in Figure~\ref{fig:more_split}, as $\beta$ approaches 0—causing the ratio $\nicefrac{n_t}{n_e}$ to converge more slowly—the Type I error inflation becomes increasingly pronounced. However, since as $\beta$ approaches $0$ we get more testing samples, the power of the test also increases. Although theoretically any $\beta > 0$ is a valid test since they are proven to asymptotically control Type I error, the convergence speed affects their validity in finite sample cases. 

In Figure~\ref{fig:large_scale_splits}, we repeat the large scale experiment introduced in Appendix~\ref{subsubsec:largescale} for specification test under different adaptive splitting ratios to further illustrate the tradeoff. We see that CMMD$(0.1)$ while theoretically converging to the right Type I error, in the large scale experiment, even when we are at $n=7500$, the method doesn't exhibit any Type I converging behaviour. This means it requires much more samples compare to other methods to exhibit converging behaviour. In comparison, we see CMMD$(0.2)$ exhibits converging behaviour starting from $n=4000$, while all other methods converge to the right Type I control much earlier in comparison.

\begin{figure}[!h]
    \centering
    \includegraphics[width=\linewidth]{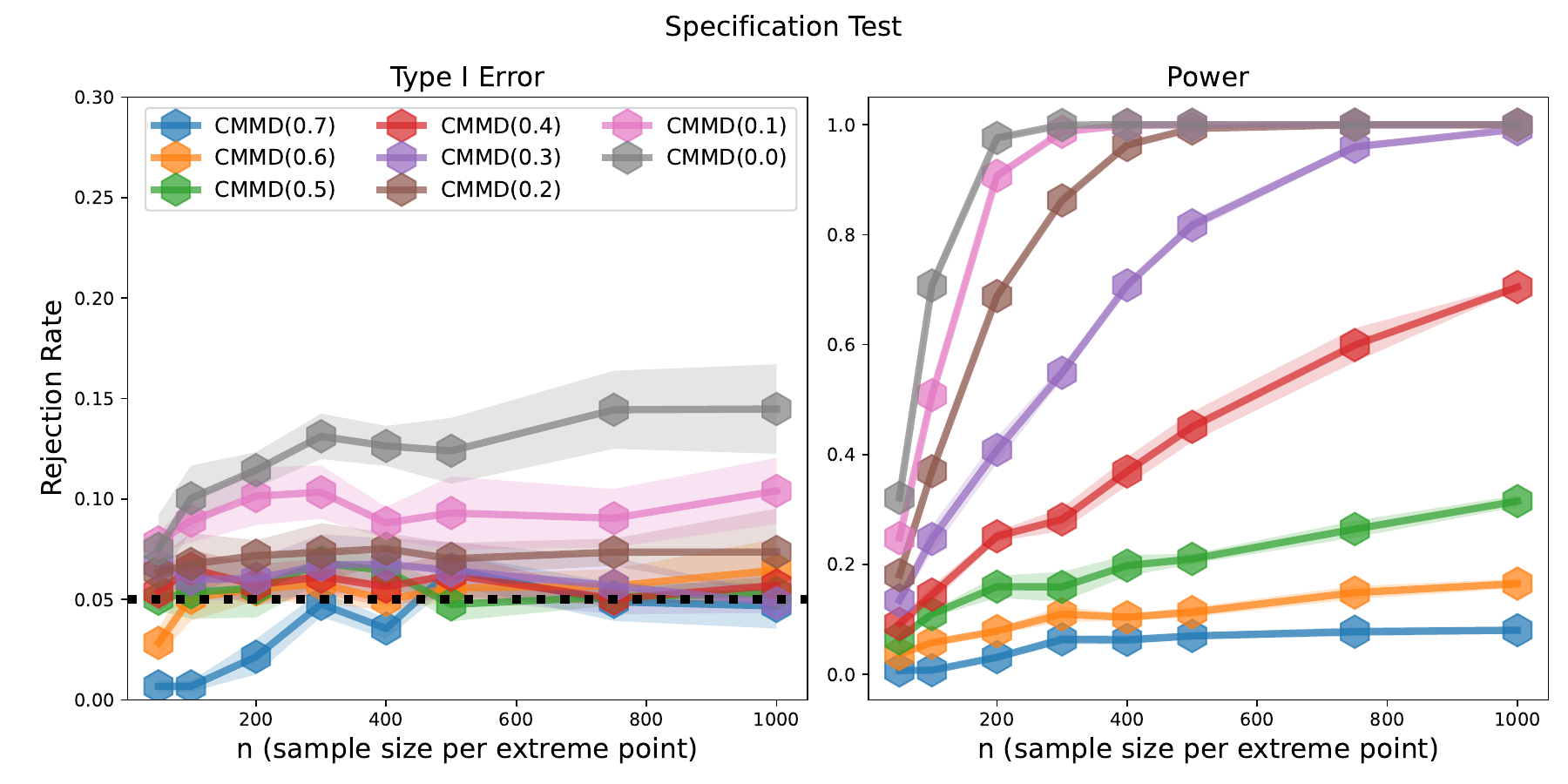}
    \caption{Demonstrating the tradeoff between Type I convergence and test power increase across different $\beta$ for $\beta$ in $\frac{n_t}{n_e} = \frac{1}{n_e^\beta}$. Rejection rate are averaged across $10$ set of convex weights and 1 standard deviation is reported. As we can see, the closer $\beta$ is to $0$, the slower the Type I convergence, but since we are using more testing samples in exchange, the power also increases.}
    \label{fig:more_split}
\end{figure}

\begin{figure}[!h]
    \centering
    \includegraphics[width=\linewidth]{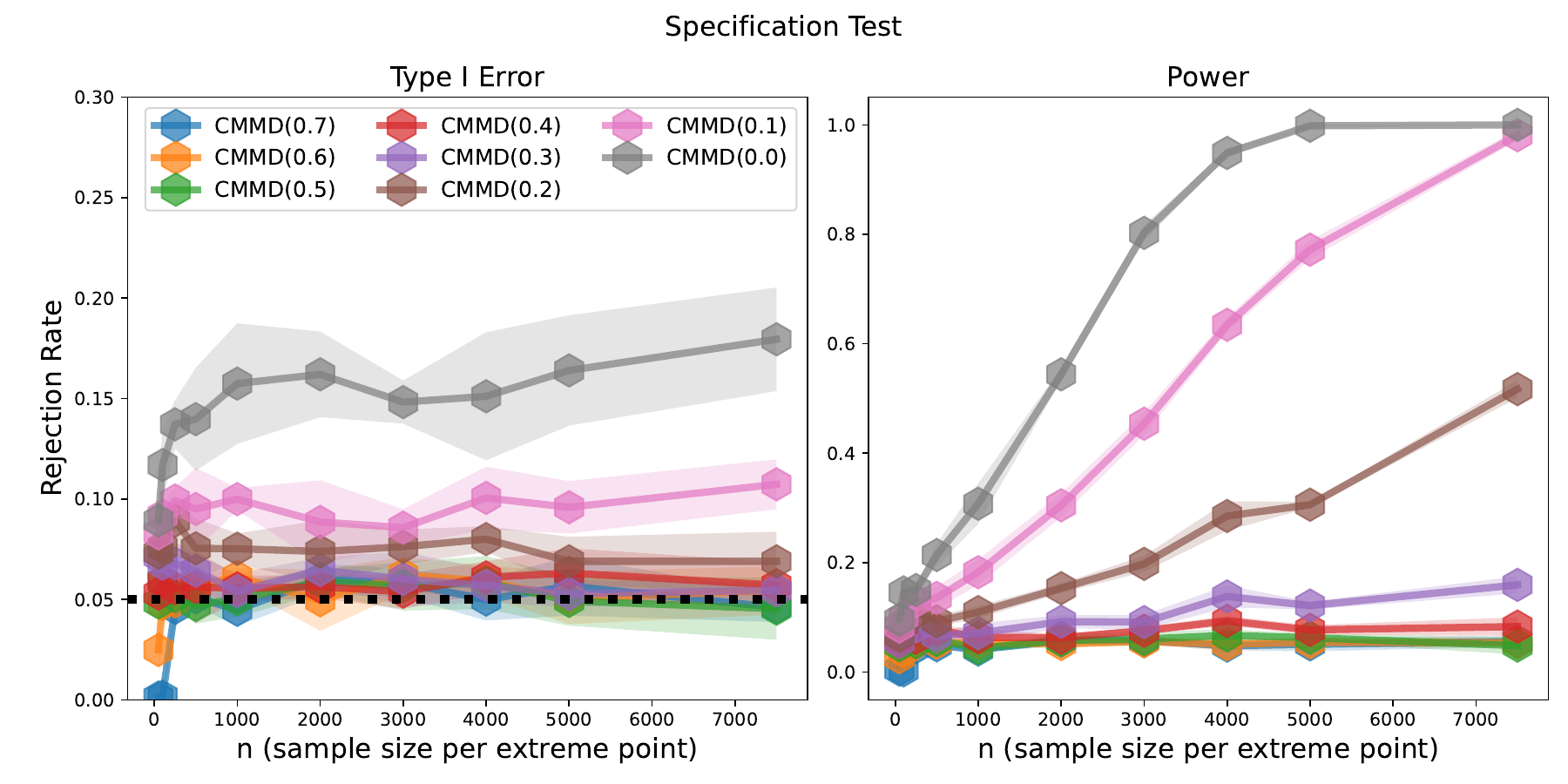}
    \caption{Repeating the large scale experiment introduced in Appendix~\ref{subsubsec:largescale} for specification test with different adaptive splitting ratios.}
    \label{fig:large_scale_splits}
\end{figure}

\subsubsection{Large Scale Experiment}
\label{subsubsec:largescale}
We demonstrate a large-scale experiment for specification tests. We test against a mixture of Gaussians with a mixture of student distribution with $10$ degrees of freedom. We use $5$ extreme point for $\cC_Y$. As we know, the larger the degrees of freedom for a student distribution, the closer it resembles a Gaussian distribution. This means we need much more samples to distinguish the two, compared to the case in the main text where we only have 3 degrees of freedom.

\begin{figure}[!h]
    \centering
    \includegraphics[width=\linewidth]{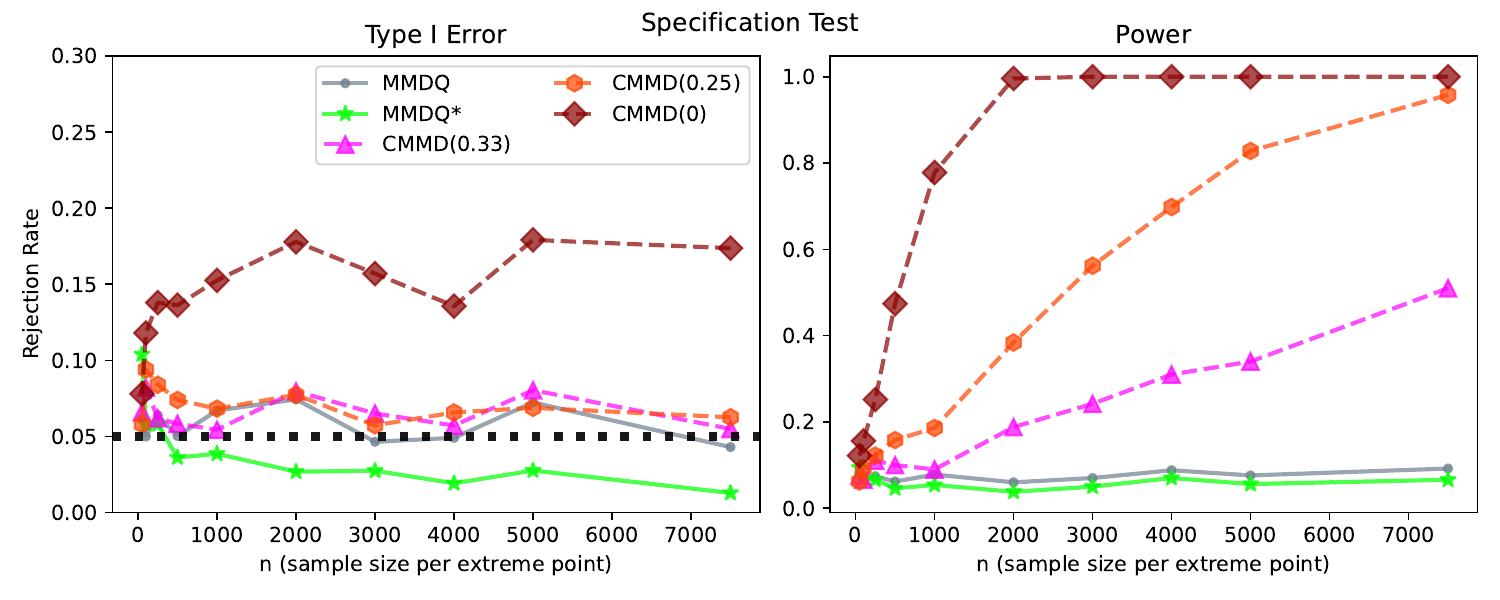}
    \caption{Large scale specification test comparing mixture of Gaussians with mixtures of student distributions with $10$ degrees of freedom. We see that fixed sample splitting approach yields consistent Type I inflation whereas using adaptive sample splitting shows Type I control at level $0.05$ asymptotically.}
    \label{fig:large_scale_specification}
\end{figure}

Figure~\ref{fig:large_scale_specification} illustrates the result for the large-scale experiment. We see that overall fixed sample splitting has a persistent inflated Type I and the other CMMD tests with adaptive splitting approach $0.05$ level. The studentisd tests exhibit very low power compared to our permutation approaches despite using more testing samples than ours.

\subsubsection{Comparison with Double-dipping Approaches}
\label{appendix_subsubsec: comparing_with_double_dipping}

In \citet{key_composite_2024} and \citet{bruck_distribution_2023}, double-dipping approaches were explored, where the same sample set is reused for both estimation and hypothesis testing. It is natural to examine how this method applies in our setting. Using the notation from Section \ref{sec: experiments}, we pick the split ratio $\rho$ such that $\nicefrac{n_t}{n_e} = \nicefrac{1}{n_e^\beta}$ for some $\beta\in[0,1]$, the double-dipping approach then corresponds to setting $n_e = n$ and $n_t = n^{1-\beta}$. The theoretical analysis of the test statistic in this scenario is challenging due to the inter-dependence between the estimated parameters and the test statistic samples (see \citep[Section 2.1]{bruck_distribution_2023} for an illustrative example).

We repeat the experiments we conducted in the main paper, including double-dipping approaches. The results are shown in Figure~\ref{fig: double_dip}. These methods are referred to as ddip$(\beta)$ in the plots. Empirically, we show that adaptive sample splitting strategies ($\beta \neq 0$) tend to produce very conservative Type I error, resulting in reduced test power compared to non-double-dipping methods with the same splitting ratio $\rho$. While the fixed split ratio method ($\beta = 0$) often achieves conservative but valid Type I error control and higher power, there is no theoretical guarantee of consistent performance. Notably, for the plausibility test, we observe an increase in Type I error as sample sizes grow—a trend also seen with non-double-dipping fixed sample splitting. In Figure~\ref{fig:large-scale-ddip}, we repeat the large-scale experiments described in Appendix~\ref{subsubsec:largescale} with double-dipping approaches. We see that the overly conservativeness of double-dipping approaches results in extremely low power compared to the non-double-dipping approaches that share the same sample splitting ratio.


\begin{figure}[!h]
    \centering
    \includegraphics[width=0.49\linewidth]{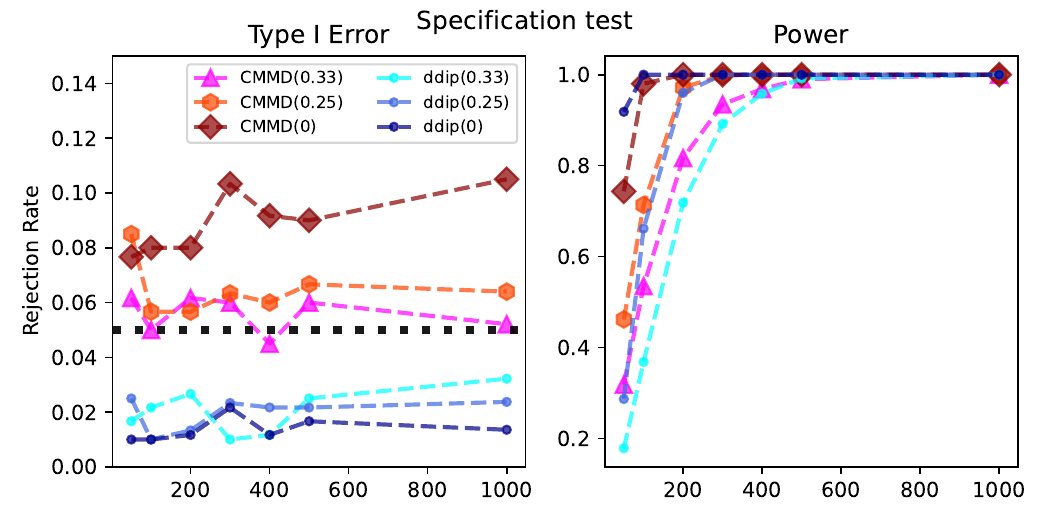}
    \includegraphics[width=0.49\linewidth]{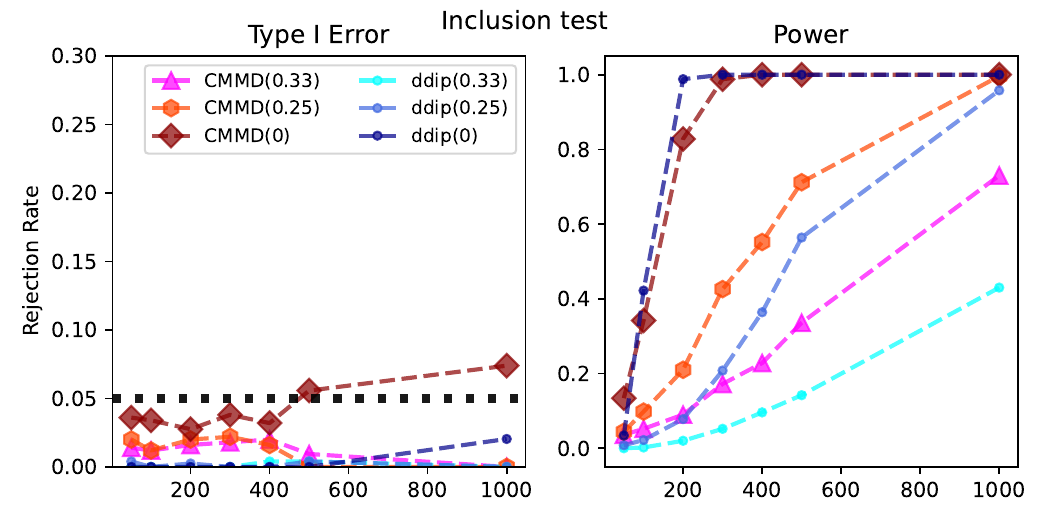}
    \includegraphics[width=0.49\linewidth]{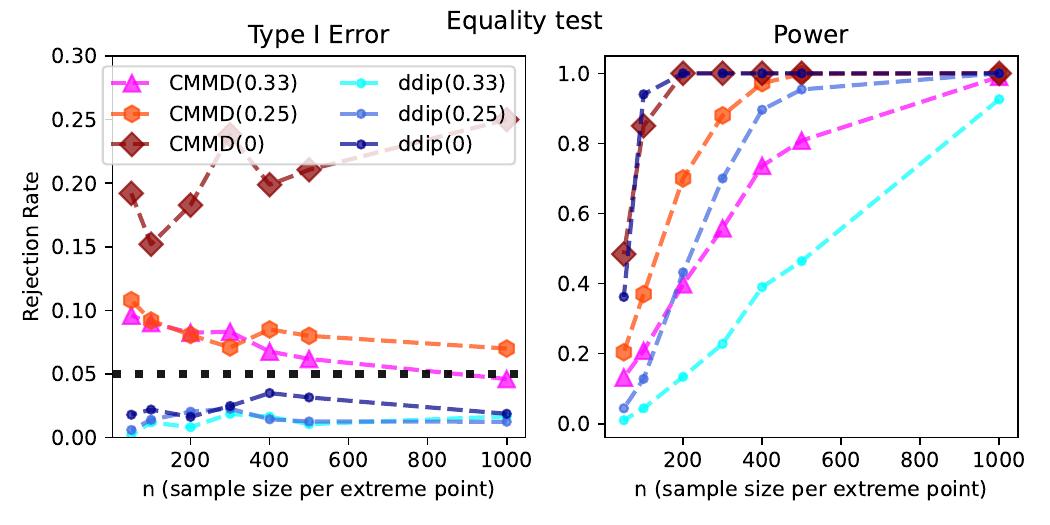}
    \includegraphics[width=0.49\linewidth]{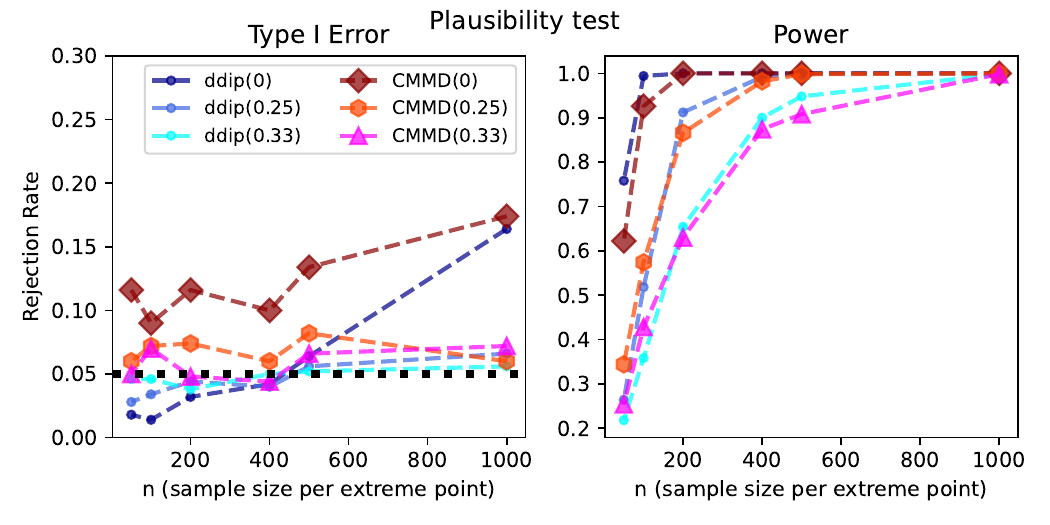}
    \caption{Comparing the performance of double-dipping-based methods with non-double-dipping-based methods. We observe that double-dipping methods often produce too conservative Type I control, resulting in lower power compared to their counterpart methods which shares the same sample splitting ratios. While in some cases ddip$(0)$ exhibits valid Type I control and yields high power compared to other adaptive splitting methods, in the plausibility test experiments we see ddip$(0)$ fails to control Type I. This means that in practice we should not use ddip$(0)$ because we do not know when ddip$(0)$ is valid or not.}
    \label{fig: double_dip}
\end{figure}

\begin{figure}[!h]
    \centering
    \includegraphics[width=0.8\linewidth]{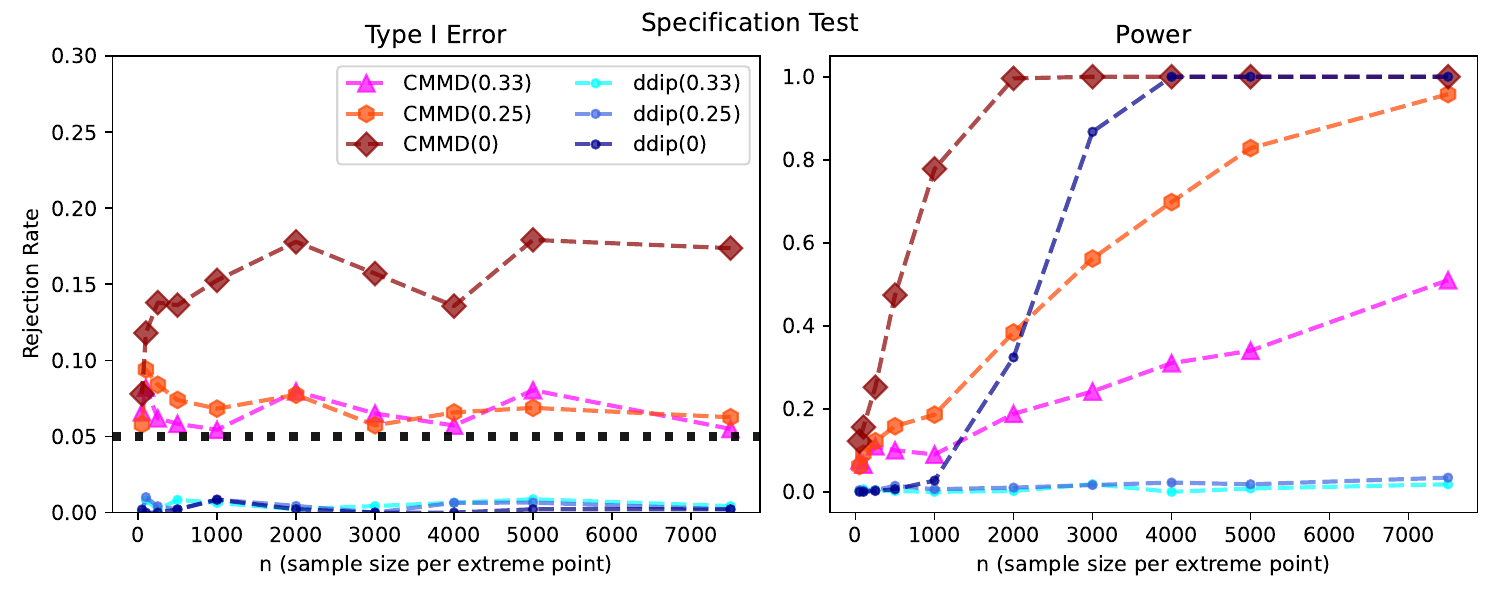}
    \caption{Repeating the large-scale experiment for the double dipping method and comparing the performance between standard sample splitting with double dipping approaches. We see that double-dipping approaches are extremely conservative, therefore resulting in very low power. \citet{key_composite_2024} also explored a double-dipping approach for composite goodness-of-fit where they also share the same observation that their method leads to very conservative results. Since double dipping approach has been shown empirically to be an invalid test in Figure~\ref{fig: double_dip} so even though it exhibits decent power here, we should not use this test in practice.}
    \label{fig:large-scale-ddip}
\end{figure}

\subsubsection{What if Some Extreme Points are Linearly Dependent?}
\label{appendix subsub linearly dependent}

To demonstrate that our credal test is robust to violations of Assumption~\ref{assumption 0}, we use the specification test as an example. To simulate the null hypothesis, we generate the credal set $\cC_Y = \operatorname{CH}(P_1, P_2, P_3, P_4)$, where $\operatorname{CH}$ is the convex hull operator, $P_4 = \frac{1}{3}P_1 + \frac{1}{3}P_2 + \frac{1}{3}P_3$, and $P_1, P_2, P_3$ are 10-dimensional multivariate Gaussians generated as described in Section~\ref{sec: experiments}. $P_X$ is then a convex weighted aggregation of the three extreme points. To simulate the alternative hypothesis, we generate $P_X$ using mixtures of Student’s t-distributions instead of Gaussian mixtures. Figure~\ref{fig:violation_of_assumption_1} illustrates this experiment. The tests with adaptive sample splitting approaches still converge to the required Type I error level, while the fixed sample splitting approach consistently shows inflated Type I errors.

\begin{figure}[!h]
\centering
\includegraphics[width=0.8\linewidth]{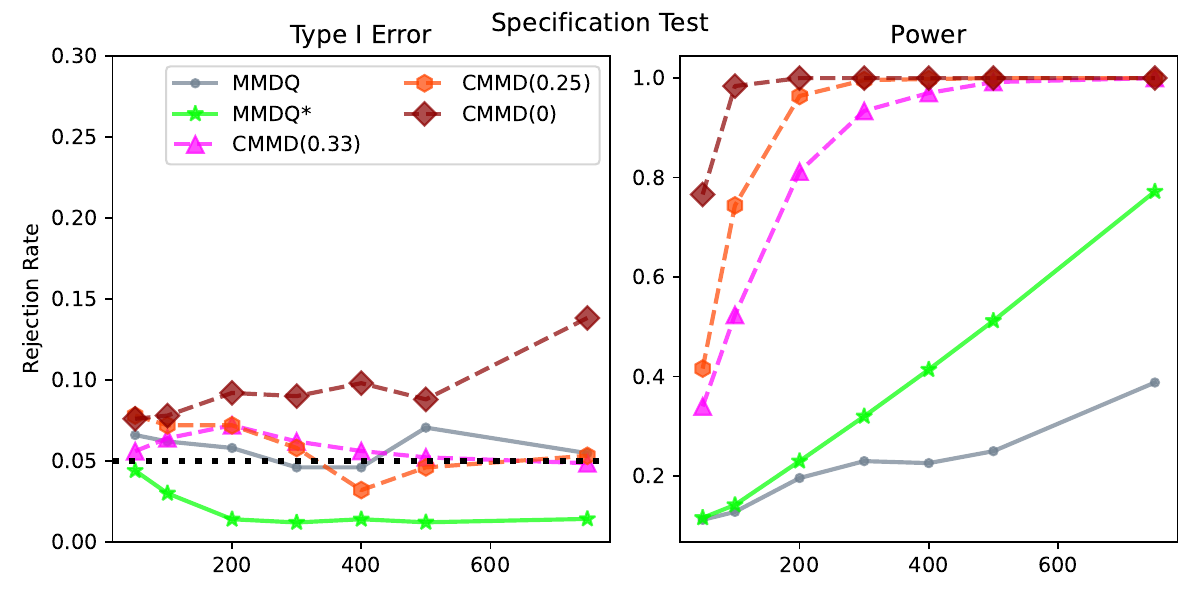}
\caption{Specification test results when Assumption~\ref{assumption 0} is violated. The tests with adaptive sample splitting approaches still converge to the required Type I error level, while the fixed sample splitting approach consistently shows inflated Type I error.}
\label{fig:violation_of_assumption_1}
\end{figure}

We now explain why this violation does not affect the test, using similar reasoning as to why multiple solutions in the plausibility test do not cause issues. When Assumption~\ref{assumption 0} is violated, some extreme points may become linearly dependent. In the specification test, this could result in the optimization procedure yielding multiple solutions. However, following the argument from the proof of Theorem~\ref{thm: main_theorem_h0}, any local minimum of the KCD, where $\nabla L(1, \bfeta) = 0$, will—due to the characteristicness of the kernel—produce a set of parameters $\bfeta^e$ that satisfies the null hypothesis $H_{0,\in}$ (see Proposition~\ref{prop: local_is_global} for details). Moreover, the uniform convergence of the objective function ensures that the estimator will converge to some solution of the population level objective. As we increase the sample size, the estimators approach one of the parameters in the solution set $\arg\min_{\bfeta \in \Delta_r} L(1,\bfeta)$. Therefore, using the adaptive sample splitting strategy, our test statistic asymptotically converges to the same distribution as if we had access to a certain set of true parameter. Combining this with the arguments in Appendix~\ref{subsubsec: null_dist_plausibility_multiplicity}, we justify why the test maintains proper Type I error control asymptotically.

\subsection{MNIST Experiments}
\label{appendix_subsec: mnist}

Following \citet{kubler2022automl} and \citet{schrab2023mmd}, we also validate our credal tests using the MNIST dataset~\citep{lecun1998mnist}. We utilise a pretrained image classifier to extract vector embeddings for the MNIST images. Specifically, we used the pretrained Resnet-18~\citep{he2016deep} model to extract 512-dimensional vectors for our images. The kernel between images is then an RBF kernel applied to these 512-dimensional vectors.

\begin{figure}[]
    \centering
    \includegraphics[width=0.49\linewidth]{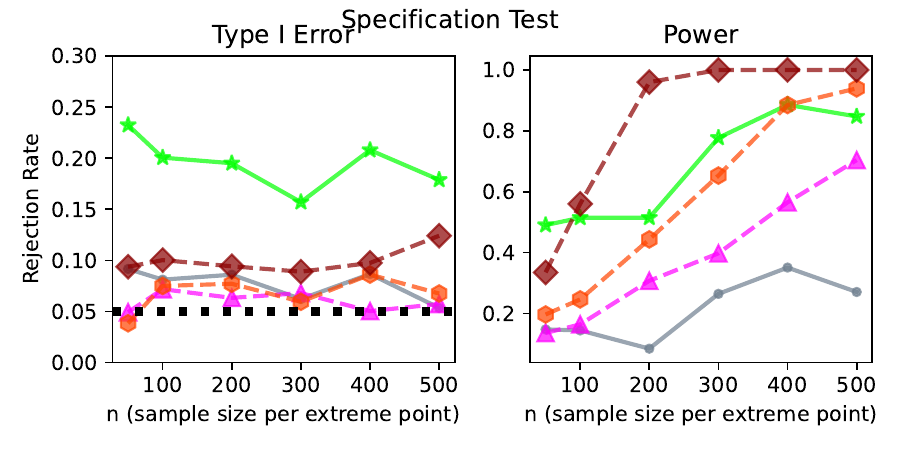}
    \includegraphics[width=0.49\linewidth]{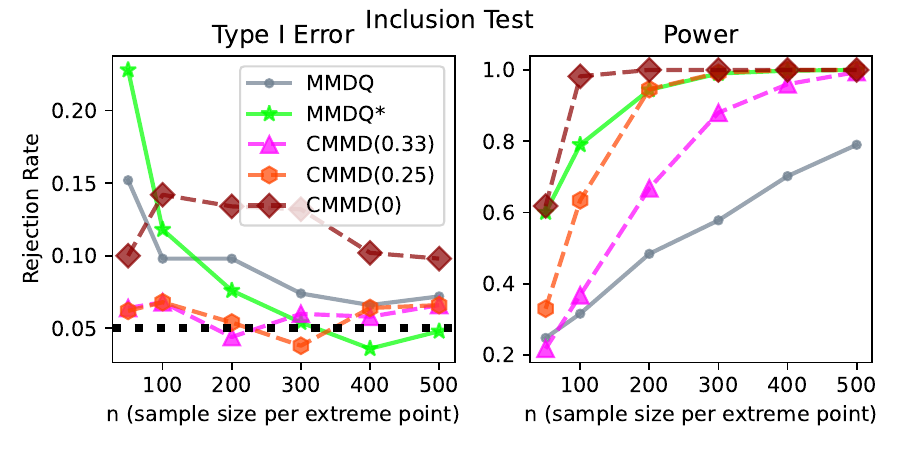}
    \includegraphics[width=0.49\linewidth]{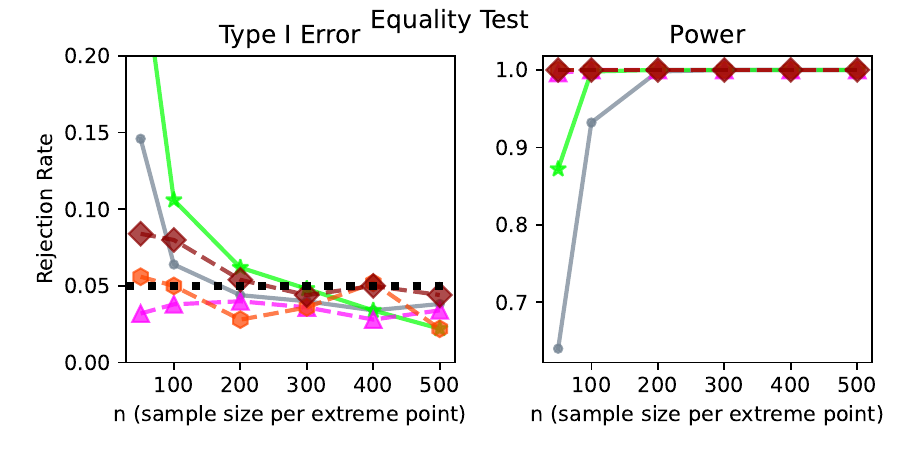}
    \includegraphics[width=0.49\linewidth]{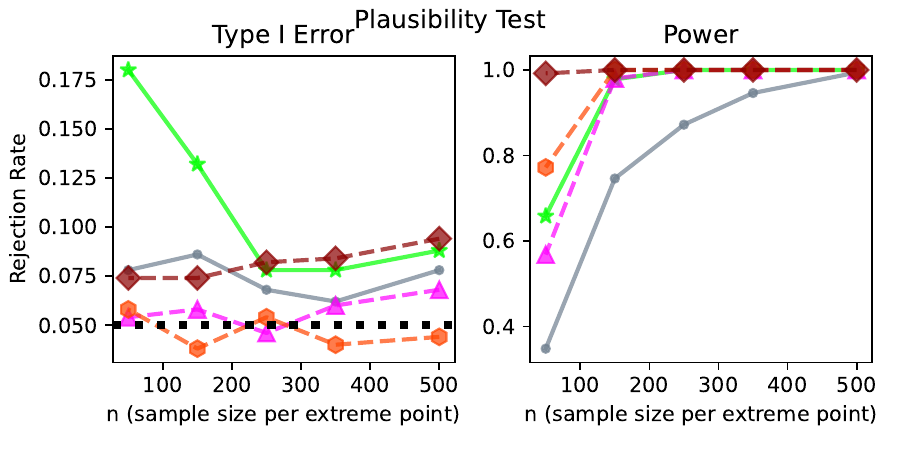}
    \caption{MNIST Credal Testing Experimental Results. In addition to the consistent Type I error inflation for CMMD$(0)$ using a fixed splitting strategy, we observe that MMDQ$^\star$ also exhibits significantly inflated Type I error rates for small sample sizes. In the specification test, MMDQ$^\star$ fails to achieve any Type I error control.}
    \label{fig:MNIST}
\end{figure}

In our experiments, each extreme point in a credal set represents the distribution of images for a specific digit.

\begin{itemize}
    \item \textbf{Specification test.} Under the null hypothesis, our credal set $\cC_Y$ consists of extreme points representing distributions of digits [1, 3, 7], while $P_X$ is a mixture distribution of these three digits. To simulate the alternative, $P_X$ remains a mixture distribution of digits [1, 3, 7], but the credal set $\cC_Y$ now consists of extreme points representing distributions of digits [1, 3, 9].
    \item \textbf{Inclusion test.} For the inclusion test, under the null hypothesis, we simulate three mixture distributions to construct $\cC_X$ based on a credal set $\cC_Y$ that includes extreme points for digits [1, 3, 7]. To simulate the alternative, $\cC_X$ is constructed similarly to the null setting, but the credal set $\cC_Y$ now includes extreme points for digits [1, 3, 9].
    \item \textbf{Equality test.} Under the null hypothesis, we build $\cC_X$ and $\cC_Y$ both constructed using images of $[1, 3, 7]$, and for the alternative, we modify $\cC_Y$ to be constructed using images of digits $[1, 3, 9]$.
    \item \textbf{Plausibility test.} Under the null hypothesis, $\cC_X$ are built using digits $[1, 3, 7]$ and $\cC_Y$ are built using digits $[1, 3, 9]$. Under the alternative, $\cC_Y$ are built using digits $[0,2,9]$.
\end{itemize}

Figure~\ref{fig:MNIST} illustrates the performance of our credal tests on the MNIST dataset. The results align with those from the synthetic experiments, showing consistent Type I error inflation for fixed splitting approaches. In contrast, adaptive sample splitting converges to the correct Type I level as sample size increases. Notably, MMDQ$^*$ either fails to maintain Type I error control or exhibits significantly inflated Type I error when sample sizes are small.

%% file: sections/appendix/E_background_on_kernel_testing.tex
\newpage
\section{Preliminary materials on kernel methods, kernel mean embeddings, and kernel two-sample tests.}
\label{appendix: kernel stuff}

Here, we provide preliminary materials on kernel methods, kernel mean embeddings, and kernel-based hypothesis testing for readers less familiar with these topics. For a foundational understanding of reproducing kernel Hilbert spaces (RKHS), we recommend the lecture notes by \citet{sejdinovic2012rkhs}. To gain an in-depth understanding of the kernel two-sample test, refer to the journal paper by \citet{gretton2012kernel}. Finally, for insights into how kernel mean embeddings serve as nonparametric representations of distributions and their applications, we suggest \citet{muandet2017kernel}.

\subsection{Kernel methods}
\label{appendix subsec: kernel methods}
We begin by defining what a kernel is. 

\begin{defn}[Kernel.]
    Let $\mathcal{X}$ be a nonempty set. A function $k:\mathcal{X} \times \cX \to \RR$ is called a kernel if there exist a real-valued Hilbert space $\mathcal{H}$ and a map $\phi:\cX\to\cH$ such that for all $x, x' \in\mathcal{X}$,
    \begin{align*}
        k(x,x') := \langle \phi(x), \phi(x') \rangle_{\mathcal{H}}.
    \end{align*}
\end{defn}
This can be understood intuitively as follows: Let $\mathcal{X}$ be a collection of TV series, and suppose we want to compare them. While computing an inner product directly in the space of TV series may not be meaningful, it becomes sensible when we compare extracted “features” of the series instead. For instance, we can define a feature map $\phi(x)$ that represents each TV series x using attributes such as its length, ratings, and production costs:
\begin{align*}
    \phi(x) = [\text{length, costs, ratings}].
\end{align*}
This representation allows us to analyse and compare TV series in a mathematical way. For this reason, $\phi$ is also known as a feature map. Kernels satisfying certain properties are core to their popularity in machine learning literature, and they have their special name, reproducing kernels.
\begin{defn}[Reproducing kernel~\citep{berlinet2011reproducing}]
    Let $\cH$ be a Hilbert space of real-valued functions defined on a non-empty set $\cX$. A function $k:\cX\times\cX\to\RR$ is a called a reproducing kernel of $\cH$ if it satisfies:
    \begin{itemize}
        \item $\forall x \in \mathcal{X}, k(\cdot, x)\in\mathcal{H},$
        \item $\forall x, \forall f\in\cH \langle f, k(\cdot, x)\rangle = f(x) $
    \end{itemize}
    The second property is also known as the reproducing property and the map $x:\mapsto k(\cdot, x)$ is often denoted as the canonical feature map of $x$. 
\end{defn}
In particular, for any $x, x' \in\cX$, 
\begin{align*}
    k(x,x') = \langle k(\cdot, x), k(\cdot, x') \rangle_\mathcal{H}
\end{align*}
From this illustration, it is obvious that a reproducing kernel is also a kernel with the canonical feature map as the $\phi$ map. For illustraiton purpose, let $\cX \subseteq \RR^d$, popular examples of kernels are:
\begin{itemize}
    \item Linear kernel: $k(x,x') = \langle x, x'\rangle$
    \item Polynomial kernel of degree $p$: $k(x,x') = (\langle x,x' \rangle + c)^p$
    \item Radial basis function kernel with bandwidth $\ell>0$: $k(x,x') = \exp\left(\frac{\|x-x'\|^2}{2\ell^2}\right)$
    \item Matérn Kernel with smoothness $\nu>0$, bandwidth $\ell>0$:
    \begin{align*}
        k(x,x') = \frac{1}{\Gamma(\nu)2^{\nu -1}}\left(\frac{\sqrt{2}\nu \|x-x'\|}{\ell}\right)K_\nu\left(\frac{\sqrt{2\nu}\|x-x'\|}{\ell}\right)
    \end{align*}
    where $K_\nu$ is the modified Bessel function of the second kind and $\Gamma(\nu)$ is the gamma function.
\end{itemize}


Another important notion in the literature of kernel methods is the reproducing kernel Hilbert space, a space of functions $f:\cX\to\RR$ adhere to specific properties, defined as follows.
\begin{defn}[Reproducing kernel Hilbert space.] A Hilbert space of real-valued functions $f:\cX\to\RR$, defined on a non-empty set $\cX$ is said to be a Reproducing kernel Hilbert space (RKHS) if the evaluation function $\delta_x: f\mapsto f(x)$ is continuous $\forall x\in\cX$.
\end{defn}
Moore and Aronsjn \citep{aronszajn1950theory} have shown that not only given any RKHS $\cH$, we can define a unique reproducing kernel associated with $\cH$, but for any reproducing kernel $k$, there corresponds an unique RKHS $\cH$. This is known as the Moore-Arnsjon theorem. But the readers may now wonder why we are even interested in working with such a specific function space instead of the more general space of bounded continuous real-valued functions. Turns out, with common choices of kernels such as RBF and Matern, the corresponding RKHS is dense in the space of bounded continuous functions, a property known as $C_0$ universality. This is a desirable property, meaning that for any bounded continuous function of interest, we can find an element in the RKHS that can get arbitrarily close to it. Please refer to \citet{sriperumbudur2011universality} for further discussion.

\subsection{Kernel Mean Embeddings}
\label{appendix subsec: kme}
Given an instance \( x \in \mathcal{X} \), the canonical feature map \( k(\cdot, x) \) serves as its representation. Now, given a random variable \( X \) distributed according to a law \( P_X \), can we construct a representation of \( P_X \) using the feature map? This can be done by the kernel mean embedding~\citep{smola2007hilbert,muandet2017kernel} for a specific class of kernels with properties satisfying by most commonly used kernels such as the RBF and Matern kernel.
\begin{defn}
    The kernel mean embedding $\mu: \cX \to \cH_k$ of a distribution $P_X$ is defined as:
    \begin{align*}
        \mu(P_X) := \mathbb{E}_{X\sim P_X}[k(\cdot, X)].
    \end{align*}
\end{defn}
For characteristic kernels, the mapping $\mu:P_X\to \mu(P_X)$ is injective. We often write $\mu(P_X)$ simply as $\mu_{P_X}$ when the context is clear.

This representation is particularly convenient in practice since it can be estimated directly from samples. In other words, without requiring any parametric assumptions on \( P_X \), we can approximate its representation using only observed samples. The estimation is also quite (statistically) efficient, as can be seen by the following result from \citet{tolstikhin2017minimax},
\begin{theorem}[\citet{tolstikhin2017minimax} Proposition A.1]
    Let $X_1,\dots,X_n \overset{iid}{\sim} P_X$ and let $k:\cX\times\cX \to \RR$ be a continuous positive definite kernel on a separable topological space $\cX$ with $\sup_{x\in\cX} k(x,x) \leq C_k\infty$. Then for any $\delta \in (0,1)$ with probability at least $1-\delta$,
    \begin{align*}
        \|\frac{1}{n}\sum_{i=1}^n k(\cdot, X_i) - \mu_{P_X}\|_{\cH} \leq \sqrt{\frac{C_k}{n}} + \sqrt{\frac{2C_k \log{\frac{1}{\delta}}}{n}}.
    \end{align*}
\end{theorem}
This representation of distributions in the RKHS via embedding operations has since led to numerous developments, including extensions to model conditional distributions and independences, with applications to fairness~\citep{tan2020learning}, reinforcement learning~\citep{zhang2021distributional}, domain generalization~\citep{muandet_domain_2013}, distribution regression~\citep{Szabo16:DR}, causal inference~\citep{sejdinovic2024overview}, generative model~\citep{briol_statistical_2019}, feature attributions~\citep{NEURIPS2022_54bb63ea,chau_explaining_2023,hu2022explaining}, and many more. Most connected to the theme of the current paper is the incorporation of epistemic uncertainty in distribution representations, which has mostly been studied from the Bayesian perspective, considered in the works of \citet{flaxman2016bayesian,hsu2019bayesian, hsu2019bayesian2, chau2021bayesimp, chau2021deconditional, zhang2022bayesian}. In comparison, we considered a credal version of kernel mean embeddings of the (finitely generated) credal set instead.

\subsection{Kernel Two-sample Test}
\label{appendix subsec: kernel two-sample test}

\subsubsection{MMD as Test Statistic} 

Recall the goal of a two-sample test is to use iid samples from $X_1,\dots, X_n\overset{iid}{\sim}P_X$ and $Y_1,\dots,Y_m\overset{iid}{\sim}P_Y$ to assess whether there are sufficient evidence to reject the null hypothesis $H_0: P_X = P_Y$. A fundamental component of any hypothesis test is the choice of a test statistic that quantifies the deviation of the observed samples from what is expected under the null. In the case of testing whether $P_X = P_Y$, a common approach is to use some kind of distributional divergence measures that tell us how far $P_Y$ is from $P_X$. In statistics, one popular class of divergence measure is the integral probability metric (IPM,\citep{muller1997integral}), which is defined as,
\begin{defn}[Integral Probability Metric \citep{muller1997integral}]
    Given distributions $P_X, P_Y$ on $\cX$ and a function class $\cF:\cX\to\RR$, an integral probability metric $d$ is defined as 
    \begin{align*} 
    d(P_X, P_Y) = \sup_{f\in\mathcal{F}}\left|\mathbb{E}_{X\sim P_X}[f(X)] - \mathbb{E}_{Y\sim P_Y}[f(Y)]\right|
    \end{align*}
\end{defn}
Popular divergence measures can be recovered based on how we specify the set of functions $\mathcal{F}$, see \citet{sriperumbudur2009integral}. For example, when $\cF = \{f: \|f\|_{\infty} + \|f\|_L\ \leq 1\}$ where $\|f\|_L$ is the Lipschitz semi-norm of $f$, then we have $d$ the Dudley metric. When $\cF =\{f: \|f\|_L\leq 1\}$ we recover the Kantorovich metric. When $\cF=\{f: \|f\|_{\infty} \leq 1\}$, we recover the total variation distance. 

Specifically, when $\cF = \{f\in\mathcal{H}_k: \|f\|_{\cH_k} \leq 1\}$ for some RKHS $\cH_k$, then we have the popular maximum mean discrepancy (MMD), which corresponds exactly to the RKHS norm between the mean embeddings of the kernel of $P_X$ and $P_Y$. To see this, realise,
\begin{align*}
    d(P_X, P_Y) 
        &= \operatorname{MMD}(P_X, P_Y) \\
        &= \sup_{f\in \cH_k: \|f\|_k\leq 1} \left| \EE_{X\sim P_X}[f(X)] - \EE_{Y\sim P_Y}[f(Y)]\right| \\
        &= \sup_{f\in \cH_k: \|f\|_k\leq 1} \left|\langle f, \EE_{X\sim P_X}[k(\cdot, X)] - \EE_{Y\sim P_Y}[k(\cdot, Y)]\rangle\right| \tag{reproducing property}\\
        &= \|\mu_{P_X} - \mu_{P_Y}\|_{\cH_k} \tag{Cauchy Schwarz} \\
\end{align*}
Through the kernel mean embeddings, we can estimate the proximity of $P_X$ and $P_Y$ purely through samples with no parametric assumptions on $P_X, P_Y$. This property underpins the popularity of MMD.

While it is possible to estimate $\operatorname{MMD}(P_X,P_Y)$ through the empirical estimate of the kernel mean embeddings $\mu_{P_X}, \mu_{P_Y}$, this could lead to bias estimation since, by expanding out the terms,
\begin{align*}
    \widehat{\operatorname{MMD}^2_b(P_X,P_Y)} 
        &= \|\frac{1}{n}\sum_{i=1}^nk(\cdot, X_i) - \frac{1}{m}\sum_{j=1}^mk(\cdot, Y_j)\|^2_{\cH_k} \\
        &= \frac{1}{m^2}\sum_{i,j=1}^m k(Y_i,Y_j) -\frac{2}{mn}\sum_{i=1}^n\sum_{j=1}^mk(X_i, Y_j) + \frac{1}{n^2}\sum_{i,j=1}^n k(X_i, X_j)
\end{align*}
the terms $k(X_i,X_i)$ and $k(Y_j, Y_j)$ gives unwanted bias to the estimation. While this bias goes to $0$ asymptotically, it does not disappear in finite sample cases. Instead, people consider the unbiased estimator,
\begin{align*}
    \widehat{\operatorname{MMD}^2(P_X, P_Y)} = \frac{1}{m(m-1)}\sum_{i=1, j\neq i}^m k(Y_i, Y_j) -\frac{2}{mn}\sum_{i=1}^n\sum_{j=1}^m k(X_i, Y_j) + \frac{1}{n(n-1)}\sum_{i=1, j\neq i}^n k(X_i, X_j)
\end{align*}
For simplicity, we denote the random statistic $\widehat{\operatorname{MMD}^2(P_X, P_Y)}$ as $T(\bfZ)$ with $\bfZ :=\{X_1,\dots,X_n,Y_1,\dots, Y_m\}$ and $T(\bfz)$ as the realised statistic with $\bfz := \{x_1,\dots, x_n, y_1,\dots,y_m\}$.

\subsubsection{Permutation Test}
\label{appendix subsubsec: permutation}
In hypothesis testing, the decision to reject the null hypothesis hinges on selecting an appropriate threshold \(\gamma\) such that  
\begin{equation*}
    \operatorname{Pr}(T(\mathbf{Z}) \geq \gamma \mid H_0) \leq \alpha,
\end{equation*}  
where \(\alpha\) is the prespecified Type I error control level, commonly set to \(0.05\) by convention. A standard approach to determining \(\gamma\) involves analyzing the asymptotic distribution of the test statistic under the null hypothesis and selecting the \((1 - \alpha)\)-quantile as the threshold. While this method does not ensure exact Type I error control in finite samples, it provides asymptotic control when the distribution of the test statistic closely approximates its asymptotic counterpart.

However, in the case for $T(\bfZ)$ the unbiased MMD-squared estimate, the asymptotic distribution is an infinite sum of chi-square distributions, as shown in \citet[Theorem 12.]{gretton2012kernel}. We provide a simplified version of the result here for completeness, now for simplicity assume $n = m$,
\begin{theorem}
    Under the conditions satisfied in \citet{gretton2012kernel}[Theorem 12.], it follows that:
    \begin{align*}
        n\widehat{\operatorname{MMD}^2(P_X, P_Y)} \overset{D}{\to}\sum_{i=1}^\infty \lambda_i Z_i^2
    \end{align*}
    where $(Z_i)_{i\geq 1}$ are a collection of iid standard normal random variable and $(\lambda_i)_{i\geq 1}$ are constants depended on the choice of kernel, where $\sum_{i=1}^\infty \lambda_i < \infty$.
\end{theorem}
The key takeaway from this result is that the asymptotic distribution of our test statistic is not accessible. However, we can resort to permutation procedure to estimate the rejection threshold. First, recall $\bfZ = \{X_1,\dots, X_n, Y_1,\dots, Y_n\}$ and define $\cG$ as a subset of the permutation group for $2n$ elements of size $M$. Under the null of $P_X=P_Y$, these samples are exchangeable, meaning the random statistic $T(\bfZ)$ and $T(g\bfZ)$ share the same distribution for any $g\in \cG$. For simplicity, we assume no ties. Now specify a level $\alpha$, the permutation test can be conducted as follows:
\begin{enumerate}
    \item Compute $T(g\bfz)$ for each $g\in\cG$.
    \item Sort $\{T(g\bfz)\}_{g\in\cG}$, such that,
    \begin{align*}
        T(g^{(1)}\bfz) < T(g^{(2)}\bfz),\dots, < T(g^{(M)}\bfz)
    \end{align*}
    \item Now pick the $M - \lfloor M\alpha \rfloor^{th}$ element of this sequence as the rejection threshold, meaning that we reject the null hypothesis if $T(\bfz) \geq T(g^{(M-\lfloor M\alpha \rfloor}\bfz)$.
\end{enumerate}
This procedure provides us a finite-sample guarantee of exact Type I control, that is 
\begin{align*}
    \operatorname{Pr}(\text{Reject }H_0\mid H_0) = \alpha.
\end{align*}
To see why, look at the following derivation, which can also be found in \citet[Chapter 15]{lehmann_testing_2005}. First recall,
\begin{align*}
    \sum_{g\in\cG} \mathbf{1}[T(g\bfz) \geq T(g^{(M - \lfloor M\alpha\rfloor)}\bfz)] &= M\alpha \\
    \implies \EE[\sum_{g\in\cG} \mathbf{1}[T(g\bfZ) \geq T(g^{(M - \lfloor M\alpha\rfloor)}\bfZ)]] &= M\alpha \\
    \implies \sum_{g\in\cG} \EE\left[\mathbf{1}[T(g\bfZ) \geq T(g^{(M-\lfloor M\alpha \rfloor)}\bfZ)]\right] &= M\alpha \\
    \implies \EE\left[\mathbf{1}[T(g\bfZ) \geq T(g^{(M-\lfloor M\alpha \rfloor)}\bfZ)]\right] &= \alpha
\end{align*}
Now due to exchangeability, $\EE\left[\mathbf{1}[T(g\bfZ) \geq T(g^{(M-\lfloor M\alpha \rfloor)}\bfZ)]\right]$ = $\EE\left[\mathbf{1}[T(\bfZ) \geq T(g^{(M-\lfloor M\alpha \rfloor)}\bfZ)]\right]$, therefore 
\begin{align*}
    \operatorname{Pr}(\text{Reject }H_0\mid H_0) = \operatorname{Pr}(T(\bfZ) \geq T(g^{(M - \lfloor M\alpha\rfloor)}\bfZ)) = \alpha.
\end{align*}
As a result, we get finite Type I error control exactly at level $\alpha$.


%% file: main.bbl
\begin{thebibliography}{116}
\providecommand{\natexlab}[1]{#1}
\providecommand{\url}[1]{\texttt{#1}}
\expandafter\ifx\csname urlstyle\endcsname\relax
  \providecommand{\doi}[1]{doi: #1}\else
  \providecommand{\doi}{doi: \begingroup \urlstyle{rm}\Url}\fi

\bibitem[Abellan and Masegosa(2010)]{abellan2010ensemble}
Joaquin Abellan and Andres~R Masegosa.
\newblock An ensemble method using credal decision trees.
\newblock \emph{European journal of operational research}, 205\penalty0 (1):\penalty0 218--226, 2010.

\bibitem[Abell{\'a}n et~al.(2017)Abell{\'a}n, Mantas, and Castellano]{abellan2017random}
Joaqu{\'\i}n Abell{\'a}n, Carlos~J Mantas, and Javier~G Castellano.
\newblock A random forest approach using imprecise probabilities.
\newblock \emph{Knowledge-Based Systems}, 134:\penalty0 72--84, 2017.

\bibitem[Abellán and Gómez(2006)]{abellan_measures_2006}
Joaquín Abellán and Manuel Gómez.
\newblock Measures of divergence on credal sets.
\newblock \emph{Fuzzy Sets and Systems}, 157\penalty0 (11):\penalty0 1514--1531, June 2006.
\newblock ISSN 01650114.
\newblock \doi{10.1016/j.fss.2005.11.021}.

\bibitem[Aitkin and Rubin(1985)]{aitkin_estimation_1985}
Murray Aitkin and Donald~B. Rubin.
\newblock Estimation and {Hypothesis} {Testing} in {Finite} {Mixture} {Models}.
\newblock \emph{Journal of the Royal Statistical Society. Series B (Methodological)}, 47\penalty0 (1):\penalty0 67--75, 1985.
\newblock ISSN 0035-9246.
\newblock Publisher: [Royal Statistical Society, Wiley].

\bibitem[Andersen et~al.(2013)Andersen, Dahl, Vandenberghe, et~al.]{andersen2013cvxopt}
Martin~S Andersen, Joachim Dahl, Lieven Vandenberghe, et~al.
\newblock Cvxopt: A python package for convex optimization.
\newblock \emph{Available at cvxopt. org}, 54, 2013.

\bibitem[Aronszajn(1950)]{aronszajn1950theory}
Nachman Aronszajn.
\newblock Theory of reproducing kernels.
\newblock \emph{Transactions of the American mathematical society}, 68\penalty0 (3):\penalty0 337--404, 1950.

\bibitem[Augustin et~al.(2014)Augustin, Coolen, De~Cooman, and Troffaes]{augustin_introduction_2014}
Thomas Augustin, Frank P.~A. Coolen, Gert De~Cooman, and Matthias C.~M. Troffaes, editors.
\newblock \emph{Introduction to imprecise probabilities}.
\newblock Wiley series in probability and statistics. Wiley, Hoboken, NJ, 2014.
\newblock ISBN 978-0-470-97381-3.

\bibitem[Baringhaus and Franz(2004)]{baringhaus2004new}
Ludwig Baringhaus and Carsten Franz.
\newblock On a new multivariate two-sample test.
\newblock \emph{Journal of multivariate analysis}, 88\penalty0 (1):\penalty0 190--206, 2004.

\bibitem[Bauer and Hackl(1987)]{bauer1987multiple}
Peter Bauer and Peter Hackl.
\newblock Multiple testing in a set of nested hypotheses.
\newblock \emph{Statistics}, 18\penalty0 (3):\penalty0 345--349, 1987.

\bibitem[Bellot and van~der Schaar(2021)]{bellot_kernel_2021}
Alexis Bellot and Mihaela van~der Schaar.
\newblock Kernel {Hypothesis} {Testing} with {Set}-valued {Data}, February 2021.
\newblock arXiv:1907.04081 [stat].

\bibitem[Berger et~al.(1994)Berger, Moreno, Pericchi, Bayarri, Bernardo, Cano, De~la Horra, Mart{\'\i}n, R{\'\i}os-Ins{\'u}a, Betr{\`o}, et~al.]{berger1994overview}
James~O Berger, El{\'\i}as Moreno, Luis~Raul Pericchi, M~Jes{\'u}s Bayarri, Jos{\'e}~M Bernardo, Juan~A Cano, Juli{\'a}n De~la Horra, Jacinto Mart{\'\i}n, David R{\'\i}os-Ins{\'u}a, Bruno Betr{\`o}, et~al.
\newblock An overview of robust bayesian analysis.
\newblock \emph{Test}, 3\penalty0 (1):\penalty0 5--124, 1994.

\bibitem[Berlinet and Thomas-Agnan(2011)]{berlinet2011reproducing}
Alain Berlinet and Christine Thomas-Agnan.
\newblock \emph{Reproducing kernel Hilbert spaces in probability and statistics}.
\newblock Springer Science \& Business Media, 2011.

\bibitem[Bernoulli(1713)]{bernoulli1713}
Jacob Bernoulli.
\newblock \emph{Ars Conjectandi}.
\newblock 1713.

\bibitem[Boyd and Vandenberghe(2004)]{boyd2004convex}
Stephen Boyd and Lieven Vandenberghe.
\newblock \emph{Convex optimization}.
\newblock Cambridge university press, 2004.

\bibitem[Bradbury et~al.(2018)Bradbury, Frostig, Hawkins, Johnson, Leary, Maclaurin, Necula, Paszke, Vander{P}las, Wanderman-{M}ilne, and Zhang]{jax2018github}
James Bradbury, Roy Frostig, Peter Hawkins, Matthew~James Johnson, Chris Leary, Dougal Maclaurin, George Necula, Adam Paszke, Jake Vander{P}las, Skye Wanderman-{M}ilne, and Qiao Zhang.
\newblock {JAX}: composable transformations of {P}ython+{N}um{P}y programs, 2018.

\bibitem[Briol et~al.(2019)Briol, Barp, Duncan, and Girolami]{briol_statistical_2019}
Francois-Xavier Briol, Alessandro Barp, Andrew~B. Duncan, and Mark Girolami.
\newblock Statistical {Inference} for {Generative} {Models} with {Maximum} {Mean} {Discrepancy}, June 2019.
\newblock arXiv:1906.05944 [cs, math, stat].

\bibitem[Bromberger(1971)]{bromberger1971science}
Sylvain Bromberger.
\newblock Science and the forms of ignorance.
\newblock \emph{Observation and Theory in Science}, pages 112--27, 1971.

\bibitem[Bronevich and Spiridenkova(2017)]{bronevich_characteristics_2017}
Andrey~G. Bronevich and Natalia~S. Spiridenkova.
\newblock Some {Characteristics} of {Credal} {Sets} and {Their} {Application} to {Analysis} of {Polls} {Results}.
\newblock \emph{Procedia Computer Science}, 122:\penalty0 572--578, 2017.
\newblock ISSN 18770509.
\newblock \doi{10.1016/j.procs.2017.11.408}.

\bibitem[Br{\"u}ck et~al.(2023)Br{\"u}ck, Fermanian, and Min]{bruck2023distribution}
Florian Br{\"u}ck, Jean-David Fermanian, and Aleksey Min.
\newblock Distribution free mmd tests for model selection with estimated parameters.
\newblock \emph{arXiv preprint arXiv:2305.07549}, 2023.

\bibitem[Brück et~al.(2023)Brück, Fermanian, and Min]{bruck_distribution_2023}
Florian Brück, Jean-David Fermanian, and Aleksey Min.
\newblock Distribution free {MMD} tests for model selection with estimated parameters, May 2023.
\newblock arXiv:2305.07549 [stat].

\bibitem[Caprio et~al.(2023)Caprio, Dutta, Jang, Lin, Ivanov, Sokolsky, and Lee]{caprio2023credal}
Michele Caprio, Souradeep Dutta, Kuk~Jin Jang, Vivian Lin, Radoslav Ivanov, Oleg Sokolsky, and Insup Lee.
\newblock Credal bayesian deep learning.
\newblock \emph{arXiv e-prints}, pages arXiv--2302, 2023.

\bibitem[Caprio et~al.(2024)Caprio, Sultana, Elia, and Cuzzolin]{caprio_credal_2024}
Michele Caprio, Maryam Sultana, Eleni Elia, and Fabio Cuzzolin.
\newblock Credal {Learning} {Theory}, February 2024.
\newblock arXiv:2402.00957 [cs, stat].

\bibitem[Chau et~al.(2021{\natexlab{a}})Chau, Bouabid, and Sejdinovic]{chau2021deconditional}
Siu~Lun Chau, Shahine Bouabid, and Dino Sejdinovic.
\newblock Deconditional downscaling with gaussian processes.
\newblock \emph{Advances in Neural Information Processing Systems}, 34:\penalty0 17813--17825, 2021{\natexlab{a}}.

\bibitem[Chau et~al.(2021{\natexlab{b}})Chau, Ton, Gonz{\'a}lez, Teh, and Sejdinovic]{chau2021bayesimp}
Siu~Lun Chau, Jean-Francois Ton, Javier Gonz{\'a}lez, Yee Teh, and Dino Sejdinovic.
\newblock Bayesimp: Uncertainty quantification for causal data fusion.
\newblock \emph{Advances in Neural Information Processing Systems}, 34:\penalty0 3466--3477, 2021{\natexlab{b}}.

\bibitem[Chau et~al.(2022)Chau, Hu, González, and Sejdinovic]{NEURIPS2022_54bb63ea}
Siu~Lun Chau, Robert Hu, Javier González, and Dino Sejdinovic.
\newblock {RKHS}-{SHAP}: {Shapley} values for kernel methods.
\newblock In S.~Koyejo, S.~Mohamed, A.~Agarwal, D.~Belgrave, K.~Cho, and A.~Oh, editors, \emph{Advances in neural information processing systems}, volume~35, pages 13050--13063. Curran Associates, Inc., 2022.

\bibitem[Chau et~al.(2023)Chau, Muandet, and Sejdinovic]{chau_explaining_2023}
Siu~Lun Chau, Krikamol Muandet, and Dino Sejdinovic.
\newblock Explaining the {Uncertain}: {Stochastic} {Shapley} {Values} for {Gaussian} {Process} {Models}, May 2023.
\newblock arXiv:2305.15167 [cs, stat].

\bibitem[Chen and Lei(2024)]{chen2024biased}
Yuchen Chen and Jing Lei.
\newblock De-biased two-sample u-statistics with application to conditional distribution testing.
\newblock \emph{arXiv preprint arXiv:2402.00164}, 2024.

\bibitem[Ch{\'e}rief-Abdellatif and Alquier(2022)]{cherief2022finite}
Badr-Eddine Ch{\'e}rief-Abdellatif and Pierre Alquier.
\newblock Finite sample properties of parametric mmd estimation: robustness to misspecification and dependence.
\newblock \emph{Bernoulli}, 28\penalty0 (1):\penalty0 181--213, 2022.

\bibitem[Choquet(1953)]{choquet1953}
G.~Choquet.
\newblock Théorie des capacités.
\newblock \emph{Ann. Inst. Fourier 5 (1953/1954) 131–292.}, 1953.

\bibitem[Cozman(2000)]{cozman2000credal}
Fabio~G Cozman.
\newblock Credal networks.
\newblock \emph{Artificial intelligence}, 120\penalty0 (2):\penalty0 199--233, 2000.

\bibitem[Cozman(2008)]{cozman2008sets}
Fabio~G Cozman.
\newblock Sets of probability distributions and independence.
\newblock \emph{SIPTA Summer School Tutorials, July 2-8, Montpellier, France}, 2008.

\bibitem[Cuzzolin(2024)]{cuzzolin2024uncertainty}
Fabio Cuzzolin.
\newblock Uncertainty measures: A critical survey.
\newblock \emph{Information Fusion}, page 102609, 2024.

\bibitem[Davies(1987)]{davies1987hypothesis}
Robert~B Davies.
\newblock Hypothesis testing when a nuisance parameter is present only under the alternative.
\newblock \emph{Biometrika}, 74\penalty0 (1):\penalty0 33--43, 1987.

\bibitem[De~Finetti(1937)]{de1937foresight}
Bruno De~Finetti.
\newblock Foresight: Its logical laws, its subjective sources.
\newblock In \emph{Breakthroughs in Statistics: Foundations and Basic Theory}, pages 134--174. Springer, 1937.

\bibitem[Destercke(2012)]{destercke_handling_2012}
Sébastien Destercke.
\newblock Handling bipolar knowledge with imprecise probabilities.
\newblock \emph{International Journal of Intelligent Systems}, 26\penalty0 (5):\penalty0 426--443, March 2012.
\newblock \doi{10.1002/int.20475}.
\newblock Publisher: Wiley.

\bibitem[Dubois et~al.(1996)Dubois, Prade, and Smets]{dubois1996representing}
Didier Dubois, Henri Prade, and Philippe Smets.
\newblock Representing partial ignorance.
\newblock \emph{IEEE Transactions on Systems, Man, and Cybernetics-Part A: Systems and Humans}, 26\penalty0 (3):\penalty0 361--377, 1996.

\bibitem[Elkin(2017)]{elkin2017imprecise}
Lee Elkin.
\newblock \emph{Imprecise probability in epistemology}.
\newblock PhD thesis, lmu, 2017.

\bibitem[Fishburn(1970)]{Fishburn70:utility}
P.C. Fishburn.
\newblock \emph{Utility Theory for Decision Making}.
\newblock Operations Research Society of America. Publications in operations research. Wiley, 1970.

\bibitem[Flaxman et~al.(2016)Flaxman, Sejdinovic, Cunningham, and Filippi]{flaxman2016bayesian}
Seth Flaxman, Dino Sejdinovic, John~P Cunningham, and Sarah Filippi.
\newblock Bayesian learning of kernel embeddings.
\newblock \emph{arXiv preprint arXiv:1603.02160}, 2016.

\bibitem[F{\"o}ll et~al.(2023)F{\"o}ll, Dubatovka, Ernst, Chau, Maritsch, Okanovic, Thaeter, Buhmann, Wortmann, and Muandet]{foll_gated_2023}
Simon F{\"o}ll, Alina Dubatovka, Eugen Ernst, Siu~Lun Chau, Martin Maritsch, Patrik Okanovic, Gudrun Thaeter, Joachim~M Buhmann, Felix Wortmann, and Krikamol Muandet.
\newblock Gated domain units for multi-source domain generalization.
\newblock \emph{Transactions on Machine Learning Research}, 2023.

\bibitem[Gartner(2008)]{gartner2008kernels}
Thomas Gartner.
\newblock \emph{Kernels for structured data}, volume~72.
\newblock World Scientific, 2008.

\bibitem[Gelman et~al.(1995)Gelman, Carlin, Stern, and Rubin]{gelman1995bayesian}
Andrew Gelman, John~B Carlin, Hal~S Stern, and Donald~B Rubin.
\newblock \emph{Bayesian data analysis}.
\newblock Chapman and Hall/CRC, 1995.

\bibitem[Giron and Rios(1980)]{giron_quasi-bayesian_1980}
F.~J. Giron and S.~Rios.
\newblock Quasi-{Bayesian} {Behaviour}: {A} more realistic approach to decision making?
\newblock \emph{Trabajos de Estadistica Y de Investigacion Operativa}, 31\penalty0 (1):\penalty0 17--38, February 1980.
\newblock ISSN 0041-0241.
\newblock \doi{10.1007/BF02888345}.
\newblock URL \url{http://link.springer.com/10.1007/BF02888345}.

\bibitem[Gretton et~al.(2006)Gretton, Borgwardt, Rasch, Sch{\"o}lkopf, and Smola]{gretton2006kernel}
Arthur Gretton, Karsten Borgwardt, Malte Rasch, Bernhard Sch{\"o}lkopf, and Alex Smola.
\newblock A kernel method for the two-sample-problem.
\newblock \emph{Advances in neural information processing systems}, 19, 2006.

\bibitem[Gretton et~al.(2012)Gretton, Borgwardt, Rasch, Sch{\"o}lkopf, and Smola]{gretton2012kernel}
Arthur Gretton, Karsten~M Borgwardt, Malte~J Rasch, Bernhard Sch{\"o}lkopf, and Alexander Smola.
\newblock A kernel two-sample test.
\newblock \emph{The Journal of Machine Learning Research}, 13\penalty0 (1):\penalty0 723--773, 2012.

\bibitem[Guo and Shah(2024)]{guo2024rank}
F~Richard Guo and Rajen~D Shah.
\newblock Rank-transformed subsampling: inference for multiple data splitting and exchangeable p-values.
\newblock \emph{Journal of the Royal Statistical Society Series B: Statistical Methodology}, page qkae091, 2024.

\bibitem[H{\'a}jek(2002)]{hajek2002interpretations}
Alan H{\'a}jek.
\newblock Interpretations of probability.
\newblock 2002.

\bibitem[He et~al.(2016)He, Zhang, Ren, and Sun]{he2016deep}
Kaiming He, Xiangyu Zhang, Shaoqing Ren, and Jian Sun.
\newblock Deep residual learning for image recognition.
\newblock In \emph{Proceedings of the IEEE conference on computer vision and pattern recognition}, pages 770--778, 2016.

\bibitem[Hibshman and Weninger(2021)]{hibshman2021higher}
Justus Hibshman and Tim Weninger.
\newblock Higher order imprecise probabilities and statistical testing.
\newblock \emph{arXiv preprint arXiv:2107.04542}, 2021.

\bibitem[Hofman et~al.()Hofman, Sale, and H{\"u}llermeier]{hofman2024quantifying}
Paul Hofman, Yusuf Sale, and Eyke H{\"u}llermeier.
\newblock Quantifying aleatoric and epistemic uncertainty: A credal approach.
\newblock In \emph{ICML 2024 Workshop on Structured Probabilistic Inference $\{$$\backslash$\&$\}$ Generative Modeling}.

\bibitem[Holmes et~al.(2015)Holmes, Caron, Griffin, and Stephens]{holmes2015two}
Chris~C. Holmes, Fran{\c c}ois Caron, Jim~E. Griffin, and David~A. Stephens.
\newblock Two-sample {{Bayesian Nonparametric Hypothesis Testing}}.
\newblock \emph{Bayesian Analysis}, 10\penalty0 (2):\penalty0 297--320, June 2015.

\bibitem[Hora(1996)]{hora1996aleatory}
Stephen~C Hora.
\newblock Aleatory and epistemic uncertainty in probability elicitation with an example from hazardous waste management.
\newblock \emph{Reliability Engineering \& System Safety}, 54\penalty0 (2-3):\penalty0 217--223, 1996.

\bibitem[Hsu and Ramos(2019{\natexlab{a}})]{hsu2019bayesian}
Kelvin Hsu and Fabio Ramos.
\newblock Bayesian learning of conditional kernel mean embeddings for automatic likelihood-free inference.
\newblock In \emph{The 22nd International Conference on Artificial Intelligence and Statistics}, pages 2631--2640. PMLR, 2019{\natexlab{a}}.

\bibitem[Hsu and Ramos(2019{\natexlab{b}})]{hsu2019bayesian2}
Kelvin Hsu and Fabio Ramos.
\newblock Bayesian deconditional kernel mean embeddings.
\newblock In \emph{International Conference on Machine Learning}, pages 2830--2838. PMLR, 2019{\natexlab{b}}.

\bibitem[Hu et~al.(2022)Hu, Chau, Ferrando~Huertas, and Sejdinovic]{hu2022explaining}
Robert Hu, Siu~Lun Chau, Jaime Ferrando~Huertas, and Dino Sejdinovic.
\newblock Explaining preferences with shapley values.
\newblock \emph{Advances in Neural Information Processing Systems}, 35:\penalty0 27664--27677, 2022.

\bibitem[Huber and Ronchetti(2011)]{huber2011robust}
Peter~J Huber and Elvezio~M Ronchetti.
\newblock \emph{Robust statistics}.
\newblock John Wiley \& Sons, 2011.

\bibitem[Hüllermeier and Waegeman(2021)]{hullermeier_aleatoric_2021}
Eyke Hüllermeier and Willem Waegeman.
\newblock Aleatoric and {Epistemic} {Uncertainty} in {Machine} {Learning}: {An} {Introduction} to {Concepts} and {Methods}.
\newblock \emph{Machine Learning}, 110\penalty0 (3):\penalty0 457--506, March 2021.
\newblock ISSN 0885-6125, 1573-0565.
\newblock \doi{10.1007/s10994-021-05946-3}.
\newblock arXiv:1910.09457 [cs, stat].

\bibitem[Kall(1986)]{kall1986approximation}
Peter Kall.
\newblock Approximation to optimization problems: An elementary review.
\newblock \emph{Mathematics of Operations Research}, 11\penalty0 (1):\penalty0 9--18, 1986.

\bibitem[Kendall and Gal(2017)]{kendall2017uncertainties}
Alex Kendall and Yarin Gal.
\newblock What uncertainties do we need in bayesian deep learning for computer vision?
\newblock \emph{Advances in neural information processing systems}, 30, 2017.

\bibitem[Key et~al.(2024)Key, Gretton, Briol, and Fernandez]{key_composite_2024}
Oscar Key, Arthur Gretton, François-Xavier Briol, and Tamara Fernandez.
\newblock Composite {Goodness}-of-fit {Tests} with {Kernels} ({latesT}), February 2024.
\newblock arXiv:2111.10275 [cs, stat].

\bibitem[Keynes(1921)]{keynes1921treatise}
John~Maynard Keynes.
\newblock \emph{A treatise on probability}.
\newblock 1921.

\bibitem[Kolmogorov and Bharucha-Reid(2018)]{kolmogorov2018foundations}
Andreui~Nikolaevich Kolmogorov and Albert~T Bharucha-Reid.
\newblock \emph{Foundations of the theory of probability: Second English Edition}.
\newblock Courier Dover Publications, 2018.

\bibitem[Kolmogorov(1960)]{Kolmogorov60:Foundations}
Andrey~N. Kolmogorov.
\newblock \emph{Foundations of the Theory of Probability}.
\newblock Chelsea Pub Co, 2 edition, 1960.

\bibitem[K{\"u}bler et~al.(2020)K{\"u}bler, Jitkrittum, Sch{\"o}lkopf, and Muandet]{Kubler20:LKT-WDS}
J.~M. K{\"u}bler, W.~Jitkrittum, B.~Sch{\"o}lkopf, and K.~Muandet.
\newblock Learning kernel tests without data splitting.
\newblock In \emph{Advances in Neural Information Processing Systems 33}, pages 6245--6255. Curran Associates, Inc., 2020.

\bibitem[K{\"u}bler et~al.(2022{\natexlab{a}})K{\"u}bler, Jitkrittum, Sch\"olkopf, and Muandet]{Kubler22:WTS}
J.~M. K{\"u}bler, Wittawat Jitkrittum, Bernhard Sch\"olkopf, and Krikamol Muandet.
\newblock A witness two-sample test.
\newblock In \emph{Proceedings of The 25th International Conference on Artificial Intelligence and Statistics}, volume 151 of \emph{Proceedings of Machine Learning Research}, pages 1403--1419. PMLR, 2022{\natexlab{a}}.

\bibitem[K{\"u}bler et~al.(2022{\natexlab{b}})K{\"u}bler, Stimper, Buchholz, Muandet, and Sch{\"o}lkopf]{Kubler22:ATS}
J.~M. K{\"u}bler, V.~Stimper, S.~Buchholz, K.~Muandet, and B.~Sch{\"o}lkopf.
\newblock Automl two-sample test.
\newblock In \emph{Advances in Neural Information Processing Systems 35}, volume~35, pages 15929--15941. Curran Associates, Inc., 2022{\natexlab{b}}.

\bibitem[K{\"u}bler et~al.(2022{\natexlab{c}})K{\"u}bler, Stimper, Buchholz, Muandet, and Sch{\"o}lkopf]{kubler2022automl}
Jonas~M K{\"u}bler, Vincent Stimper, Simon Buchholz, Krikamol Muandet, and Bernhard Sch{\"o}lkopf.
\newblock Automl two-sample test.
\newblock \emph{Advances in Neural Information Processing Systems}, 35:\penalty0 15929--15941, 2022{\natexlab{c}}.

\bibitem[Kutterer(2004)]{kutterer2004statistical}
Hansj{\"o}rg Kutterer.
\newblock Statistical hypothesis tests in case of imprecise data.
\newblock In \emph{V Hotine-Marussi Symposium on Mathematical Geodesy: Matera, Italy June 17--21, 2003}, pages 49--56. Springer, 2004.

\bibitem[Kyburg~Jr(1998)]{kyburg1998interval}
Henry~E Kyburg~Jr.
\newblock Interval-valued probabilities.
\newblock \emph{Imprecise Probabilities Project}, 1998.

\bibitem[Laplace(1812)]{laplace1812}
Pierre-Simon Laplace.
\newblock \emph{Théorie Analytique des Probabilités}.
\newblock 1812.

\bibitem[LeCun(1998)]{lecun1998mnist}
Yann LeCun.
\newblock The mnist database of handwritten digits.
\newblock \emph{http://yann. lecun. com/exdb/mnist/}, 1998.

\bibitem[Lehmann and Romano(2005)]{lehmann_testing_2005}
E.~L. Lehmann and Joseph~P. Romano.
\newblock \emph{Testing statistical hypotheses}.
\newblock Springer texts in statistics. Springer, New York, 3rd ed edition, 2005.
\newblock ISBN 978-0-387-98864-1.

\bibitem[Lehmann et~al.(1986)Lehmann, Romano, and Casella]{lehmann1986testing}
Erich~Leo Lehmann, Joseph~P Romano, and George Casella.
\newblock \emph{Testing statistical hypotheses}, volume~3.
\newblock Springer, 1986.

\bibitem[Lewis(1980)]{lewis1980subjectivist}
David Lewis.
\newblock A subjectivist’s guide to objective chance.
\newblock In \emph{IFS: Conditionals, Belief, Decision, Chance and Time}, pages 267--297. Springer, 1980.

\bibitem[Li(2007)]{li2007hypothesis}
Pengfei Li.
\newblock Hypothesis testing in finite mixture models.
\newblock 2007.

\bibitem[Liu et~al.(2020)Liu, Zhang, and Lu]{liu2020novel}
Feng Liu, Guangquan Zhang, and Jie Lu.
\newblock A novel non-parametric two-sample test on imprecise observations.
\newblock In \emph{2020 IEEE International Conference on Fuzzy Systems (FUZZ-IEEE)}, pages 1--6. IEEE, 2020.

\bibitem[Mansour et~al.(2012)Mansour, Mohri, and Rostamizadeh]{mansour_multiple_2012}
Yishay Mansour, Mehryar Mohri, and Afshin Rostamizadeh.
\newblock Multiple {Source} {Adaptation} and the {Renyi} {Divergence}, May 2012.
\newblock arXiv:1205.2628 [cs, stat].

\bibitem[Mashrur et~al.(2020)Mashrur, Luo, Zaidi, and Robles-Kelly]{mashrur2020machine}
Akib Mashrur, Wei Luo, Nayyar~A Zaidi, and Antonio Robles-Kelly.
\newblock Machine learning for financial risk management: a survey.
\newblock \emph{Ieee Access}, 8:\penalty0 203203--203223, 2020.

\bibitem[Mortier et~al.(2023)Mortier, Bengs, H{\"u}llermeier, Luca, and Waegeman]{mortier2023calibration}
Thomas Mortier, Viktor Bengs, Eyke H{\"u}llermeier, Stijn Luca, and Willem Waegeman.
\newblock On the calibration of probabilistic classifier sets.
\newblock In \emph{International Conference on Artificial Intelligence and Statistics}, pages 8857--8870. PMLR, 2023.

\bibitem[Muandet and Sch{\"o}lkopf(2013)]{muandet2013one}
Krikamol Muandet and Bernhard Sch{\"o}lkopf.
\newblock One-class support measure machines for group anomaly detection.
\newblock \emph{arXiv preprint arXiv:1303.0309}, 2013.

\bibitem[Muandet et~al.(2013)Muandet, Balduzzi, and Schölkopf]{muandet_domain_2013}
Krikamol Muandet, David Balduzzi, and Bernhard Schölkopf.
\newblock Domain {Generalization} via {Invariant} {Feature} {Representation}, January 2013.
\newblock URL \url{http://arxiv.org/abs/1301.2115}.
\newblock arXiv:1301.2115 [cs, stat].

\bibitem[Muandet et~al.(2017)Muandet, Fukumizu, Sriperumbudur, Sch{\"o}lkopf, et~al.]{muandet2017kernel}
Krikamol Muandet, Kenji Fukumizu, Bharath Sriperumbudur, Bernhard Sch{\"o}lkopf, et~al.
\newblock Kernel mean embedding of distributions: A review and beyond.
\newblock \emph{Foundations and Trends{\textregistered} in Machine Learning}, 10\penalty0 (1-2):\penalty0 1--141, 2017.

\bibitem[M{\"u}ller(1997)]{muller1997integral}
Alfred M{\"u}ller.
\newblock Integral probability metrics and their generating classes of functions.
\newblock \emph{Advances in applied probability}, 29\penalty0 (2):\penalty0 429--443, 1997.

\bibitem[Nguyen et~al.(2019)Nguyen, Destercke, and H{\"u}llermeier]{nguyen2019epistemic}
Vu-Linh Nguyen, S{\'e}bastien Destercke, and Eyke H{\"u}llermeier.
\newblock Epistemic uncertainty sampling.
\newblock In \emph{Discovery Science: 22nd International Conference, DS 2019, Split, Croatia, October 28--30, 2019, Proceedings 22}, pages 72--86. Springer, 2019.

\bibitem[Pogodin et~al.(2024)Pogodin, Schrab, Li, Sutherland, and Gretton]{pogodin24conditional}
Roman Pogodin, Antonin Schrab, Yazhe Li, Danica Sutherland, and Arthur Gretton.
\newblock Practical kernel tests of conditional independence.
\newblock \emph{arXiv preprint arXiv:2402.13196}, 2024.

\bibitem[Rudin(1991)]{Rudin91:FA}
W.~Rudin.
\newblock \emph{Functional Analysis}.
\newblock McGraw-Hill, 1991.
\newblock ISBN 9780070542365.

\bibitem[Sagawa et~al.(2020)Sagawa, Koh, Hashimoto, and Liang]{sagawa_distributionally_2020}
Shiori Sagawa, Pang~Wei Koh, Tatsunori~B. Hashimoto, and Percy Liang.
\newblock Distributionally {Robust} {Neural} {Networks} for {Group} {Shifts}: {On} the {Importance} of {Regularization} for {Worst}-{Case} {Generalization}, April 2020.
\newblock arXiv:1911.08731 [cs, stat].

\bibitem[Sale et~al.(2023)Sale, Caprio, and H{\"o}llermeier]{sale2023volume}
Yusuf Sale, Michele Caprio, and Eyke H{\"o}llermeier.
\newblock Is the volume of a credal set a good measure for epistemic uncertainty?
\newblock In \emph{Uncertainty in Artificial Intelligence}, pages 1795--1804. PMLR, 2023.

\bibitem[Schrab et~al.(2023)Schrab, Kim, Albert, Laurent, Guedj, and Gretton]{schrab2023mmd}
Antonin Schrab, Ilmun Kim, M{\'e}lisande Albert, B{\'e}atrice Laurent, Benjamin Guedj, and Arthur Gretton.
\newblock Mmd aggregated two-sample test.
\newblock \emph{Journal of Machine Learning Research}, 24\penalty0 (194):\penalty0 1--81, 2023.

\bibitem[Seidenfeld et~al.(1989)Seidenfeld, Kadane, and Schervish]{seidenfeld_shared_1989}
Teddy Seidenfeld, Joseph~B. Kadane, and Mark~J. Schervish.
\newblock On the {Shared} {Preferences} of {Two} {Bayesian} {Decision} {Makers}.
\newblock \emph{The Journal of Philosophy}, 86\penalty0 (5):\penalty0 225--244, 1989.
\newblock ISSN 0022-362X.
\newblock \doi{10.2307/2027108}.
\newblock URL \url{https://www.jstor.org/stable/2027108}.
\newblock Publisher: Journal of Philosophy, Inc.

\bibitem[Sejdinovic(2024)]{sejdinovic2024overview}
Dino Sejdinovic.
\newblock An overview of causal inference using kernel embeddings.
\newblock \emph{arXiv preprint arXiv:2410.22754}, 2024.

\bibitem[Sejdinovic and Gretton(2012)]{sejdinovic2012rkhs}
Dino Sejdinovic and Arthur Gretton.
\newblock What is an rkhs?
\newblock \emph{Lecture Notes}, 25, 2012.

\bibitem[Sejdinovic et~al.(2013)Sejdinovic, Sriperumbudur, Gretton, and Fukumizu]{sejdinovic2013equivalence}
Dino Sejdinovic, Bharath Sriperumbudur, Arthur Gretton, and Kenji Fukumizu.
\newblock Equivalence of distance-based and rkhs-based statistics in hypothesis testing.
\newblock \emph{The annals of statistics}, pages 2263--2291, 2013.

\bibitem[Shafer(1992)]{shafer1992dempster}
Glenn Shafer.
\newblock Dempster-shafer theory.
\newblock \emph{Encyclopedia of artificial intelligence}, 1:\penalty0 330--331, 1992.

\bibitem[Singh et~al.(2024)Singh, Chau, Bouabid, and Muandet]{singh2024domain}
Anurag Singh, Siu~Lun Chau, Shahine Bouabid, and Krikamol Muandet.
\newblock Domain generalisation via imprecise learning.
\newblock \emph{arXiv preprint arXiv:2404.04669}, 2024.

\bibitem[Smola et~al.(2007)Smola, Gretton, Song, and Sch{\"o}lkopf]{smola2007hilbert}
Alex Smola, Arthur Gretton, Le~Song, and Bernhard Sch{\"o}lkopf.
\newblock A hilbert space embedding for distributions.
\newblock In \emph{International conference on algorithmic learning theory}, pages 13--31. Springer, 2007.

\bibitem[Sriperumbudur et~al.(2010)Sriperumbudur, Gretton, Fukumizu, Sch{\"o}lkopf, and Lanckriet]{SriGreFukLanetal10}
B.~Sriperumbudur, A.~Gretton, K.~Fukumizu, B.~Sch{\"o}lkopf, and G.~Lanckriet.
\newblock Hilbert space embeddings and metrics on probability measures.
\newblock \emph{Journal of Machine Learning Research}, 11:\penalty0 1517--1561, 2010.

\bibitem[Sriperumbudur et~al.(2009)Sriperumbudur, Fukumizu, Gretton, Sch{\"o}lkopf, and Lanckriet]{sriperumbudur2009integral}
Bharath~K Sriperumbudur, Kenji Fukumizu, Arthur Gretton, Bernhard Sch{\"o}lkopf, and Gert~RG Lanckriet.
\newblock On integral probability metrics,$\backslash$phi-divergences and binary classification.
\newblock \emph{arXiv preprint arXiv:0901.2698}, 2009.

\bibitem[Sriperumbudur et~al.(2011)Sriperumbudur, Fukumizu, and Lanckriet]{sriperumbudur2011universality}
Bharath~K Sriperumbudur, Kenji Fukumizu, and Gert~RG Lanckriet.
\newblock Universality, characteristic kernels and rkhs embedding of measures.
\newblock \emph{Journal of Machine Learning Research}, 12\penalty0 (7), 2011.

\bibitem[Student(1908)]{student1908probable}
Student.
\newblock The probable error of a mean.
\newblock \emph{Biometrika}, pages 1--25, 1908.

\bibitem[Szab\'{o} et~al.(2016)Szab\'{o}, Sriperumbudur, P\'{o}czos, and Gretton]{Szabo16:DR}
Zolt\'{a}n Szab\'{o}, Bharath~K. Sriperumbudur, Barnab\'{a}s P\'{o}czos, and Arthur Gretton.
\newblock Learning theory for distribution regression.
\newblock \emph{Journal of Machine Learning Research}, 17\penalty0 (1):\penalty0 5272--5311, jan 2016.

\bibitem[Sz\'ekely and Rizzo(2005)]{SzeRiz05}
G\'abor~J. Sz\'ekely and Maria~L. Rizzo.
\newblock A new test for multivariate normality.
\newblock \emph{Journal of Multivariate Analysis}, 93\penalty0 (1):\penalty0 58--80, 2005.

\bibitem[Tan et~al.(2020)Tan, Yeom, Fredrikson, and Talwalkar]{tan2020learning}
Zilong Tan, Samuel Yeom, Matt Fredrikson, and Ameet Talwalkar.
\newblock Learning fair representations for kernel models.
\newblock In \emph{International Conference on Artificial Intelligence and Statistics}, pages 155--166. PMLR, 2020.

\bibitem[Thams et~al.(2023)Thams, Saengkyongam, Pfister, and Peters]{thams2023statistical}
Nikolaj Thams, Sorawit Saengkyongam, Niklas Pfister, and Jonas Peters.
\newblock Statistical testing under distributional shifts.
\newblock \emph{Journal of the Royal Statistical Society Series B: Statistical Methodology}, 85\penalty0 (3):\penalty0 597--663, 2023.

\bibitem[Tolstikhin et~al.(2017)Tolstikhin, Sriperumbudur, Mu, et~al.]{tolstikhin2017minimax}
Ilya Tolstikhin, Bharath~K Sriperumbudur, Krikamol Mu, et~al.
\newblock Minimax estimation of kernel mean embeddings.
\newblock \emph{Journal of Machine Learning Research}, 18\penalty0 (86):\penalty0 1--47, 2017.

\bibitem[Troffaes and De~Cooman(2014)]{troffaes2014lower}
Matthias~CM Troffaes and Gert De~Cooman.
\newblock \emph{Lower previsions}.
\newblock John Wiley \& Sons, 2014.

\bibitem[Ulmer(2021)]{ulmer2021survey}
Dennis~Thomas Ulmer.
\newblock A survey on evidential deep learning for single-pass uncertainty estimation.
\newblock 2021.

\bibitem[Walley(1991)]{walley1991statistical}
Peter Walley.
\newblock \emph{Statistical reasoning with imprecise probabilities}, volume~42.
\newblock Springer, 1991.

\bibitem[Walley(2000)]{walley2000towards}
Peter Walley.
\newblock Towards a unified theory of imprecise probability.
\newblock \emph{International Journal of Approximate Reasoning}, 24\penalty0 (2-3):\penalty0 125--148, 2000.

\bibitem[Walley and Fine(1979)]{walley1979varieties}
Peter Walley and Terrence~L Fine.
\newblock Varieties of modal (classificatory) and comparative probability.
\newblock \emph{Synthese}, pages 321--374, 1979.

\bibitem[Weisstein(2004)]{weisstein2004bonferroni}
Eric~W Weisstein.
\newblock Bonferroni correction.
\newblock \emph{https://mathworld. wolfram. com/}, 2004.

\bibitem[Wichitchan et~al.(2019)Wichitchan, Yao, and Yang]{wichitchan_hypothesis_2019}
Supawadee Wichitchan, Weixin Yao, and Guangren Yang.
\newblock Hypothesis testing for finite mixture models.
\newblock \emph{Computational Statistics \& Data Analysis}, 132:\penalty0 180--189, April 2019.
\newblock ISSN 01679473.
\newblock \doi{10.1016/j.csda.2018.05.005}.

\bibitem[Williamson(2010)]{williamson_defence_2010}
Jon Williamson.
\newblock \emph{In {Defence} of {Objective} {Bayesianism}}.
\newblock Oxford University Press, April 2010.
\newblock ISBN 978-0-19-922800-3.
\newblock \doi{10.1093/acprof:oso/9780199228003.001.0001}.

\bibitem[Zaffalon(2002)]{zaffalon2002naive}
Marco Zaffalon.
\newblock The naive credal classifier.
\newblock \emph{Journal of statistical planning and inference}, 105\penalty0 (1):\penalty0 5--21, 2002.

\bibitem[Zhang et~al.(2021)Zhang, Chen, Zhao, Xiong, Qin, and Liu]{zhang2021distributional}
Pushi Zhang, Xiaoyu Chen, Li~Zhao, Wei Xiong, Tao Qin, and Tie-Yan Liu.
\newblock Distributional reinforcement learning for multi-dimensional reward functions.
\newblock \emph{Advances in Neural Information Processing Systems}, 34:\penalty0 1519--1529, 2021.

\bibitem[Zhang et~al.(2022)Zhang, Wild, Filippi, Flaxman, and Sejdinovic]{zhang2022bayesian}
Qinyi Zhang, Veit Wild, Sarah Filippi, Seth Flaxman, and Dino Sejdinovic.
\newblock Bayesian kernel two-sample testing.
\newblock \emph{Journal of Computational and Graphical Statistics}, 31\penalty0 (4):\penalty0 1164--1176, 2022.

\end{thebibliography}
